\documentclass{article}

\usepackage{arxiv}

\usepackage[utf8]{inputenc} % allow utf-8 input
\usepackage[T1]{fontenc}    % use 8-bit T1 fonts
\usepackage{hyperref}       % hyperlinks
\usepackage{url}            % simple URL typesetting
\usepackage{booktabs}       % professional-quality tables
\usepackage{amsfonts}       % blackboard math symbols
\usepackage{nicefrac}       % compact symbols for 1/2, etc.
\usepackage{microtype}      % microtypography
\usepackage{xcolor}         % colors
\usepackage{graphicx}
\usepackage{doi}
\usepackage{wrapfig}

\usepackage{amsthm}
\usepackage{amsmath, amssymb}
\usepackage{float}
\usepackage{mathtools}
\usepackage{algorithm}
\usepackage[noend]{algpseudocode}
\usepackage{thmtools,thm-restate}

\newtheorem{theorem}{Theorem}
\newtheorem{definition}[theorem]{Definition}

\theoremstyle{definition}

\newcommand{\norm}[1]{\left\lVert#1\right\rVert}

\newcommand{\R}{\mathbb{R}}

\usepackage{xspace}
\makeatletter
\DeclareRobustCommand\onedot{\futurelet\@let@token\@onedot}
\def\@onedot{\ifx\@let@token.\else.\null\fi\xspace}

\def\ie{{i.e}\onedot}

\def\wrt{w.r.t\onedot} 

\makeatother

\newcommand\at[2]{\left.#1\right|_{#2}}

\title{Neural Collapse versus Low-rank Bias: \\Is Deep Neural Collapse Really Optimal?}

\author{%
  Peter S\'uken\'ik \\
  Institute of Science and Technology Austria\\
  3400 Klosterneuburg, Austria \\
  \texttt{peter.sukenik@ista.ac.at} \\
  \And
  Christoph Lampert\footnotemark[1] \\
  Institute of Science and Technology Austria\\
  3400 Klosterneuburg, Austria \\
  \texttt{chl@ista.ac.at} \\
  \AND
  Marco Mondelli\thanks{Equal contribution} \\
  Institute of Science and Technology Austria\\
  3400 Klosterneuburg, Austria \\
  \texttt{marco.mondelli@ista.ac.at} \\
}

\begin{document}

\maketitle

\begin{abstract}
Deep neural networks (DNNs) exhibit a surprising structure in their final layer known as neural collapse (NC), and a growing body of works has currently investigated the propagation of neural collapse to earlier layers of DNNs -- a phenomenon called deep neural collapse (DNC). However, existing theoretical results are restricted to special cases: linear models, only two layers or binary classification. In contrast, we focus on non-linear models of arbitrary depth in multi-class classification and reveal a surprising qualitative shift. As soon as we go beyond two layers or two classes, DNC stops being optimal for the deep unconstrained features model (DUFM) -- the standard theoretical framework for the analysis of collapse. The main culprit is a low-rank bias of multi-layer regularization schemes: this bias leads to optimal solutions of even lower rank than the neural collapse. We support our theoretical findings with experiments on both DUFM and real data, which show the emergence of the low-rank structure in the solution found by gradient descent.
\end{abstract}

%\vspace{-.5em}

\section{Introduction}

%\vspace{-.5em}

What is the geometric structure of layers and learned representations in deep neural networks (DNNs)? To address this question, Papyan et al.\ \cite{papyan2020prevalence} focused on the very last layer of DNNs at convergence and experimentally measured what is now widely known as Neural Collapse (NC). This phenomenon refers to four properties that simultaneously emerge during the terminal phase of training: feature vectors of training samples from the same class collapse to the common class-mean (NC1); the class means form a simplex equiangular tight frame or an orthogonal frame (NC2); the class means are aligned with the rows of the last layer's weight matrix (NC3); and, finally, the classifier in the last layer is a nearest class center classifier (NC4). 
%The first (NC1) refers to within-class variability collapse where feature vectors of training samples from the same class collapse to the common class-mean. According to the second property (NC2), these class means form a simplex equiangular tight frame or orthogonal frame (depending on the setting) and, according to the third (NC3), % -- a highly symmetric structure. they are aligned with the rows of the last layer's weight matrix. Finally, these three properties imply that the classifier in the last layer is a nearest class center classifier (NC4). This 
Since the influential paper \cite{papyan2020prevalence}, a line of research has 
%After Papyan et. al.'s \cite{papyan2020prevalence} influential paper, tens of others followed, 
aimed at explaining the emergence of NC theoretically, mostly focusing on the unconstrained features model (UFM) \cite{mixon2020neural}. In this model, motivated by the network's perfect expressivity, one treats the last layer's feature vectors as a free variable and explicitly optimizes them together with the last layer's weight matrix, ``peeling off'' the rest of the network \cite{fang2021exploring, ji2021unconstrained}. With UFM, the NC was demonstrated in a variety of settings, both as the global optimum and as the convergence point of gradient flow. % (see Section~\ref{sec:related}).

The emergence of the NC in the last layer led to a natural research question -- does some form of collapse propagate beyond the last layer to earlier layers of DNNs? A number of empirical works \cite{hui2022limitations, he2022law, rangamani2023feature, parker2023neural, masarczyk2024tunnel} gave evidence that this is indeed the case, and we will refer to this phenomenon as Deep Neural Collapse (DNC).
%First mentioned as ``cascading neural collapse'' in \cite{hui2022limitations}, %this was empirically studied in 
%various works \cite{he2022law, rangamani2023feature, parker2023neural, masarczyk2024tunnel} %, which
%gave evidence that 
%. Although disagreeing in fine details, the works generally agree that 
%some form of NC indeed propagates to earlier layers. We refer to this phenomenon as Deep Neural Collapse (DNC). Similarly as with NC, the empirical observations were followed by theory trying to explain it. 
On the theoretical side, the optimality of the DNC was obtained \emph{(i)} for the UFM with two layers connected by a non-linearity in \cite{tirer2022extended}, \emph{(ii)} for the UFM with several linear layers in \cite{dang2023neural}, and \emph{(iii)} for the deep UFM (DUFM) with non-linear activations in the context of binary classification \cite{sukenik2023deep}. No existing work handles the general case in which there are \emph{multiple classes} and the UFM is \emph{deep} and \emph{non-linear}.

In this work, \textit{we close the gap %} and, by doing so, %by studying arbitrary depth non-linear DUFM in multiclass classification. By doing so 
%\textbf{we 
and reveal a surprising behavior} not occurring in the 
simpler settings above: %considered earlier: %twist that could have been discovered in neither of the previous settings. 
for multiple classes and layers, %when the number of layers $L\ge3$ and number of classes $K\ge3,$ 
the DNC as formulated in previous works is \textit{not an optimal solution} of DUFM. In particular, the class means at the optimum do not form an orthogonal frame (nor an equiangular tight frame), thus violating the second property of DNC. % DNC2.

Let $L$ and $K$ denote the number of layers and classes, respectively. Then, if either $L\ge3$ and $K\ge10$ or $L\ge4$ and $K\ge6$, we provide an explicit combinatorial construction of a class of solutions that outperforms DNC. Specifically, the loss achieved by our construction is a factor $K^{(L-3)/(2L+2)}$ lower than the loss of the DNC solution. 
Our result holds as long as all matrices are regularized.  

We also identify the reason behind the sub-optimality of DNC: a \emph{low-rank bias}. Intuitively, this bias arises from the representation cost of a DNN with $l_2$ regularization, which equals the Schatten-$p$ quasi norm \cite{ongie2022role} in the deep linear case. The quasi norm is well approximated by the rank, and this intuition carries over to the non-linear case as well.
%This bias comes from a Frobenius norm regularization of multiple consecutive weight matrices and is well explained by the variational form of Schatten-$p$ quasi-norms \cite{giampouras2020novel} that approximate rank. %We prove this result for $L\ge3$ and $K\ge10$ or $L\ge4$ and $K\ge6$ by identifying a class of solutions that outperform DNC. 
In fact, the rank of our construction is %of order 
$\Theta(\sqrt{K})$, while the rank of the DNC solution is $K$. We note that after the application of the ReLU, the rank of the final layer is again equal to $K$, in order to fit the training data. 
%Our explicit construction is based on the adjacency matrix of triangular graphs, and it leads to solutions These solutions are of rank $\Theta(\sqrt{K})$ (as opposed to rank $K$ of DNC) and are based on a highly symmetric combinatorial construction. We construct the solutions so that the gram matrix of class-means is a re-scaled adjacency matrix of a triangular graph (edge dual of a complete graph), see Figure~\ref{fig:srg_illustration} and Section~\ref{sec:theory}. Crucially, the solutions are constructed so that in the \textit{final layer}, the rank becomes full $K$ thanks to the application of ReLU activation function, which allows a good fit of the training data. 
%Importantly, this phenomenon is weight-decay agnostic as long as all matrices are regularized. On the other hand, 
We also show that the first property of neural collapse (convergence to class means) continues to be strictly optimal even in this general setting and its deep counterpart is approximately optimal with smoothed %(differentiable) 
ReLU activations. %Thus, the non-optimality of DNC is mostly due to high-rankness of the DNC2 property that requires matrix of class means to be orthogonal. 

%\begin{figure}
%    \centering
%    \includegraphics[width=0.4\textwidth]{figures/intro/losses.png}
%    \includegraphics[width=0.4\textwidth]{figures/intro/srg_10_cifar_special.png}
%    \caption{\marco{remove this?}Gram matrices of class-means in late layers. Both matrices are re-scaled adjacency matrices of triangular grpahs. \textbf{Left:} Solution found with ResNet20 on CIFAR10 dataset. \textbf{Right:} Solution found by DUFM optimization with 15 classes.}
%    \label{fig:srg_illustration}
%\end{figure}

We support our theoretical results with empirical findings %an extensive empirical investigation 
in three regimes: \emph{(i)} DUFM training, \emph{(ii)} training on standard datasets (MNIST~\cite{lecun1998gradient}, CIFAR-10~\cite{krizhevsky09}) with DUFM-like regularization, and \emph{(iii)} training on standard datasets with standard regularization. In all cases, gradient descent %training is able to 
retrieves solutions with very low rank, which can exhibit symmetric structures %that are often very 
in agreement with our combinatorial construction, see e.g.\ the lower-right plot of Figure~\ref{fig:standard_reg_main}.
%). Moreover, the solutions often exhibit highly symmetric structures, similar or even identical to our combinatorial construction in Section~\ref{sec:theory}. 
%However, we go beyond just confirming our theory. 
We also investigate the effect of three common hyperparameters -- weight decay, learning rate and width -- on the rank of the solution at convergence. On the one hand, high weight decay, high learning rate and small width lead to a strong low-rank bias. On the other hand, small (yet still non-zero!) weight decay, small learning rate or large width (and more complex datasets as well) lead to %the network to exhibit 
a higher-rank solution, even if that is not the global optimum, and % Furthermore, 
this solution often coincides with DNC, which is in agreement with earlier experimental evidence. Altogether, our findings show that if a DNC solution is found, it is not because of its global optimality, but just because of an implicit bias of the optimization procedure. %which is uch solution is not globally optimal. Thus,  our findings are not in strict conflict with earlier observations. %, although we do stress out the necessity of finer analysis of DNC2 and its common metrics in future works, because the current trends might not be sufficient to fully capture the fine distinctions in the gram structure. 

%We highlight that tT
The implications of our results go beyond %just 
\textit{deep} neural collapse. In fact, our theory suggests that even the NC2 in the last layer is not %an
optimal, % solution, 
and this is corroborated by our experiments, where the singular value structure of the last layer's class-mean matrices is imbalanced, ruling out %the 
orthogonality. This means that standard single-layer UFM, as well as its deep-linear or two-layer extensions, are not sufficient to describe the full picture, as they display a qualitatively different phenomenology. %results from reality. 

\vspace{-.5em}

\section{Related work} \label{sec:related}

\vspace{-.5em}

\paragraph{Neural Collapse.} %The seminal work %by Papyan et. al.
%\cite{papyan2020prevalence} triggered significant %trong research 
%interest in the community. 
Several papers (a non-exhaustive list includes \cite{galanti2022improved, haas2022linking, ben2022nearest, li2022principled, li2023no, zhang2024epa}) use neural collapse as a practical tool in %various
applications, among which OOD detection and transfer learning are the most prevalent. 
On the theoretical side, the emergence of NC has been investigated, with the %vast 
majority of works considering some form of UFM \cite{mixon2020neural, fang2021exploring}. % for their analysis. 
\cite{wojtowytsch2020emergence, lu2022neural} show global optimality of NC under the cross-entropy (CE) loss, and \cite{zhou2022optimization} under the MSE loss. Similar results are obtained by \cite{fang2021exploring, thrampoulidis2022imbalance, hong2023neural, dang2024neural} for the class-imbalanced setting. \cite{zhu2021geometric, ji2021unconstrained, zhou2022optimization} refine the analysis by %further 
showing that the loss landscape of the UFM model is benign -- all stationary points are either local minima or strict saddle points which can be escaped by conventional optimizers. A more loss-agnostic approach connecting CE and MSE loss is considered in \cite{zhou2022all}. NC has also been analyzed for a large number of classes \cite{jiang2023generalized}, in an NTK %\cite{jacot2018neural}
regime \cite{seleznova2023neural}, or in graph neural networks \cite{kothapalli2023neural}. We refer the reader to \cite{kothapalli2022neural} for a survey. 

%Alternatively, t
The emergence of NC has also been studied through the lens of the gradient flow dynamics. \cite{mixon2020neural} considers MSE loss and small initialization, and \cite{han2021neural} a renormalized gradient flow of the last layer's features after fixing the last layer's weights to be conditionally optimal. \cite{ji2021unconstrained} studies the CE loss dynamics and shows convergence in direction of the gradient flow  to a KKT point of the max-margin problem of the UFM, extending a similar analysis for the last layer's weights in \cite{soudry2018implicit}. The convergence speed under both losses is described in \cite{wang2022linear}.
Going beyond UFM, % and try to reason about NC in different ways.
\cite{xu2023dynamics, poggio2020explicit, poggio2020implicit, kunin2022asymmetric} study the emergence of NC in homogeneous networks under gradient flow; \cite{pan2023towards} provides sufficient conditions for neural collapse; and \cite{tirer2022perturbation} perturbs %introduces perturbations of 
the unconstrained features to account for the limitations of the model.

More recently, \cite{hui2022limitations} mentions a possible propagation of the NC to earlier layers of DNNs, giving preliminary measurements. These are then significantly extended in \cite{he2022law, rangamani2023feature, galanti2023implicit, parker2023neural}, which measure the emergence of some form of DNC in DNNs. On the theoretical front, an extension to a two-layer non-linear model is provided in \cite{tirer2022extended}, to a deep linear model in \cite{dang2023neural, garrod2024unifying} and to a deep non-linear model for binary classification in \cite{sukenik2023deep}. Alternatively to DUFM, \cite{beaglehole2024average} studies DNC in an end-to-end setting with a special layer-wise training procedure. 

\vspace{-.5em}

\paragraph{Low-rank bias.} The low-rank bias is a well-known %and rather well-studied 
phenomenon, especially in the context of matrix/tensor factorization and deep linear networks (see e.g. \cite{arora2019implicit, chou2024gradient, razin2021implicit, ji2018gradient, wang2023implicit}). %, and it is %also 
%widely utilized in applications, such as network compression (see e.g. \cite{yu2017compressing, alvarez2017compression, tukan2020compressed}). 
%
For %In the context of %Moving to %Less is known for 
non-linear DNNs, %although some understanding is achieved in the following works. 
\cite{timor2023implicit} studies the gradient flow optimization of ReLU networks, giving lower and upper bounds on the average soft rank. % in several settings. 
\cite{galanti2022sgd} studies SGD training on deep ReLU networks, showing upper bounds on the rank of the weight matrices as a function of batch size, weight decay and learning rate. \cite{le2022training} proves several training invariances that may lead to low-rank, but the results require the norm of at least one weight matrix to diverge and the architecture to end with a couple of linear layers. \cite{baker2024low} presents bounds on the singular values of non-linear layers in a rather generic setting, not necessarily at convergence. More closely related to our work is \cite{ongie2022role}, which considers a deep linear network followed by a single non-linearity and then by a single layer. Their arguments to study the low-rank bias are similar to the intuitive explanation of %we put forward in %through the lens of representation cost. They use arguments similar to what we believe is an intuitive explanation behind low-rank bias in a context most relevant to ours, see 
Section~\ref{sec:dnc_not_optimal}. \cite{huh2021low} shows that increasing the depth results in lower effective rank of the penultimate layer’s Gram matrix both at initialization and at convergence. The true rank is also measured, but on rather shallow networks and it is far above the DNC rank. \cite{andriushchenko2024sharpness} shows a strong low-rank bias of sharpness-aware minimization, although only in layers where DNC does not yet occur and the rank is high. 
\cite{feng2022rank, jacot2022implicit} study special functional ranks (Jacobi and bottleneck) of DNNs, providing %mostly
asymptotic results and empirical measurements. These results are %further
refined %for finite depths 
in \cite{jacot2023bottleneck, wen2024frequencies}, which show a bottleneck structure of the rank both experimentally and theoretically. The measurements of the singular values %of the networks 
at convergence in \cite{wen2024frequencies} are in agreement with those of Section \ref{ssec:standard_reg}.  We highlight that \textit{none} of the results above allows to % works discussed above provides results that would be even close to 
reason about DNC optimality, as they focus % in our context. %, since most of them are focusing 
on infinite width/depth, effective or functional ranks, orthogonal settings, or are not quantitative enough. %ranks relevant to DNC context. Moreover, neither of the presented results is quantitative enough to provide bounds that could reason about the optimality of the DNC. 

\vspace{-.5em}

\section{Preliminaries} \label{sec:definitions} 

\vspace{-.5em}

We study the class balanced setting with $N=Kn$ samples from $K$ classes, $n$ per class. Let $f(x)=W_L\sigma(W_{L-1}\sigma(\dots W_1\mathcal{B}(x)\dots))$ be a DNN with backbone $\mathcal{B}(\cdot)$. The backbone represents the majority of the deep network \emph{before} the last $L$ layers, e.g. the convolutional part of a ResNet20. Let $X \in \mathbb{R}^{d\times N}$ be the training data, and $H_1=\mathcal{B}(X_0) \in \R^{d_1\times N}, H_2=\sigma(W_1X_1) \in \R^{d_2\times N}, \dots, H_L=\sigma(W_{L-1}X_{L-1}) \in \R^{d_L\times N}$ its feature vector representations in the last $L$ layers, with $\Tilde{X}_l$ denoting their counterparts before applying the ReLU $\sigma$. % (note that the indexing of layers starts only after the backbone!). 
We refer to $h^l_{ci}$ and $ \Tilde{h}^l_{ci}$ as to the $i$-th sample of $c$-th class of $H_l$ and $\Tilde{H}_l$, respectively. Let $\mu^l_c=\frac{1}{n}\sum_{i=1}^n h^l_{ci}$ and $\Tilde{\mu}^l_c=\frac{1}{n}\sum_{i=1}^n \Tilde{h}^l_{ci}$ be the class means at layer $l$ after and before applying $\sigma$, and $M_l, \Tilde{M}_l$ the matrices of the respective class means stacked into columns. %For convenience w
We organize the training samples so that the labels $Y\in \R^{K\times N}$ equal $I_K \otimes \mathbf{1}_n^T,$ where $I_K$ is a $K\times K$ identity matrix, $\otimes$ is the Kronecker product and $\mathbf{1}_n$ the all-one vector of size $n.$

\vspace{-.5em}

\paragraph{Deep neural collapse (DNC).}

As there are no biases in our network model, the second property of DNC requires the class mean matrices to be orthogonal (instead of forming an ETF) \cite{rangamani2023feature, sukenik2023deep}. 

\begin{definition}\label{def:dnc}
We say that layer $l$ exhibits DNC 1, 2 or 3 if the corresponding conditions are satisfied (the properties can be stated for both after and before the application of ReLU):

\vspace{-.5em}

\begin{itemize}
    \item[DNC1:] The within-class variability of either $H_l$ or $\Tilde{H}_l$ is $0$. Formally, %Mathematically, the condition requires
    $h_{ci}^l=h_{cj}^l, \Tilde{h}_{ci}^l=\Tilde{h}_{cj}^l$ for all $i, j\in [n]$ or, in matrix notation, $H_l=M_l \otimes \mathbf{1}_n^T, \Tilde{H}_l=\Tilde{M}_l \otimes \mathbf{1}_n^T.$
\vspace{-.25em}
    \item[DNC2:] The class-mean matrices $M_l, \Tilde{M}_l$ are orthogonal, i.e., $M_l^TM_l \propto I_K,  \Tilde{M}_l^T\Tilde{M}_l \propto I_K$. 
\vspace{-.25em}
    \item[DNC3:] The rows of the weight matrix $W_l$ are either 0 or collinear with one of the columns of the class-means matrix $M_l.$
\end{itemize}
\end{definition}

\vspace{-.5em}

%In practice these properties do not hold perfectly, but the training approaches very close these ideal states. 

\paragraph{Deep unconstrained features model.} To define DUFM, we generalize the model in \cite{sukenik2023deep} to an arbitrary number of classes $K.$

\begin{definition}\label{def:DUFM}
The $L$-layer deep unconstrained features model ($L$-DUFM) denotes the following optimization problem: 
\vspace{-.5em}
\begin{align}\label{eq:LDUFM}
\min_{H_1, W_1, \dots, W_L} \ 
&\frac{1}{2N}\norm{W_L\sigma(W_{L-1}\sigma(\dots W_2\sigma(W_1H_1)\dots))-Y}_F^2+\sum_{l=1}^L\frac{\lambda_{W_l}}{2}\norm{W_l}_F^2+
\frac{\lambda_{H_1}}{2}\norm{H_1}_F^2,    
\end{align}
where $\norm{\cdot}_F$ denotes the Frobenius norm and $\lambda_{H_1}, \lambda_{W_1}, \ldots, \lambda_{W_L}>0$ are regularization parameters.
\end{definition}

%For completeness we note that the DNC solution we will use in the following theorems is also properly defined in Appendix~\ref{app:theory} and the definition is slightly stricter as that of Definition~\ref{def:dnc}, because it also defines the scales of the matrices and the weight matrices are defined implicitly. However, any DNC solution of that definition satisfies the properties outlined in Definition~\ref{def:dnc} and any solution satisfying the properties in Definition~\ref{def:dnc} satisfies the definition in Appendix~\ref{app:theory} after proper rescaling.

\section{Low-rank solutions outperform deep neural collapse} \label{sec:dnc_not_optimal}

%\chl{This section needs some illustrations, of the graphs and the corresponding matrices, otherwise I don't think the reader will appreciate (or understand) it.} \marco{I agree!}

\vspace{-.5em}

\paragraph{Intuitive explanation of the low-rank bias.} Consider a simplified version of $L-$DUFM:
\vspace{-.5em}
\begin{align}\label{eq:LDUFM_simplified}
\min_{H_1, W_1, \dots, W_L} \ 
&\frac{1}{2N}\norm{W_L\sigma(W_{L-1}\dots W_2W_1H_1)-Y}_F^2+\sum_{l=1}^L\frac{\lambda_{W_l}}{2}\norm{W_l}_F^2+
\frac{\lambda_{H_1}}{2}\norm{H_1}_F^2.
\end{align}
Compared to \eqref{eq:LDUFM}, \eqref{eq:LDUFM_simplified} removes all non-linearities except in the last layer, making the remaining part of the network a deep linear model, a construction similar to the one in \cite{ongie2022role}. Now, we leverage the variational form of the Schatten-$p$ quasi-norm \cite{shang2020unified}, which gives
\vspace{-.5em}
\begin{align*}
c\norm{\Tilde{H}_L}_{S_{2/L}}^{2/L}=\underset{H_1, W_1, \dots, W_{L-1}: H_1W_1\dots W_{L-1}=\Tilde{H}_L}{\min} \sum_{l=1}^{L-1}\frac{\lambda_{W_l}}{2}\norm{W_l}_F^2+\frac{\lambda_{H_1}}{2}\norm{H_1}_F^2,
\end{align*}
where $c$ can be computed explicitly. Thus, after solving for $H_1, W_1, \dots, W_{L-1}$, the simplified $L$-DUFM problem~\eqref{eq:LDUFM_simplified} can be reduced to
\vspace{-.5em}
\begin{align*}
\min_{\Tilde{H}_L W_L} \ 
&\frac{1}{2N}\norm{W_L\sigma(\Tilde{H}_L)-Y}_F^2+\frac{\lambda_{W_L}}{2}\norm{W_L}_F^2+
\frac{\lambda_{\Tilde{H}_L}}{2}\norm{\Tilde{H}_L}_{S_{2/L}}^{2/L}.
\end{align*}
For large values of $L$,  $\norm{\Tilde{H}_L}_{S_{2/L}}^{2/L}$ is well approximated by the rank of $\Tilde{H}_L.$ Hence, the objective value is low when the output $W_LH_L$ fits $Y$ closely, while keeping $\Tilde{H}_L$ low-rank, which justifies the low-rank bias. Crucially, the presence of additional non-linearities in the $L$-DUFM model~\eqref{eq:LDUFM} does not change this effect much, as long as one is able to define solutions for which most of the intermediate feature matrices $\Tilde{H}_l$ are non-negative (so that ReLU does not have an effect). 

\vspace{-.5em}

\paragraph{Low-rank solution outperforming DNC.}
We define the combinatorial solution that outperforms DNC, starting from the graph structure on which the construction is based. % (for additional details and illustrations, see Appendix~\ref{app:theory}).
%Here we provide simplified definitions of our combinatorial solution that outperforms DNC. Please see Appendix~\ref{app:theory} for detailed definitions and illustrations. We start by defining the graph structure on which the construction is based:

\begin{restatable}{definition}{trianggraph}
\label{def:triangular_graph}
A triangular graph $\mathcal{T}_n$ of order $n$ is a \textit{line graph} of a complete graph $\mathcal{K}_n$ of order $n$. $\mathcal{T}_n$ has ${n\choose2}$ vertices, each representing an edge of the complete graph, and there is an edge between a pair of vertices if and only if the corresponding edges in the complete graph share a vertex. Moreover, let $T_n$ be the normalized incidence matrix of $\mathcal{K}_n,$ \ie, $(T_n)_{i,j}=\frac{1}{\sqrt{n-1}}$ if vertex $i$ belongs to edge $j$ and 0 otherwise. Let $G_n$ denote the adjacency matrix of $\mathcal{T}_n.$ 
\end{restatable}

We recall that $\mathcal{T}_n$ is a strongly regular graph with parameters $(n(n-1)/2, 2(n-2), n-2, 4)$ and spectrum $2(n-2)$ with multiplicity 1, $n-4$ with multiplicity $n-1$ and $-2$ with multiplicity $n(n-3)/2.$ Next, we construct an explicit solution $(H_1, W_1, \dots, W_L)$ based on the triangular graph. For ease of exposition, we focus on the case where the number of classes $K$ equals $r\choose2$ for some $r\ge 4$, deferring the general definition to Appendix~\ref{app:theory_dnc2}.

\begin{definition} \label{def:srg_solution_main}
Let $K= {r\choose2}$ for $r\ge4$. Then, a \textit{strongly regular graph (SRG)} solution of the $L$-DUFM problem \eqref{eq:LDUFM} is obtained by setting the matrices $(H_1, W_1, \dots, W_L)$ as follows:
\begin{itemize}
    \item For all $l,$ the feature matrices $H_l, \Tilde{H}_l$ are DNC1 collapsed, \ie, $H_l=M_l\otimes \mathbf{1}_n^T, \Tilde{H}_l=\Tilde{M}_l \otimes \mathbf{1}_n^T$.
    \item For $2\le l \le L-1$, $M_l=\Tilde{M}_l$, each row of $\Tilde{M}_l$ is a non-negative multiple of a row of $T_r$ (as in Definition \ref{def:triangular_graph}), and the sum of squared norms of the rows of $\Tilde{M}_l$ corresponding to a row of $T_r$ is the same for each row of $T_r$. Since $\Tilde{M}_l$ is entry-wise non-negative, $M_l=\Tilde{M}_l.$
    \item For $l=1,$ $W_1, M_1$ are any pair of matrices minimizing the objective conditionally on $M_2$ defined above.
    \item For $l\ge 2$, $W_l$ minimizes the objective conditional to input and output to that layer.
    \item As for the last layer $L,$ let $A_L$ be a $K \times r$ matrix where the set of rows equals the set of vectors with two $(-1)$ entries and $r-2$ $(+1)$ entries. Then, $M_L=\sigma(\Tilde{M}_L)$, the rows of $\Tilde{M}_L$ are a non-negative multiple of $A_LT_r$, and the sum of their squared norms corresponding to either row of $A_LT_r$ is equal.
    \item Finally, the Frobenius norms (i.e. scales) of $M_1, W_1, \dots, W_L$ are chosen so as to minimize~\eqref{eq:LDUFM} while satisfying the construction above. 
\end{itemize}
\end{definition}

%\begin{wrapfigure}{l}{0.4\textwidth}
%  \begin{center}
%    \includegraphics[width=0.17\textwidth]{figures/illustrations/H_3.png}
%    \hspace{2em}\includegraphics[width=0.17\textwidth]{figures/illustrations/H_4.png}
%  \end{center}
%  \caption{$M_3$ (left) and $\Tilde{M}_4$ (right) of a $L=4, K=10$ SRG solution.}
%  \label{fig:srg_illustration}
%\end{wrapfigure}

\begin{figure}
    \centering
    \includegraphics[width=0.16\textwidth]{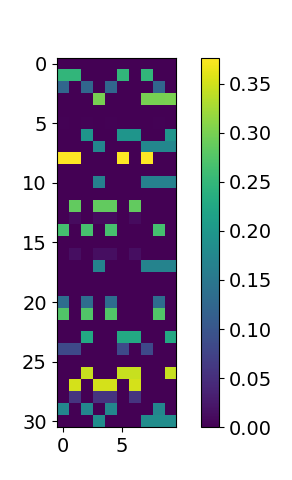}
    \hspace{.75em}
    \includegraphics[width=0.16\textwidth]{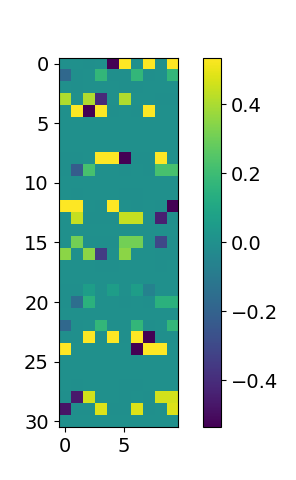}
    \includegraphics[width=0.34\textwidth]{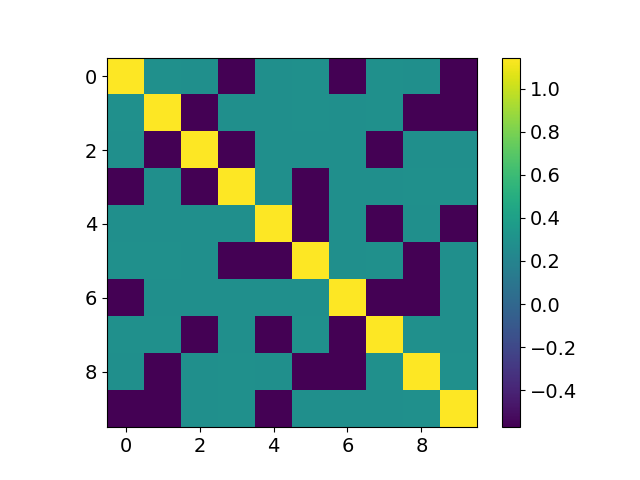}
    \hspace{-1.6em}\includegraphics[width=0.34\textwidth]{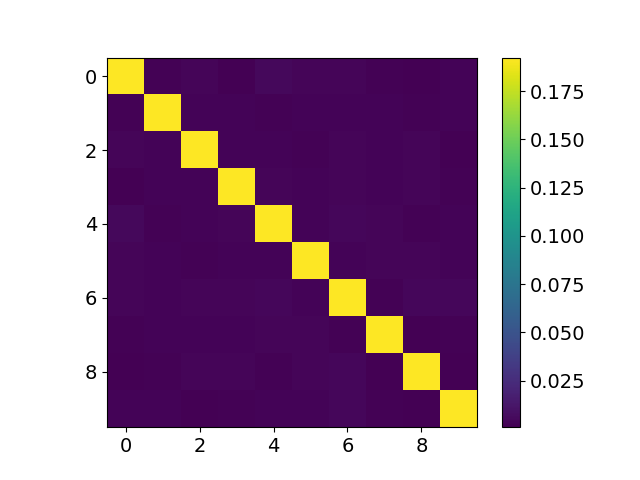}
    \hspace{-1.6em}
    \vspace{-.5em}
    \caption{%Visualization of the s
    Strongly regular graph (SRG) solution with $L=4$, $K=10$ and $r=5$. \textbf{Left:} Class-mean matrix of the third layer $M_3$. The non-zero entries of each row have the same value and their number is $r-1$, which corresponds to the degree of the complete graph $\mathcal K_r$. \textbf{Middle:} Class-mean matrix of the fourth layer before ReLU $\Tilde{M}_4$ (\textbf{middle left}), and its Gram matrix $\Tilde{M}_4^T\Tilde{M}_4$ (\textbf{middle right}). The SRG construction has very low rank before %the application of the 
    ReLU: ${\rm rank}(\Tilde{M}_4)=r$ and ${\rm rank}(\sigma(\Tilde{M}_4))=K$. \textbf{Right:} $\Tilde{M}_4^T\Tilde{M}_4$ for DNC. The DNC solution has rank $K$ in all layers before and after ReLU.}
    \label{fig:srg_illustration}
    \vspace{-1em}
\end{figure}

In this construction, columns and rows of class-mean matrices %(classes) in this construction 
are associated to edges and vertices of the complete graph $\mathcal{K}_r$. Each row (corresponding to a vertex) has non-zero entries at columns that correspond to edges containing the vertex. %, respectively. 
%Each entry is non-zero if the vertex (corresponding to the row) has an adjacent edge (corresponding to the column). %. The rows correspond to vertices of the $\mathcal{K}_r$. E
%each row has non-zero entries corresponding points to the edges (columns, classes) that contain it. 
In the final layer, each row of $\Tilde{M}_L$ corresponds to a weighting of vertices in $\mathcal{K}_r$ s.t.\ exactly two vertices get $-1$ weight and the rest $+1$, and the value at a column is the sum of the values of the vertices of the edge. The class-mean matrices of the SRG solution are illustrated in  Figure~\ref{fig:srg_illustration} for $L=4$ and $K=10$ (which gives $r=5$): we display $M_3, \Tilde{M}_4, \Tilde{M}_4^T\Tilde{M}_4$ %(corresponding to the adjacency matrix of $\mathcal{T}_5$) 
and, for comparison, also $\Tilde{M}_4^T\Tilde{M}_4$ of a DNC solution. Very similar solutions to SRG are shown for $K=6$ and $K=15$ in Figures~\ref{fig:app_srg_six} and~\ref{fig:app_srg_fifteen} of Appendix~\ref{app:experiments_dufm}.

%\begin{remark}
Let us highlight the properties of the SRG solution, which are crucial to outperform DNC. First, the rank of the intermediate feature and weight matrices is very low, only of order $\Theta(K^{1/2})$, since by construction there are only $r=\Theta(K^{1/2})$ linearly independent rows. This is contrasted with the DNC solution that has rank $K$ in all intermediate feature and weight matrices. The low rank of the SRG solution is due to the specific structure of the triangular graph, which has many eigenvalues equal to $-2$ that become 0 after adding twice a  diagonal matrix. Second, the definition of $\Tilde{M}_L$ ensures that $M_L=\sigma(\Tilde{M}_L)$ has full rank $K$. This allows the output $W_LM_L$ to also have full rank and, therefore, fit the identity matrix $I_K$, thus reducing the first term in the loss \eqref{eq:LDUFM}. Finally, the highly symmetric nature of the SRG solution balances the feature and weight matrices so as to minimize large entries and, therefore, the Frobenius norms, thus reducing the other terms in the loss \eqref{eq:LDUFM}. 
%\end{remark}

\paragraph{Main result.} For any $L$-DUFM problem (specified by $K, n$ and all the regularization parameters), let $\mathcal{L}_{SRG}, \mathcal{L}_{DNC}$ be the losses incurred by the SRG and DNC solutions, see Definitions~\ref{def:srg_solution_main} and \ref{def:dnc}, respectively. At this point we are ready to state our key result.

\begin{restatable}{theorem}{srgbetterthandnc}
\label{thm:dnc2notoptimal}
%Let $\mathcal{L}_{SRG}$ and $\mathcal{L}_{DNC}$ be the losses incurred by the SRG and DNC solutions, respectively. 
If $K\ge6, L\ge 4$ or $K\ge 10, L=3$ and $d_l\ge K$ for all $l$, then $\mathcal{L}_{SRG} < \mathcal{L}_{DNC}.$ Moreover, consider any sequence of $L$-DUFM problems for which $K\xrightarrow[]{}\infty$ so that $0.499 > \mathcal{L}_{DNC}$ for each problem. In that case, 
\begin{equation}\label{eq:grow}    
\frac{\mathcal{L}_{SRG}}{\mathcal{L}_{DNC}}=\mathcal{O}(K^{\frac{3-L}{2(L+1)}}).
\end{equation}
\end{restatable}
In words, as long as the number of classes and layers is not too small, the SRG solution always outperforms the collapsed one and the gap grows with the number of classes $K$.
%In particular, for $L\ge4,$ there is a $K$-dependent gap between the low-rank SRG solution loss and the full rank DNC solution loss.

The proof first computes the conditionally optimal values of $\norm{W_l}_F^2$ for both the SRG and DNC solutions. The specific structure of these solutions enables to calculate pseudoinverses of the intermediate features, thus enabling the explicit computation of the weight norms. All these values depend only on the singular values of the feature matrices, which are explicitly given by their scale. As a result, both $\mathcal{L}_{SRG}$ and $ \mathcal{L}_{DNC}$ are expressed via an optimization problem in a single scalar variable and, by comparing these problems, % again thanks to the special structure of the solution. % and spectral graph theory. 
%By a careful comparison of these two functions \marco{which functions?}, 
the statement follows. The details are deferred to  Appendix~\ref{app:theory_dnc2}. 

Although the argument requires $L= 3, K\ge10$ or $L\ge 4, K\ge6,$ the experiments in Appendix~\ref{app:experiments_dufm} show that the DNC solution is not optimal when $L\ge 4, K\ge 3$ or $L=3, K\ge 7$. Furthermore, for $L=3$ and large $K$, there is a large gap between $\mathcal{L}_{SRG}$ and $\mathcal{L}_{DNC}$ (even if \eqref{eq:grow} trivializes). % non-optimality of collapse also holds for $L\ge4, K\ge3$ and several $K<10$ values for $L=3,$ as we empirically demonstrate in DUFM experiments in Section~\ref{sec:experiments}. Similarly, while according to our proof, the difference between the loss of SRG solution and DNC solution is very thin for $L=3,$ our experiments suggest that the gap is, in fact, significant. 
For either $K=2$ or $L=2$, the DNC is optimal, as shown in \cite{tirer2022extended, sukenik2023deep}. % therefore only in this general setting is the discovery of the non-optimality of DNC possible. 

%This result has strong implications also for the last-layer NC. Although we do not prove it, based on our theory\marco{how exactly? This is not convincing} and numerous DUFM optimization experiments we strongly believe that under $L$-DUFM, the last-layer NC2 property does not occur in optimal solutions, since both the SRG solution presented here as well the low-rank solutions found in our experiments have highly uneven singular values and rather poor conditioning, which rules out orthogonality. Moreover, if NC2 would hold in optimal solution, based on our theory it would have needed to be created by applying ReLU to a low-rank matrix and making the result orthogonal. \marco{not very clear}

\section{Within-class variability collapse is still optimal}

While the DNC2 property conflicts with the low-rank bias, the same is not true for DNC1, as the within-class variability collapse supports a low rank. We show below that the last-layer NC1 property remains optimal for any $L$-DUFM problem. A proof sketch follows, with the complete argument deferred to Appendix \ref{app:theory_dnc1}.

\begin{restatable}{theorem}{nconefromdufm}
\label{thm:nc1_dufm}
The optimal solutions of the $L$-DUFM \eqref{eq:LDUFM} exhibit DNC1 at layer $L$, i.e., % In other words,
$$H_L^*=M_L^*\otimes \mathbf{1}_n^T$$ holds for any optimal solution $(H_1^*, W_1^*, \dots, W_L^*)$ of the $L$-DUFM problem. 
\end{restatable}

%In words, 
%\chl{If we have space issues, I'd be in favor of removing all proofs, even the "sketch" below, from the main body, but keep/add anything that gives an intution, such as simple explicit examples, visualizations etc}  \marco{I agree!}

\textit{Proof sketch:} Assume by contradiction that there exists an optimal solution of \eqref{eq:LDUFM} with regularization parameters $(\lambda_{H_1}, \lambda_{W_1}, \dots, \lambda_{W_L})$, denoted as $(H_1^*, W_1^*, \dots, W_L^*)$, which does not exhibit neural collapse at layer $L$. Then, we can construct two \textit{different} optimal solutions of the $L$-DUFM problem with $n=1$ and regularization parameters $(n\lambda_{H_1}, \lambda_{W_1}, \dots, \lambda_{W_L})$ of the form $(H_1^{(1)}, W_1^*, \dots, W_L^*)$ and $(H_1^{(2)}, W_1^*, \dots, W_L^*)$. These two solutions share the weight matrices, and $H_1$ (and, therefore, $H_L$) only differs in a single column (w.l.o.g., the first column). The optimality of these solutions can be proved using separability and symmetry of the loss function w.r.t. the columns of $H_1.$ 

Denote %for simplicity 
the first (differing) columns of $H_L^{(1)}$ and $H_L^{(2)}$ as $x$ and $y$, respectively. By exploiting the linearity of the loss function on a ray $\{th_{11}^1, t\ge 0\}$ for any $h_{11}^1$, a direct %By a relatively straightforward 
computation %one can 
gives that $x$ and $y$ are not aligned. % = \alpha y$ or $y=\alpha x$ for $\alpha \ge 0$ is not possible. %To see this one only needs linearity of the loss function on a ray $th_{11}^1, t\ge 0$ for any $h_{11}^1.$ 
%For convenience we denote $H_L^x := H_L^{(1)} \neq H_L^{(2)} =: H_L^y,$ \marco{why the extra notation? $H_L^x$ is as complicated as $H_L^{(1)}$}
Let $\mathcal{L}$ be the loss in \eqref{eq:LDUFM}. % and $\mathcal{L}_F$ its first term, corresponding to the fit to the labels. % Now we also know that $x, y$ are not aligned. 
By optimality of both solutions, we get 
\begin{equation}    \label{eq:der}
\at{\frac{\partial \mathcal{L}}{\partial W_L}}{(H_1, W_1, \dots, W_L)=(H_1^{(1)}, W_1^*, \dots, W_L^*)}=0=\at{\frac{\partial \mathcal{L}}{\partial W_L}}{(H_1, W_1, \dots, W_L)=(H_1^{(2)}, W_1^*, \dots, W_L^*)}.
\end{equation}
An application of the chain rule gives $$\frac{\partial \mathcal{L}}{\partial W_L}=\frac{\partial \mathcal{L}_F}{\partial \Tilde{H}_{L+1}}\frac{\partial \Tilde{H}_{L+1}}{\partial W_L}+\lambda_{W_L}W_L=\frac{\partial \mathcal{L}_F}{\partial \Tilde{H}_{L+1}}H_L^T+\lambda_{W_L}W_L,$$ where $\Tilde{H}_{L+1}$ is the model output and $\mathcal{L}_F$ the first term of $\mathcal{L}$, corresponding to the label fit. Plugging this back into \eqref{eq:der} and using that $W_L^*$ is the same in both expressions, we get $A(H_L^{(1)})^T=B(H_L^{(2)})^T,$
where we have denoted by $A$ and $B$ the partial derivatives $\frac{\partial \mathcal{L}_F}{\partial \Tilde{H}_{L+1}}$ evaluated at $(H_1^{(1)}, W_1^*, \dots, W_L^*)$ and $(H_1^{(2)}, W_1^*, \dots, W_L^*)$, respectively. As $\mathcal{L}_F$ is separable with respect to the columns of $H_l, \Tilde{H}_l$ for all $l$, the matrices $A, B$ can only differ in their first columns (denoted by $a, b$), and they are identical otherwise. %We denote these columns $a, b$ for $A, B,$ respectively. 
This implies that %Using that a matrix product is just a sum of outer products of to-each-other corresponding columns of the left matrix and rows of the right matrix, we see that from the above equation we get: 
$ax^T=by^T$. After some simple considerations and using that $x$ and $y$ are not aligned, we reach a contradiction, as we conclude that $x \neq y$ is impossible. \qed

%The full proof is in Appendix~\ref{app:theory}.
%We were not able to obtain the result for the unique optimality of DNC1 throughout all the layers. The same proof strategy as employed in Theorem~\ref{thm:nc1_dufm} would not work for further layers -- the role of $W_L$ is special since it can be computed explicitly as the ridge regression solution. However, 
The difficulty in extending Theorem \ref{thm:nc1_dufm} to a result on the unique optimality of DNC1 for all layers stems from the special role of $W_L$ as the loss is %always partially 
differentiable w.r.t. it. By considering a differentiable relaxation of ReLU, %activation function, 
we show below an approximate result for a \textit{relaxed} $L$-DUFM model. %, where a differentiable relaxation of the ReLU activation function is considered. 

\begin{restatable}{definition}{relaxeddufm}
\label{def:relaxed_dufm}
We denote by ReLU$_\epsilon$ (or $\sigma_\epsilon$) a function satisfying the following conditions: (i) $\sigma_\epsilon(x)=\sigma(x)$, for $x \in (-\infty, 0] \cup [\epsilon, \infty),$ (ii) $0<\sigma_\epsilon(x) < \sigma(x)$ for $x \in (0, \epsilon)$, and (iii) $\sigma_\epsilon$ is continuously differentiable with derivative bounded by a universal constant and strictly positive on $(0, \epsilon)$. % and uniformly bounded and %being bounded between 0 and 10 everywhere and 
%. %we mean any function\footnote{Formally we should define a set of such functions and then either use quantifiers or choose an element of the set but it would make the notation heavier.} for which the following conditions hold:
%\begin{itemize}
%    \item $\sigma_\epsilon(x)=\sigma(x) \hspace{1mm}\forall x \in (-\infty, 0] \cup [\epsilon, \infty),$
%    \item $0<\sigma_\epsilon(x) < \sigma(x) \hspace{1mm}\forall x \in (0, \epsilon),$
%    \item $\sigma_\epsilon$ is continuously differentiable everywhere with the derivative being bounded between 0 and 10 everywhere and strictly positive on $(0, \epsilon).$ 
%\end{itemize}    
\end{restatable}
    
%Let us denote by $L$-DUFM$_\epsilon$ the equivalent of \eqref{eq:LDUFM}, but with any fixed $\sigma_\epsilon$ instead of $\sigma$ used as a non-linearity. Now we can state the approximate strict optimality of DNC1 across all the layers:

\begin{restatable}{theorem}{approxdnconerelaxedrelu}
\label{thm:approx_dnc1_relaxed}
Denote by $L$-DUFM$_\epsilon$ the equivalent of \eqref{eq:LDUFM}, with $\sigma$ replaced by $\sigma_\epsilon$. Let $D=\max\{d_2, d_3, \dots, d_L\}$ and $\Bar{\lambda}=\lambda_{H_1}\lambda_{W_1}\dots\lambda_{W_L}$, with the regularization parameters upper bounded by  $1/(L+1)$. Then, for any globally optimal solution of the $L$-DUFM$_\epsilon$ problem, the distance between any two feature vectors of the same class in any layer is at most
\begin{align}\label{eq:relaxation_dist}
    \frac{6\epsilon \sqrt{D(L+1)}}{(L+1)^{L+1}\Bar{\lambda}\sqrt{n}}.
\end{align}
\end{restatable}

In words, as the activation function approaches ReLU (i.e., $\epsilon \to 0$), the within-class variability tends to $0$. The proof starts with a similar strategy as the argument of Theorem~\ref{thm:nc1_dufm} and then explicitly tracks the error due to replacing $\sigma$ with $\sigma_\epsilon$ through the layers. 
The full argument is in Appendix \ref{app:theory_dnc1}.

%Similarly as above, this theorem also doesn't work for $\epsilon=0$ and thus full DNC1 because of the non-differentiability of the ReLU. This cannot be overcome by taking subgradients or generalized gradients. However, the fact that this theorem holds for arbitrarily small $\epsilon$ suggests that the statement also holds for $\epsilon=0,$ although this is not a formal link. 

\section{Numerical results}\label{sec:experiments}

%For all experiments, w
We employ the standard DNC1 metric $\text{tr}(\Sigma_W)/\text{tr}(\Sigma_B),$ where $\Sigma_W, \Sigma_B$ are the within and between class variabilities. This is widely used in the literature \cite{tirer2022perturbation, rangamani2023feature, beaglehole2024average} and considered more stable than other metrics \cite{rangamani2023feature}. We measure the DNC2 metric as the condition number of $M_l$ for $l\ge 1$ \cite{sukenik2023deep}. %, however we only include figures of DNC2 when it is reached. 
We do not measure DNC3 here, as it is not well-defined for solutions that do not satisfy DNC2. 
%We consider three settings: \emph{(i)} the DUFM of Definition \ref{def:DUFM}, \emph{(ii)} data-dependent end-to-end training with DUFM-like regularization, and \emph{(iii)} standard end-to-end training. 
For end-to-end DNN experiments, we employ a model from \cite{sukenik2023deep} where an MLP with a few layers is attached to a ResNet20 backbone. The output of the backbone is then treated as unconstrained features, and DNC metrics are measured for the MLP layers. %We do this so that our experiments better correspond to our theory that is carried for fully-connected layers and to avoid residual connections and batch normalization that distort the DNC measurements.

\subsection{DUFM training}\label{ssec:dufm_experiments}

%\marco{can make these paragraphs instead of subsections to get some space}

We start with the $L$-DUFM model~\eqref{eq:LDUFM}, training both features and weights. 
%The purpose of these experiments is to investigate the types of solutions found by real training and how their structure (especially the rank), but also general training dynamics depend on the number of classes, depth, width, learning rate and regularization strength in the $L$-DUFM model. The general take-away is that the type of solutions found by GD does not only depend on the optimality, but also on the hyperparameters, proving the presence of various implicit biases. For instance, with some hyperparameters, the DNC solutions are reached despite being provably sub-optimal. We discuss these effects one-by-one, but due to space constraints defer most of the figures into Appendix~\ref{app:experiments}. 
In the top row of Figure~\ref{fig:loss_dnc1_training_progressions_main}, we consider a $4$-DUFM, with $K=10$ and $n=50$, presenting the training progression of the losses (left plot), the DNC1 metrics (center plot) and the singular values at convergence (right plot). %We use constant weight regularization of $0.004,$ learning rate $0.5$ and width $30.$ 

The results are in excellent agreement with our theory. First, the training loss outperforms that of the DNC solution, and it is rather close to that of the SRG solution. Second, DNC1 holds in a clear way in all layers, especially in the last ones. Third, the solution at convergence exhibits a strong low rank bias:   
the ranks of intermediate layers range from 5 to 8, and they are always the same in all intermediate layers within one run. For comparison, we recall that the intermediate layers of the DNC solution have full rank $K=10$. Third, for a few runs, the Gram matrices of the intermediate class means resulting from gradient descent training %(i.e., $M_3$) 
coincide with those of an SRG solution. Finally we highlight that, similarly to our theory, the solutions found in all our experiments in the entire Section~\ref{sec:experiments} have non-negative pre-activations in all intermediate layers of the MLP head except the last one. 

%We can see that the loss generally outperforms the DNC solution. The results are averaged over 10 runs and 
\begin{figure}
    \centering
    \includegraphics[width=0.32\textwidth]{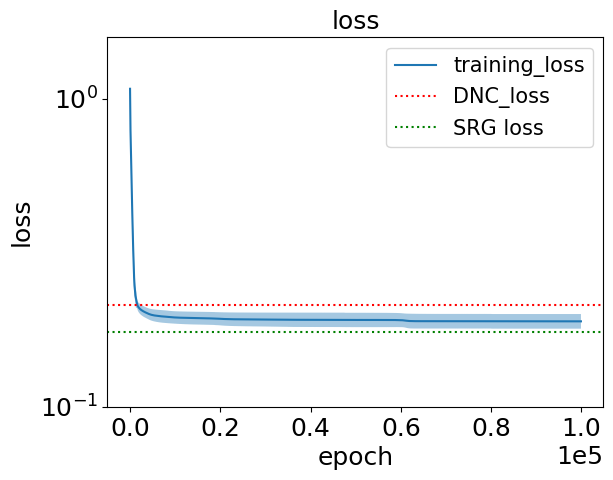}
    \includegraphics[width=0.32\textwidth]{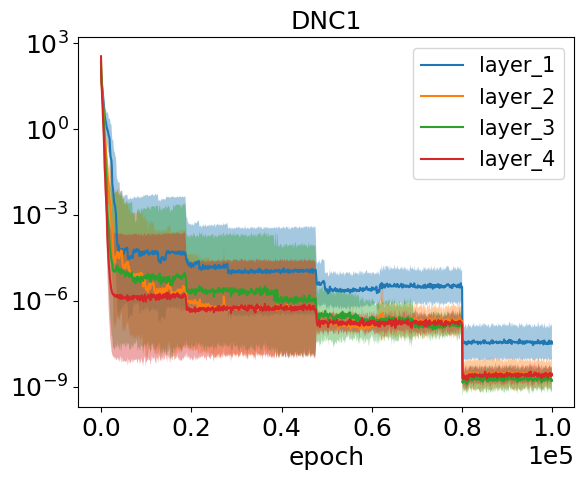}
    \includegraphics[width=0.32\textwidth]{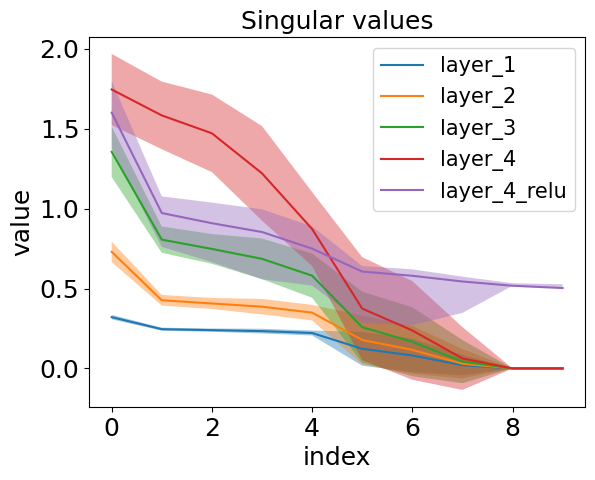}
    \includegraphics[width=0.32\textwidth]{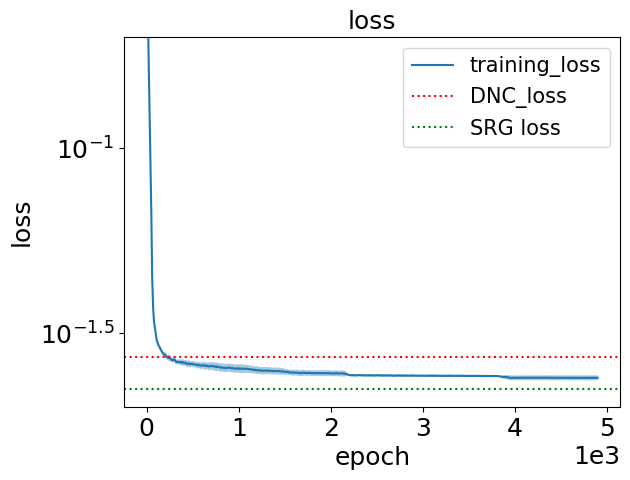}
    \includegraphics[width=0.32\textwidth]{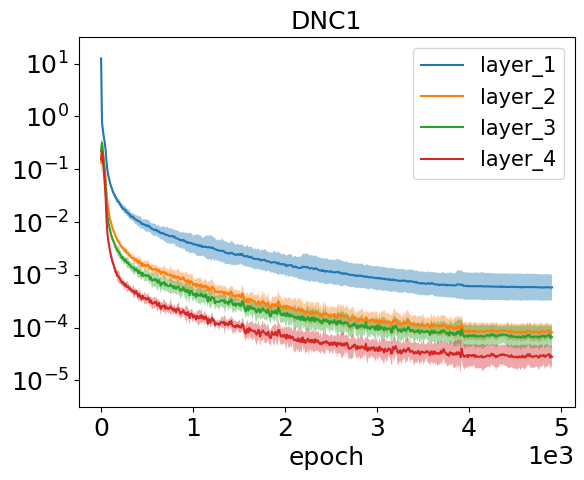}
    \includegraphics[width=0.32\textwidth]{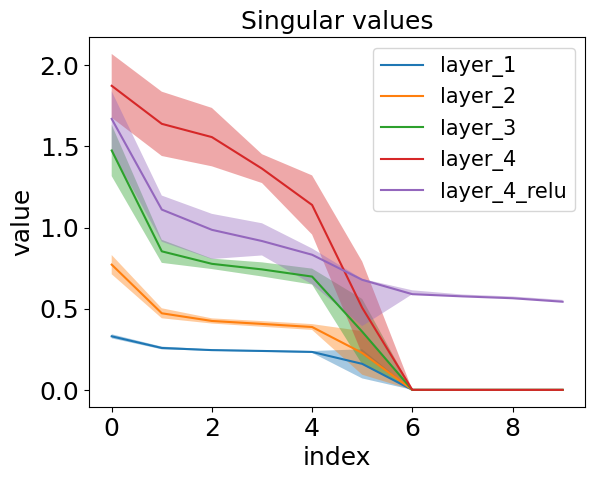}
\vspace{-.5em}    \caption{Training loss compared against DNC and SRG losses \textbf{(left)}, DNC1 metric training progression \textbf{(middle)} and singular value distribution at convergence \textbf{(right)}. \textbf{Top row:} $4$-DUFM training with $K=10$, $\lambda=0.004$ for all regularization parameters, learning rate of $0.5$ and width $30$. Results are averaged over 10 runs, and we show the confidence intervals at $1$ standard deviation. \textbf{Bottom row:} Training of a ResNet20 with a 4-layer MLP head on CIFAR10, using a DUFM-like regularization. We use weight decay $0.005$ except $\lambda_{H_1}=0.000005$ (to compensate for $n=5000$, which significantly influences the total regularization strength), learning rate $0.05$ and width $64$ for all the MLP layers. Results are averaged over 5 runs, and we show the confidence intervals at $1$ standard deviation.}
    \vspace{-1em}\label{fig:loss_dnc1_training_progressions_main}
\end{figure}

\vspace{-.5em}

\paragraph{Impact of number of classes and depth.} For $K=2$ or $L=2,$ we recover the results of \cite{sukenik2023deep, tirer2022extended} irrespective of other hyperparameters. The higher the number of classes, the more prevalent are low-rank solutions, while finding DNC solutions becomes challenging. The same holds for increasing the number of layers. For $L=3$ and low number of classes ($K\le6$), we weren't able to experimentally find solutions that would outperform DNC, which aligns nicely with the fact that SRG outperforms DNC only from $K=10$ for $L=3$. For large number of classes, the difference between the loss of low-rank solutions and the DNC loss is considerable already for $L=3$ and becomes even larger for higher $L.$ This is illustrated in the left plot of Figure~\ref{fig:dufm_ablations_main}. 

For $L\le 5$ and moderate number of classes ($K\le 30$), gradient descent solutions are as follows: until layer $L-1$, feature matrices share the same rank and have similar Gram matrices; intermediate activations are typically non-negative, and the ReLU has no effect; then, the rank jumps to $K$ after the final ReLU, as pre-activations are also negative. For large $L$ or large $K,$ the rank of the first few layers is low,  growing gradually in the last couple of layers (see Figure~\ref{fig:app_large_L} in Appendix~\ref{app:experiments_dufm}); the ReLU is active only in the final layers. This means that not only very low-rank solutions outperform DNC (as shown by our theory), but such solutions are routinely reached by gradient descent. %We highlight that, while our theThis behaviour is remarkable, since the fact that GD can find solutions for which the last ReLU is often more than squaring the rank is not directly motivated by our theoretical analysis, but rather a surprising practical ability. 

\vspace{-.5em}

\paragraph{Impact of weight decay and width.} %The weight decay and width can be seen as complementary in real training. 
While neither weight decay nor width influence Theorem \ref{thm:dnc2notoptimal} -- which shows that DNC is not optimal -- both quantities influence the nature of the solutions found by gradient descent. In particular, the stronger the weight decay, the lower the rank, see the middle plot in Figure~\ref{fig:dufm_ablations_main}. For very small weight decay, DNC is sometimes recovered; for very high weight decay, it is never recovered. The width has an opposite effect, see the right plot of Figure~\ref{fig:dufm_ablations_main}. For small width, low-rank solutions are much more likely to be found; large width has a strong implicit bias towards DNC and, thus, rank $K$ solutions. This means that, surprisingly, a larger width leads to a larger loss, since low-rank solutions exhibit a smaller loss than DNC. Thus, at least in DUFM, the infinite-width limit prevents gradient descent from finding a globally optimal solution, and sub-optimal solutions are reached with increasingly high probability.

\begin{figure}
    \centering
    \includegraphics[width=0.32\textwidth]{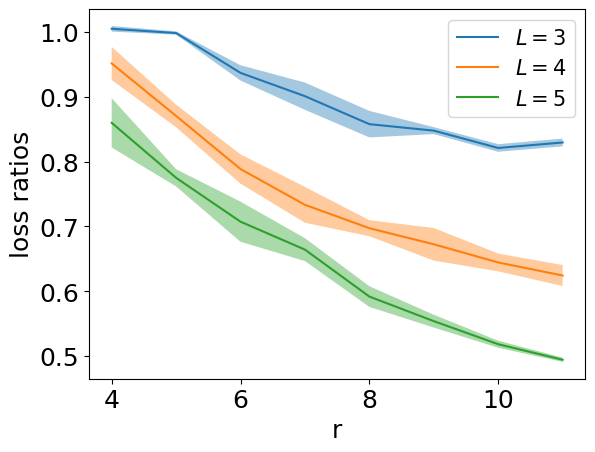}
    \includegraphics[width=0.32\textwidth]{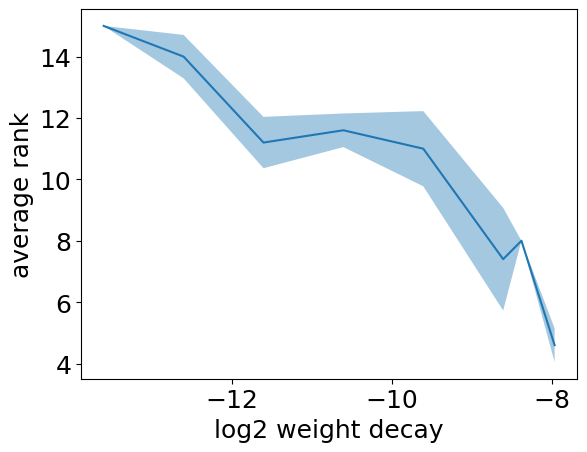}
    \includegraphics[width=0.32\textwidth]{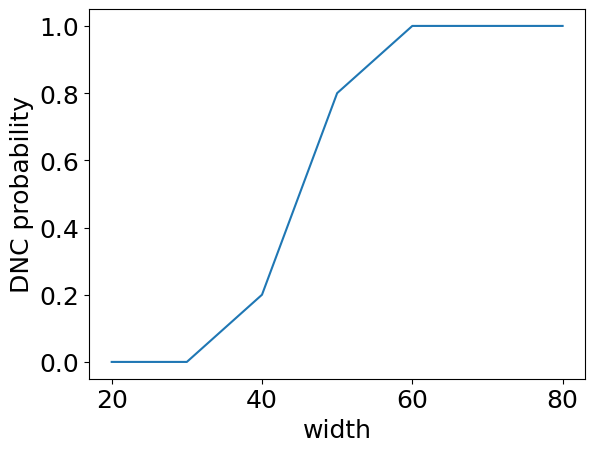}
\vspace{-.5em}    \caption{All experiments refer to the training of an $L$-DUFM model. Results are averaged over 5 runs, and we show the confidence intervals at $1$ standard deviation.  \textbf{Left:} Ratio between SRG and DNC loss ($\mathcal{L}_{SRG}/\mathcal{L}_{DNC}$), as a function of $r$, where the number of classes is $K= {r\choose2}$. Different curves correspond to different values of $L\in\{3, 4, 5\}$. \textbf{Middle:} Average rank at convergence, as a function of the weight decay in $\log_2$-scale, when $L=4$ and $K=15$. \textbf{Right:} Empirical probability of finding a DNC solution as a function of the width, when $L=4$ and $K=10$.}
    \vspace{-1em}\label{fig:dufm_ablations_main}
\end{figure}

\subsection{End-to-end experiments with DUFM-like regularization}\label{ssec:resnet_with_dufm_regularization}
%As a mid-way between DUFM training and standard end-to-end training, we 
Next, we train a DNN backbone with an MLP head, regularizing \textit{only} the output of the backbone and the layers of the MLP head (and not the layers of the backbone). This regularization is closer to our theory than the standard one, since we explicitly regularize the Frobenius norm of the unconstrained features. 
We also note that training with such a regularization scheme is easier than training with the standard regularization scheme. In the bottom row of Figure~\ref{fig:loss_dnc1_training_progressions_main}, we consider a ResNet20 backbone with a 4-layer MLP head trained on CIFAR10. %We use weight decay $0.005$ except $\lambda_{H_1}=0.000005$ (to compensate for $n=5000$, which significantly influences the total regularization strength), learning rate $0.05$ and width $64$ for all the MLP layers. 

The results agree well with our theory, and they are qualitatively similar to those of Section~\ref{ssec:dufm_experiments} for DUFM training. The DNNs consistently outperform the DNC loss, but still achieve DNC1. The ranks of class-mean matrices range from 5 to 6, and they are always the same in all intermediate layers within one run. Remarkably, the SRG solution was found by gradient descent also in this setting. 

Both weight decay and learning rate affect the average rank of the solutions found by gradient descent. Varying the width can lead to unexpected results, as it changes the ratio between the number of parameters in the MLP and that in the backbone, so the effect of the width is harder to interpret. Similar results can be seen on MNIST.

\subsection{End-to-end experiments}\label{ssec:standard_reg}
Finally, we perform experiments with standard regularization and the same architecture (i.e., DNN backbone plus MLP head) as in Section~\ref{ssec:resnet_with_dufm_regularization}. In particular, in Figure~\ref{fig:standard_reg_main} we consider a ResNet20 backbone with a 5-layer MLP head trained on CIFAR10 and MNIST with standard weight regularization. 

Overall, the results remain qualitatively similar to those discussed above. This demonstrates that, in spite of a different loss landscape compared to previous settings, the low-rank bias is still responsible for DNC2 not being attained. % for many hyperparameter choices.
Specifically, for CIFAR10, the rank in the third layer ranges between 8 and 9, and for MNIST ranges between 5 and 7; in contrast, the DNC solution has rank $K=10$. All DNNs display DNC1 across all layers. Remarkably, for the MNIST experiment the solution displayed in Figure~\ref{fig:standard_reg_main} found by gradient descent is the SRG solution (compare the gram matrices in bottom right plot of Figure \ref{fig:standard_reg_main} with the right-most plot of Figure \ref{fig:srg_illustration}). 

\begin{figure}
    \centering
    \includegraphics[width=0.32\textwidth]{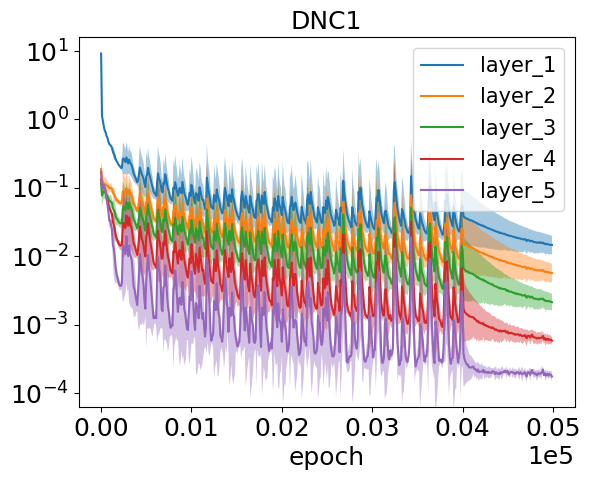}
    \includegraphics[width=0.32\textwidth]{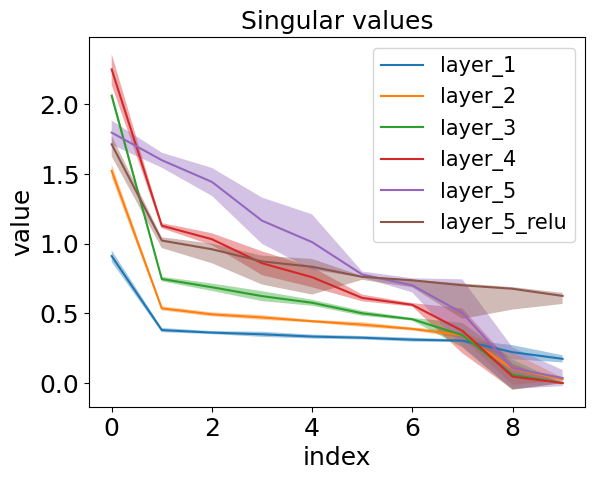}
    \includegraphics[width=0.32\textwidth]{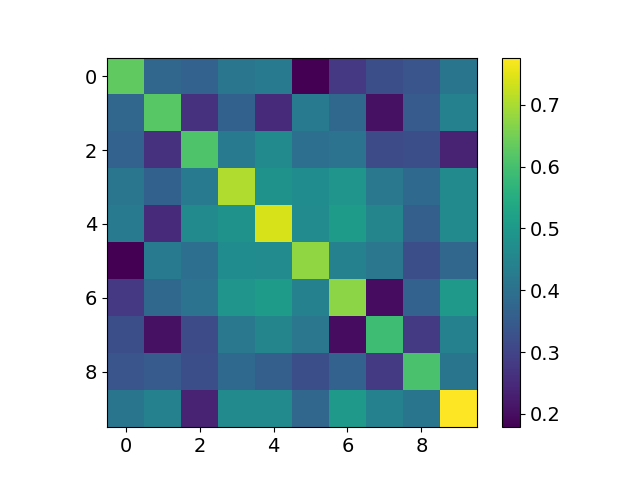}
    \includegraphics[width=0.32\textwidth]{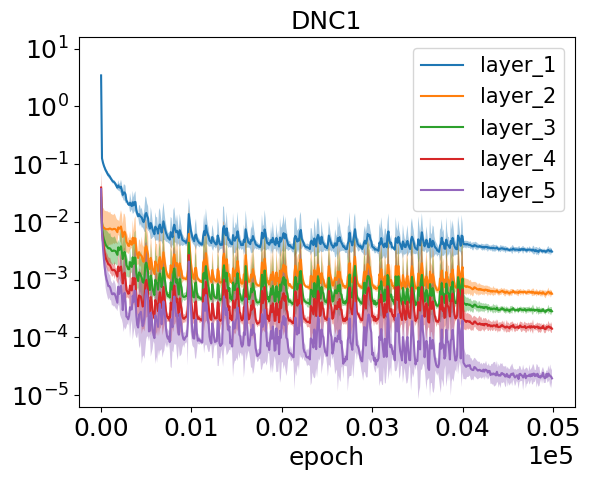}
    \includegraphics[width=0.32\textwidth]{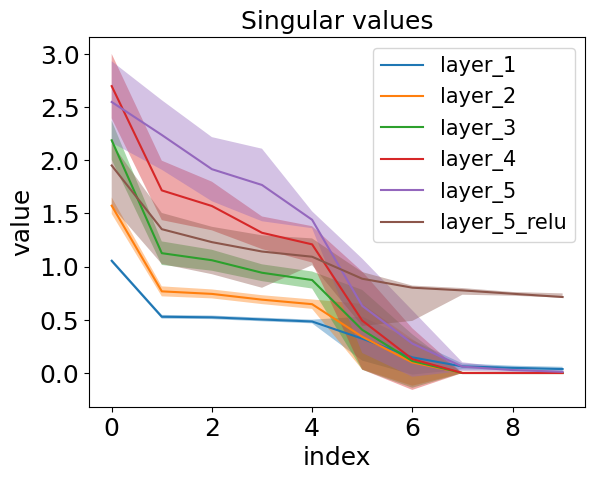}
    \includegraphics[width=0.32\textwidth]{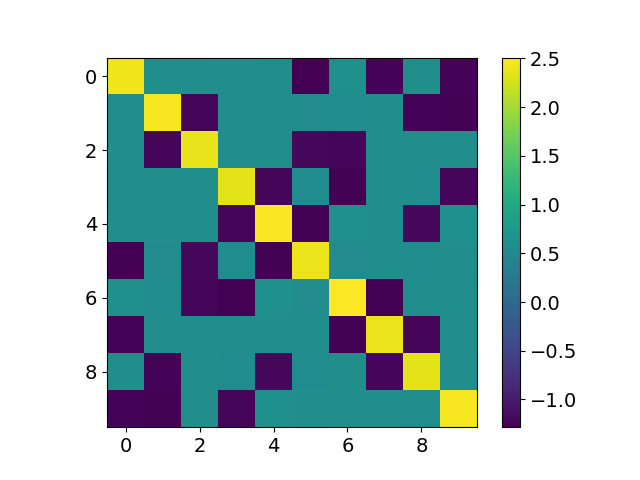}
\vspace{-.5em}
    \caption{Training of a ResNet20 with a 5-layer MLP head on CIFAR-10 (\textbf{top row}) and MNIST (\textbf{bottom row}), using the standard regularization. We pick a large weight decay ($0.08$ for CIFAR-10 and $0.04$ for MNIST) and a large learning rate ($0.005$ for CIFAR-10 and $0.01$ for MNIST).  Results are averaged over 5 runs, and we show the confidence intervals at $1$ standard deviation. \textbf{Left:} DNC1 metric training progression. \textbf{Middle:} Singular value distributions at convergence for all the layers. \textbf{Right:} Gram matrices of $M_3$ (CIFAR-10) and $M_5$ (MNIST).}
\vspace{-1em}
    \label{fig:standard_reg_main}
\end{figure}

The difficulty of the learning task plays a significant role in this setting: %, with MNIST and CIFAR10 exhibiting considerable differences in the structure of found solutions. 
when training on MNIST, it is rather easy to reach low-rank solutions and rather difficult to reach DNC solutions and the rank depends heavily on the regularization strength as shown in Figure~\ref{fig:app_abl_standard} of Appendix~\ref{app:experiments_standard}; when training on CIFAR-10, the weight decay needs to be high for the class mean matrices to be rank deficient. Moreover, the learning rate no longer exhibits a clear relation with the rank, since gradient descent diverges when the learning is too large. We also observe that the rank deficiency is the strongest in the mid-layer of the MLP head, creating a  ``rank bottleneck''. This can be seen by a closer look at the tails of the singular values, which better match zero at intermediate layers (the green and red curves corresponding to layers 3 and 4 have tails slightly lower than the other curves). In a more precise manner, we further measured effective ranks of all the layers in Figure~\ref{fig:standard_reg_main}. For instance, the effective ranks of CIFAR10 experiment layers are $8.96, 7.46, 6.88, 7.04, 7.73),$ which shows the middle layer is closest to a low hard rank matrix. The rank bottleneck is also mentioned in \cite{jacot2022implicit, jacot2023bottleneck, wen2024frequencies}. In fact, these works also measure extremely low ranks, but \cite{jacot2022implicit, jacot2023bottleneck} do it on synthetic data with very low inner dimension, while \cite{wen2024frequencies} focuses on fully convolutional architectures trained with CE loss and including biases. 

In summary, Figure \ref{fig:standard_reg_main} shows that both
%mportantly, the fact that the rank deficiency together with DNC1 is reached in this setting which is not covered with our theory proves that 
the low-rank bias and the optimality of DNC1 carry over to the standard training regime. This means that there are hyperparameter settings for which deep neural collapse, \textit{including in the very last layer}, is not reached (and likely not even optimal). Although the sub-optimality of DNC in the last layer is not proved formally, this phenomenon is supported by evidence across all experimental settings and further corroborated by our theory where our SRG construction is far from being DNC2-collapsed in the last layer.

%\marco{just a copy-paste for now, will need to integrate this better} 
%This result and our experiments have implications also for the last-layer NC2. Our combinatorial construction does itself not exhibit NC2 and neither does any low-rank solution found in any $L$-DUFM optimization experiment of Section~\ref{sec:experiments}. This suggests that the globally optimal low-rank solution of the $L$-DUFM problem does not exhibit even NC2, something that is predicted by shallower (D)UFM models.  

 % From the plots we clearly see that even in standard weight regularization trainings, the low-rank bias is still present and does not allow for DNC2 formation (the singular values are also uneven in the last layer so we cannot speak of a good NC2 either). For CIFAR10, the rank in the third layer is 8-9 and for MNIST it is 5-7. Remarkably, we also recovered our SRG construction from Section~\ref{sec:dnc_not_optimal} in MNIST runs, shown through the gram matrix.

\section{Conclusion}
In this work, we reveal that the deep neural collapse %, a phenomenon that has recently enjoyed attention from the researchers, 
is \textit{not} an optimal solution of the deep unconstrained features model -- the extension of the widely used unconstrained features model. This finding considerably changes our overall understanding of DNC, as all the previous models in simplified settings showed the global optimality of neural collapse and of its deep counterpart. The main culprit -- the low-rank bias -- makes the orthogonal frame property of DNC, and thus DNC as a whole, too high rank to be optimal. We demonstrate this low-rank bias across a variety of experimental settings, from DUFM training to end-to-end training with the standard weight regularization. While the structure of the Gram matrices of class means is not captured by orthogonal matrices (or by the ETF), the within-class variability collapse remains optimal. Our theoretical analysis proves this for the DUFM problem, and our numerical results showcase the phenomenon across various settings. % results We show that this is the case via our theoretical results for the DUFM traiing and for  and is exhibited in all our experiments. 

Our analysis focuses on the MSE loss, but we expect similar results to hold for the cross-entropy loss and, in particular, that the same SRG construction proposed here would still refute
the optimality of DNC. We leave as an open question whether 
%believe that similar results should hold for the cross-entropy loss, but we strongly believe that the same construction would refute the optimality of DNC also in that setting, with only moderate adjustments to our proofs (simply because the low-rank bias would be the same and only the shape of the fit loss would change). We leave a more interesting open question -- is 
DNC1 is strictly optimal across \textit{all} layers. While proving this would likely require new ideas, we note that \textit{none} of our experiments converged to a solution that would not be DNC1-collapsed. %Finally, confirming our measurements for wider range of architectures and dataset complexities would nicely complement our mostly theoretical findings. 

\section*{Acknowledgements}

M. M. is partially supported by the 2019 Lopez-Loreta Prize. This research was supported by the Scientific Service Units (SSU) of ISTA through resources provided by Scientific Computing (SciComp).

%\section*{References}
%\chl{TODO: make sure all references are up to date (no arXiv for published works), use the unified macros, make sure capitalization of abbreviations etc is correct, etc.}

\bibliographystyle{plain}
\bibliography{neurips_2024}  %%% Uncomment this line and comment out the ``thebibliography'' section below to use the external .bib file (using bibtex)

%%%%%%%%%%%%%%%%%%%%%%%%%%%%%%%%%%%%%%%%%%%%%%%%%%%%%%%%%%%%

\appendix

\newpage\section{Proofs}\label{app:theory}

\subsection{Low-rank solutions outperform deep neural collapse}\label{app:theory_dnc2}

We start by providing more detailed definitions of SRG and DNC solutions.  
%In this appendix, we first extend the definitions from Section~\ref{sec:dnc_not_optimal} to account for full generality and necessary detail. For completeness of context we recycle the Definition~\ref{def:triangular_graph}: \marco{no need to repeat it}

%\trianggraph*

%We proceed by constructing an explicit counterexample based on the triangular graph. For simplicity, we first define the counterexample for the case where the number of classes $K$ equals $r\choose2$ for some $r\ge 4.$ This makes the construction simpler because we can use only a triangular graph. \marco{no need to repeat it, just explain what you did not have a chance to in the body}

\begin{definition} \label{def:srg_solution}
Let $K= {r\choose2}$ for $r\ge4$. Then, a \textit{strongly regular graph (SRG)} solution of the $L$-DUFM problem \eqref{eq:LDUFM} is obtained by setting the matrices $(H_1, W_1, \dots, W_L)$ as follows. For all $l$, 
%for which feature matrices will have a form as defined below and the weight matrices are conditionally optimal given the form of the input and output to that layer. For all layers $l,$ 
the feature matrices $H_l, \Tilde{H}_l$ are DNC1 collapsed, \ie, $H_l=M_l\otimes \mathbf{1}_n^T, \Tilde{H}_l=\Tilde{M}_l \otimes \mathbf{1}_n^T.$ % Moreover, the class-mean matrices satisfy the following. 
For $2\le l \le L-1$, $M_l=\Tilde{M}_l$ and $M_l=A_lT_r$, where each row of $A_l$ is a multiple of the standard basis vector of dimension $r$ and the sum of squared multiples corresponding to one basis vector equals the same parameter $\alpha_l.$ In other words, each row of $M_l$ is a multiple of a row of $T_r$ and the sum of squared norms of the rows of $M_l$ corresponding to a row of $T_r$ is the same for each row of $T_r$. For $l=1,$ $W_1, M_1$ are any pair of matrices that minimize the objective conditionally on $M_2$ being defined as above. For $l\ge 2$, $W_l$ minimizes the objective conditional to the input and output to that layer. %Crucially, for $L,$ the construction is significantly different. Define 
Let $A_L^{(1)}$ be a $K \times r$ matrix where each row %is one of the set of rows that have 
has exactly two $(-1)$ entries and exactly $r-2$ $(+1)$ entries. Then, normalize the matrix $A_L^{(1)}T_r$ so that each row is unit norm and multiply from the left with $A_L^{(2)}$ of dimension $d_L \times K$, where each of the rows of $A_l^{(2)}$ is a multiple of a standard basis vector of dimension $K$ and the total sum of squared multiples corresponding to each basis vector is $\alpha_L.$ Then, $\Tilde{M}_L$ is obtained via this procedure, and $M_L=\sigma(\Tilde{M}_L).$ Finally, %for $2\le l\le L$, %
the parameters $\{\alpha_l\}_{l=2}^L$ satisfy %the following %: $q\ge0$ be a parameter. Then the following equations hold: 
\begin{equation}
    \begin{split}        
\label{eq:alphas}
    \alpha_l &= \frac{\left(\sqrt{n\lambda_{H_1}\lambda_{W_1}}\left(\sqrt{2}+\sqrt{(r-1)(r-2)}\right)\right)^{l-2}}{r^{l-2}\prod_{i=2}^{l-1}\lambda_{W_i}}q^l, \qquad 2\le l\le L-1,\\    %\alpha_2 &= q^2 \\
    %\alpha_3 &= \frac{\sqrt{n\lambda_{H_1}\lambda_{W_1}}\left(\sqrt{2}+\sqrt{(r-1)(r-2)}\right)}{r\lambda_{W_2}}q^3 \\
    %\alpha_4 &= \frac{\left(\sqrt{n\lambda_{H_1}\lambda_{W_1}}\left(\sqrt{2}+\sqrt{(r-1)(r-2)}\right)\right)^2}{r^2\lambda_{W_3}\lambda_{W_2}}q^4 \\
    %&\;\;\vdots \\
    %\alpha_{L-1} &= \frac{\left(\sqrt{n\lambda_{H_1}\lambda_{W_1}}\left(\sqrt{2}+\sqrt{(r-1)(r-2)}\right)\right)^{L-3}}{r^{L-3}\lambda_{W_{L-2}}\dots\lambda_{W_3}\lambda_{W_2}}q^{L-1} \\ 
    \alpha_L &= \frac{\left(\sqrt{n\lambda_{H_1}\lambda_{W_1}}\left(\sqrt{2}+\sqrt{(r-1)(r-2)}\right)\right)^{L-2}4((r-2)(r-3)+2)}{r^{L-1}(r-1)^2\prod_{i=2}^{L-1}\lambda_{W_i}}q^L,
    \end{split}
\end{equation}
and the parameter $q\ge 0$ is chosen to minimize the objective function in \eqref{eq:LDUFM}. %Finally, the $q,$ which is to this end the only variable whose value can influence the objective function in $L$-DUFM, is chosen so as to minimize the objective value. 
\end{definition}

%For completeness, we also define a neurally collapsed solution: \marco{this is also unnecessary. Add only the details that you plan to use in the proof. Otherwise add nothing.}

\begin{definition} \label{def:dnc_solution}
A \textit{deep neural collapse (DNC)} solution for any number of classes $K$ of the $L$-DUFM problem \eqref{eq:LDUFM} is obtained by setting the matrices $(H_1, W_1, \dots, W_L)$ as follows. % for which feature matrices will have a form as defined below and the weight matrices are conditionally optimal given the form of the input and output to that layer. 
For all $l,$ the feature matrices $H_l, \Tilde{H}_l$ are DNC1 collapsed, \ie, $H_l=M_l\otimes \mathbf{1}_n^T, \Tilde{H}_l=\Tilde{M}_l \otimes \mathbf{1}_n^T.$ For $2\le l \le L$,  $M_l=\Tilde{M}_l$ and $M_l^TM_l=\alpha_l I_K.$ For $l=1,$ $W_1, M_1$ are any pair of matrices that minimize the objective conditionally on $M_2$ being defined as above. For $l\ge 2$, $W_l$ minimizes the objective conditional to the input and output to that layer. Finally, the parameters $\{\alpha_l\}_{l=2}^L$ satisfy %Moreover, the multiplication parameters $\alpha_l, l=2, \dots, L$ must satisfy the following: $q\ge0$ be a parameter. Then the following equations hold:  
\begin{align*}
 %   \alpha_2 &= \frac{\lambda_{W_{L-1}}^2}{\lambda_{W_1}\lambda_{H_1}n}q^2 \\
  %  \alpha_3 &= \frac{\lambda_{W_{L-1}}^3}{\lambda_{W_2}\lambda_{W_1}\lambda_{H_1}n}q^3 \\
   % \alpha_4 &= \frac{\lambda_{W_{L-1}}^4}{\lambda_{W_3}\lambda_{W_2}\lambda_{W_1}\lambda_{H_1}n}q^4 \\
   % &\;\;\vdots \\
    \alpha_{l-1} &= \frac{\lambda_{W_{L-1}}^{l-1}}{\lambda_{H_1}n\prod_{i=1}^{l-1}\lambda_{W_i}}q^{l-1}, \qquad 2\le l\le L-1, \\ 
    \alpha_L &= \frac{\lambda_{W_{L-1}}^{L-1}}{\lambda_{H_1}n\prod_{i=1}^{L-2}\lambda_{W_i}}q^L,
\end{align*}
and the parameter $q\ge 0$ is chosen to minimize the objective function in \eqref{eq:LDUFM}. 
%Finally, the $q,$ which is to this end the only variable whose value can influence the objective function in $L$-DUFM, is chosen so as to minimize the objective value. 
\end{definition}

%\marco{very very repetitive until here. We should remove all that is already in the main text, explaining only the differences.}

Next, we define the SRG solution when $K\neq {r\choose2}$ for any $r$, and provide two constructions, each useful for different parts of the proof of Theorem~\ref{thm:dnc2notoptimal}.

\begin{definition}\label{def:srg_solution_complete}
A strongly regular graph (SRG) solution for $K\ge6$ of the $L$-DUFM problem \eqref{eq:LDUFM}, is obtained in one of the two following ways. 

\begin{enumerate}
    \item First, we take the largest $r$ s.t.\ $K\ge {r\choose2}$ and construct the SRG solution $(\Bar{H}_1, \Bar{W}_1, \dots, \Bar{W}_L)$ as in Definition~\ref{def:srg_solution} setting the number of classes to $r\choose2$. Next, we construct a DNC solution $(\Tilde{H}_1, \Tilde{W}_1, \dots, \Tilde{W}_L)$ as in Definition~\ref{def:dnc_solution} setting the number of classes to $K-{r\choose2}.$ Then, to construct $H_1$ of the SRG solution, we create it as a diagonal block matrix with number of columns equal to $K$ and number of rows equal to $\Bar{d}_1+\Tilde{d}_1,$ where these are the numbers of rows of the respective $H$ matrices; the first block is $\Bar{H}_1$, the second block is $\Tilde{H}_1$\footnote{The order is not important, both the rows and the columns can afterward be permuted if we accordingly permute also the weight matrices.}, and the off-diagonal blocks are zero matrices. Similarly, we extend the weight matrices such that, for any $l$, $W_l$ is a block diagonal matrix where the number of rows is $\Tilde{d}_{l+1}+\Bar{d}_{l+1}$ and the number of columns is $\Tilde{d}_l+\Bar{d}_l$; the first block is $\Bar{W}_l$, the second block is $\Tilde{W}_l$, and the off-diagonal blocks are zero matrices.

    \item First, we take the smallest $r$ s.t.\ $K \le {r\choose2}$ and construct the SRG solution as in Definition~\ref{def:srg_solution} setting the number of classes to $r\choose2$. Then, we just remove the ${r\choose2}-K$ columns of $H_1$ that achieve the highest individual fit losses (in case of a tie choose arbitrarily), and define the SRG solution as the original solution without these columns. 

\end{enumerate}
\end{definition}

We recall our main result and give the proof. %can present our result that if $L\ge3,$ the low-rank bias is so strong that the SRG solution outperforms the DNC solution. The rank of the SRG solution is $r,$ while the rank of the DNC solution is $r(r-1)/2.$ For this, denote $\mathcal{L}_{SRG}$ the loss incurred by the SRG solution, while $\mathcal{L}_{DNC}$ will denote the loss achieved by the DNC solution. \marco{also redundant, move the comment on the rank to the main text}

\srgbetterthandnc*

%\marco{no need to repeat the statement}

\begin{proof}
We start by considering the case $K={r\choose2}$ for some $r.$ Without loss of generality we can assume $n=1$, because all comparisons are between solutions that are by definition DNC1 collapsed, and the ratio $\mathcal{L}_{SRG}/\mathcal{L}_{DNC}$ in the theorem statement does not depend on $n.$

We first compute the loss of the SRG solution as in Definition~\ref{def:srg_solution} up to only one degree of freedom. Let us go term-by-term. The simplest to evaluate is $\frac{\lambda_{W_l}}{2}\norm{W_l}_F^2$ for $2\le l\le L-2.$ Using Lemma~\ref{lem:optimal_intermediate_W} (which relies on Lemma~\ref{lem:optimal_generic_wm}), we get $$\frac{\lambda_{W_l}}{2}\norm{W_l}_F^2=\frac{r\lambda_{W_l}}{2}\frac{\alpha_{l+1}}{\alpha_l}.$$ Similarly, for layer $L-1$, we use Lemma~\ref{lem:optimal_penultimate_W} (again relying on Lemma~\ref{lem:optimal_generic_wm}) to compute:
$$\frac{\lambda_{W_{L-1}}}{2}\norm{W_{L-1}}_F^2=\frac{r^2(r-1)^2\lambda_{W_{L-1}}}{8((r-2)(r-3)+2)}\frac{\alpha_L}{\alpha_{L-1}}.$$
Combining Lemma~\ref{lem:eigenvals_inter_features} with Lemma~\ref{lem:variational} we get:
$$\frac{\lambda_{W_1}}{2}\norm{W_1}_F^2+\frac{\lambda_{H_1}}{2}\norm{H_1}_F^2=\sqrt{\lambda_{W_1}\lambda_{H-1}}\left(\sqrt{2}+\sqrt{(r-1)(r-2)}\right)\alpha_2^{\frac{1}{2}}.$$ Finally, combining Lemma~\ref{lem:optimal_W_L} with Lemma~\ref{lem:eigenvalues_of_final_feature_matrix} we get:
\begin{align*}
    \frac{1}{2K}\norm{W_LM_L-I_K}_F^2+\frac{\lambda_{W_L}}{2}\norm{W_L}_F^2 &= \frac{\lambda_{W_L}}{2}\frac{1}{\frac{(r-2)(5r-19)}{(r-2)(r-3)+2}\alpha_L+\frac{r(r-1)}{2}\lambda_{W_L}} \\ &+ \frac{\lambda_{W_L}}{2}\frac{r-1}{\frac{2(r-3)^2}{(r-2)(r-3)+2}\alpha_L+\frac{r(r-1)}{2}\lambda_{W_L}} \\ &+\frac{\lambda_{W_L}}{2}\frac{\frac{r(r-3)}{2}}{\frac{2}{(r-2)(r-3)+2}\alpha_L+\frac{r(r-1)}{2}\lambda_{W_L}}.
\end{align*}
The total loss of the SRG solution is just the sum of all these terms, which is expressed in terms of $\alpha_2, \alpha_3, \dots, \alpha_L.$ We now verify that the choice in \eqref{eq:alphas} minimizes the loss, having set $q\equiv \alpha_2^{1/2}.$  %Now we will optimize $L-2$ degrees of freedom out of it. Let $q\equiv \alpha_2^{1/2}.$ 
To do so, we compute the partial derivatives of $\mathcal{L}$ w.r.t.\ the $\alpha_l$'s and set them to $0$:
\begin{align*}
0&=\frac{\partial \mathcal{L}}{\partial \alpha_2}=-r\lambda_{W_2}\frac{\alpha_3}{\alpha_2^2}+\sqrt{\lambda_{W_1}\lambda_{H-1}}\left(\sqrt{2}+\sqrt{(r-1)(r-2)}\right)\alpha_2^{-\frac{1}{2}} \iff \\
\frac{\alpha_3}{\alpha_2}&=\frac{\sqrt{\lambda_{W_1}\lambda_{H-1}}\left(\sqrt{2}+\sqrt{(r-1)(r-2)}\right)}{r\lambda_{W_2}}q.
\end{align*}
For $3\le l \le L-2$, we have 
\begin{align*}
0=\frac{\partial \mathcal{L}}{\partial \alpha_l}&=-\frac{r\lambda_{W_l}}{2}\frac{\alpha_{l+1}}{\alpha_l^2}+\frac{r\lambda_{W_{l-1}}}{2}\frac{1}{\alpha_{l-1}} \iff \\
\frac{\alpha_{l+1}}{\alpha_l}&=\frac{\lambda_{W_{l-1}}}{\lambda_{W_l}}\frac{\alpha_l}{\alpha_{l-1}}.
\end{align*}
And finally, for layer $L-1$, we have
\begin{align*}
0&=\frac{\partial \mathcal{L}}{\partial \alpha_{L-1}}=-\frac{\lambda_{W_{L-1}}}{2}\frac{r^2(r-1)^2}{4((r-2)(r-3)+2)}\frac{\alpha_L}{\alpha_{L-1}^2}+\frac{r\lambda_{W_{L-2}}}{2}\frac{1}{\alpha_{L-2}} \iff \\
\frac{\alpha_L}{\alpha_{L-1}}&=\frac{\lambda_{W_{L-2}}}{\lambda_{W_{L-1}}}\frac{4((r-2)(r-3)+2)}{r(r-1)^2}\frac{\alpha_{L-1}}{\alpha_{L-2}}.
\end{align*}
Denoting by $p=\sqrt{\lambda_{W_1}\lambda_{H_1}}\left(\sqrt{2}+\sqrt{(r-1)(r-2)}\right)$, we can %now recursively 
express these fractions as %follows: 
\begin{align*}
    \frac{\alpha_{l+1}}{\alpha_l}=\frac{p}{r\lambda_{W_l}}q, \qquad 2\le l\le L-2,\\ %\hspace{2mm} \frac{\alpha_4}{\alpha_3}=\frac{p}{r\lambda_{W_3}}q; \hspace{2mm} \frac{\alpha_5}{\alpha_4}=\frac{p}{r\lambda_{W_4}}q; \hspace{2mm} \dots; \\
    %\frac{\alpha_{L-1}}{\alpha_{L-2}}=\frac{p}{r\lambda_{W_{L-2}}}q; \hspace{2mm} 
    \frac{\alpha_L}{\alpha_{L-1}}=\frac{4((r-2)(r-3)+2)p}{r^2(r-1)^2\lambda_{W_{L-1}}}q,
\end{align*}
which gives the expressions in \eqref{eq:alphas}.
%From this we explicitly compute all the $\alpha_l:$
%\begin{align*}
%    \alpha_2=q^2; \hspace{2mm} \alpha_3=\frac{p}{r\lambda_{W_2}}q^3; \hspace{2mm} \alpha_4=\frac{p^2}{r^2\lambda_{W_3}\lambda_{W_2}}q^4; \hspace{2mm} \dots; \\
%    \alpha_{L-1}=\frac{p^{L-3}}{r^{L-3}\lambda_{W_{L-2}\dots\lambda_{W_2}}}q^{L-1}; \hspace{2mm} \alpha_L=\frac{p^{L-2}4((r-2)(r-3)+2)}{r^{L-1}(r-1)^2\lambda_{W_{L-1}}\dots\lambda_{W_2}}q^L.
%\end{align*}
%As a sanity check this also agrees with the formulas for $\alpha$s in Definition~\ref{def:srg_solution}. 
Finally, plugging this back into the loss function we get a univariate $q$-dependent function of the following form: 
\begin{align*}
    \mathcal{L}_{SRG}(q)&=\frac{\lambda_{W_L}}{2}\frac{1}{\frac{4(r-2)(5r-19)\left(\sqrt{\lambda_{W_1}\lambda_{H_1}}\left(\sqrt{2}+\sqrt{(r-1)(r-2)}\right)\right)^{L-2}}{r^{L-1}(r-1)^2\prod_{i=2}^{L-1}\lambda_{W_i}}q^L+\frac{r(r-1)}{2}\lambda_{W_L}} \\ &+ \frac{\lambda_{W_L}}{2}\frac{r-1}{\frac{8(r-3)^2\left(\sqrt{\lambda_{W_1}\lambda_{H_1}}\left(\sqrt{2}+\sqrt{(r-1)(r-2)}\right)\right)^{L-2}}{r^{L-1}(r-1)^2\prod_{i=2}^{L-1}\lambda_{W_i}}q^L+\frac{r(r-1)}{2}\lambda_{W_L}} \\ &+\frac{\lambda_{W_L}}{2}\frac{\frac{r(r-3)}{2}}{\frac{8\left(\sqrt{\lambda_{W_1}\lambda_{H_1}}\left(\sqrt{2}+\sqrt{(r-1)(r-2)}\right)\right)^{L-2}}{r^{L-1}(r-1)^2\prod_{i=2}^{L-1}\lambda_{W_i}}q^L+\frac{r(r-1)}{2}\lambda_{W_L}} \\ &+\frac{L}{2}\sqrt{\lambda_{W_1}\lambda_{H_1}}\left(\sqrt{2}+\sqrt{(r-1)(r-2)}\right)q.
\end{align*}
The loss of the DNC solution can be computed by a simple extension of the expression $(17)$ from \cite{sukenik2023deep}:
\begin{align*}
    \mathcal{L}_{DNC}(q)=\frac{\lambda_{W_L}}{2}\frac{\frac{r(r-1)}{2}}{\frac{\lambda_{W_{L-1}}^{L-1}}{\lambda_{H_1}\prod_{i=1}^{L-2}\lambda_{W_i}}q^L+\frac{r(r-1)}{2}\lambda_{W_L}}+\frac{L}{2}\frac{r(r-1)}{2}\lambda_{W_{L-1}}q.
\end{align*}
At this point, we split our analysis for $L=3$ and for $L>3$. We start with $L>3$, which is simpler. % is significantly simpler. We start by the simpler one. 

\textbf{Analysis for $L>3$.} As $K\ge 6$, the following upper bound holds: %we can upper bound  We can upper-bound $\mathcal{L}_{SRG}$ by assuming all the singular values of $M_L$ are equal to the smallest one as follows:
\begin{align*}
    \mathcal{L}_{SRG}(q)\le \Bar{\mathcal{L}}_{SRG}(q) := &\frac{\lambda_{W_L}}{2}\frac{\frac{r(r-1)}{2}}{\frac{8\left(\sqrt{\lambda_{W_1}\lambda_{H_1}}\left(\sqrt{2}+\sqrt{(r-1)(r-2)}\right)\right)^{L-2}}{r^{L-1}(r-1)^2\prod_{i=2}^{L-1}\lambda_{W_i}}q^L+\frac{r(r-1)}{2}\lambda_{W_L}} \\ &+\frac{L}{2}\sqrt{\lambda_{W_1}\lambda_{H_1}}\left(\sqrt{2}+\sqrt{(r-1)(r-2)}\right)q.
\end{align*}
Now we reparametrize $\mathcal{L}_{DNC}(q)$ and the upper bound on $\mathcal{L}_{SRG}(q)$, so that they look as similar as possible. By replacing $\lambda_{W_{L-1}}q$ with $q$, we get
\begin{align}\label{eq:the_dnc_loss}
    \underset{q\ge0}{\min} \hspace{1mm} \mathcal{L}_{DNC}(q) =\underset{q\ge0}{\min} \hspace{1mm} \frac{\lambda_{W_L}}{2}\frac{\frac{r(r-1)}{2}}{\frac{1}{\lambda_{H_1}\prod_{i=1}^{L-1}\lambda_{W_i}} q^L+\frac{r(r-1)}{2}\lambda_{W_L}}+\frac{L}{2}r\frac{(r-1)}{2} q.
\end{align}
Similarly, by replacing $\frac{\sqrt{\lambda_{W_1}\lambda_{H_1}}\left(\sqrt{2}+\sqrt{(r-1)(r-2)}\right)}{r}q$ with $q$, we get:
\begin{align*}
    \underset{q\ge0}{\min} \hspace{1mm} \Bar{\mathcal{L}}_{SRG}(q) = \underset{q\ge0}{\min} \hspace{1mm}  &\frac{\lambda_{W_L}}{2}\frac{\frac{r(r-1)}{2}}{\frac{8r}{\left(\sqrt{2}+\sqrt{(r-1)(r-2)}\right)^2(r-1)^2\lambda_{H_1}\prod_{i=1}^{L-1}\lambda_{W_i}}q^L+\frac{r(r-1)}{2}\lambda_{W_L}} \\ &+\frac{L}{2}rq.
\end{align*}
Next, by replacing $\left(\frac{8r}{\left(\sqrt{2}+\sqrt{(r-1)(r-2)}\right)^2(r-1)^2}\right)^{\frac{1}{L}}q$ with $q$, we get: 
\begin{align}\label{eq:srg_loss_param_1}
    \underset{q\ge0}{\min} \hspace{1mm} \Bar{\mathcal{L}}_{SRG}(q) = \underset{q\ge0}{\min} \hspace{1mm}  &\frac{\lambda_{W_L}}{2}\frac{\frac{r(r-1)}{2}}{\frac{1}{\lambda_{H_1}\prod_{i=1}^{L-1}\lambda_{W_i}}q^L+\frac{r(r-1)}{2}\lambda_{W_L}} \\ \nonumber &+\frac{L}{2}r\left(\frac{\left(\sqrt{2}+\sqrt{(r-1)(r-2)}\right)^2(r-1)^2}{8r}\right)^{\frac{1}{L}}q.
\end{align}
After the reparameterization, $\mathcal{L}_{DNC}(q)$ and $\Bar{\mathcal{L}}_{SRG}(q)$ have almost the same form except for the multiplier $\frac{r-1}{2}$ in the DNC case and $$\left(\frac{\left(\sqrt{2}+\sqrt{(r-1)(r-2)}\right)^2(r-1)^2}{8r}\right)^{\frac{1}{L}}$$ in the SRG case. Therefore, the inequality between $\mathcal{L}_{DNC}$ and $\Bar{\mathcal{L}}_{SRG}$ is fully determined by the inequality between these two terms. We can write: 
\begin{align*}
    \frac{r-1}{2} &> \left(\frac{\left(\sqrt{2}+\sqrt{(r-1)(r-2)}\right)^2(r-1)^2}{8r}\right)^{\frac{1}{L}} \iff \\
    r(r-1)^{L-2} &> 2^{L-3}\left(\sqrt{2}+\sqrt{(r-1)(r-2)}\right)^2
\end{align*}
We first solve it for $L=4$ and any $r\ge4$ (which is guaranteed by $K\ge 6$). We get the inequality $r(r-1)^2>2(r-1)(r-2)+4+4\sqrt{2(r-1)(r-2)}.$ This inequality is equivalent to $(r-1)(r^2-3r+4)>4+4\sqrt{2(r-1)(r-2)}$, which holds for all $r\ge 4$. % We see that the order of the LHS is bigger than the order of the RHS so it just suffices to try a few first terms and then the LHS gets overwhelmingly larger than the RHS for large values of $r.$ The inequality holds for all the small values of $r$ including $r=4.$ Now f

Compared to the case $L=4$, for general $L$ % it suffices to look at the corresponding $r$ and $L=4$ and see that to get to $L$ 
the LHS gets multiplied by $(r-1)^{L-4}$ and the RHS gets multiplied by $2^{L-4},$ which is smaller for $r\ge 4.$ Hence, the inequality holds as well.

\textbf{Analysis for $L=3$.} Here, we need a tighter upper bound than $\Bar{\mathcal{L}}_{SRG}(q)$. Thus, we write % is not tight-enough upper-bound. Therefore we forget what it was in the previous case and we will use a tighter quantity instead. Let us first upper-bound the $\mathcal{L}_{SRG}(q)$ by only assuming the biggest singular value is as small as all the second-biggest singular values. Thus we get: 
\begin{align*}
\mathcal{L}_{SRG}(q)&\le\frac{\lambda_{W_3}}{2}\frac{r}{\frac{8(r-3)^2\sqrt{\lambda_{W_1}\lambda_{H_1}}\left(\sqrt{2}+\sqrt{(r-1)(r-2)}\right)}{r^{2}(r-1)^2\lambda_{W_{2}}}q^3+\frac{r(r-1)}{2}\lambda_{W_3}} \\ &+\frac{\lambda_{W_3}}{2}\frac{\frac{r(r-3)}{2}}{\frac{8\sqrt{\lambda_{W_1}\lambda_{H_1}}\left(\sqrt{2}+\sqrt{(r-1)(r-2)}\right)}{r^{L-1}(r-1)^2\lambda_{W_2}}q^3+\frac{r(r-1)}{2}\lambda_{W_3}} \\ &+\frac{3}{2}\sqrt{\lambda_{W_1}\lambda_{H_1}}\left(\sqrt{2}+\sqrt{(r-1)(r-2)}\right)q.    
\end{align*}
We equivalently re-write this by extending both of the ratios by $\frac{r-2}{r-2}\frac{r-1}{r-1}$ and then moving $\frac{r-2}{r-1}$ to denominator. Thus, %is way we get: 
\begin{align*}
\mathcal{L}_{SRG}(q)\le \Tilde{\mathcal{L}}_{SRG}(q) &:=\frac{\lambda_{W_3}}{2}\frac{\frac{r(r-1)}{r-2}}{\frac{8(r-3)^2\sqrt{\lambda_{W_1}\lambda_{H_1}}\left(\sqrt{2}+\sqrt{(r-1)(r-2)}\right)}{r^{2}(r-1)(r-2)\lambda_{W_2}}q^3+\frac{r(r-1)}{2}\lambda_{W_3}} \\ &+\frac{\lambda_{W_3}}{2}\frac{\frac{r(r-3)(r-1)}{2(r-2)}}{\frac{8\sqrt{\lambda_{W_1}\lambda_{H_1}}\left(\sqrt{2}+\sqrt{(r-1)(r-2)}\right)}{r^{2}(r-1)(r-2)\lambda_{W_2}}q^3+\frac{r(r-1)}{2}\lambda_{W_3}} \\ &+\frac{3}{2}\sqrt{\lambda_{W_1}\lambda_{H_1}}\left(\sqrt{2}+\sqrt{(r-1)(r-2)}\right)q.
\end{align*}
Now, we perform the same reparametrizations as for $L>4$, with the only exception of treating $(r-1)(r-2)$ in the denominator of the denominator of the current ratios as $(r-1)^2$ in the previous case. Then, we have % we get to the 
\begin{align*}
\underset{q\ge0}{\min} \hspace{1mm} \Tilde{\mathcal{L}}_{SRG}(q) = \underset{q\ge0}{\min} \hspace{1mm}  &\frac{\lambda_{W_3}}{2}\frac{\frac{r(r-1)}{r-2}}{\frac{(r-3)^2}{\lambda_{H_1}\lambda_{W_1}\lambda_{W_2}}q^3+\frac{r(r-1)}{2}\lambda_{W_3}}\\ &+\frac{\lambda_{W_3}}{2}\frac{\frac{r(r-3)(r-1)}{2(r-2)}}{\frac{1}{\lambda_{H_1}\lambda_{W_1}\lambda_{W_2}}q^3+\frac{r(r-1)}{2}\lambda_{W_3}} \\ &+\frac{3}{2}r\left(\frac{\left(\sqrt{2}+\sqrt{(r-1)(r-2)}\right)^2(r-1)(r-2)}{8r}\right)^{\frac{1}{3}}q.
\end{align*}
Assume that the following inequality holds:
\begin{equation}\label{eq:the_wishful_inequality} \frac{1}{\frac{(r-3)^2}{\lambda_{H_1}\lambda_{W_1}\lambda_{W_2}}q^3+\frac{r(r-1)}{2}\lambda_{W_3}} \le \frac{1}{2}\frac{1}{\frac{1}{\lambda_{H_1}\lambda_{W_1}\lambda_{W_2}}q^3+\frac{r(r-1)}{2}\lambda_{W_3}}.\end{equation}
%would hold \textit{for all $q\ge0,$} which is not true. If this false statement \textit{was} true, than we could upper-bound the $\Tilde{\mathcal{L}}_{SRG}(q)$ by (a new notation):
Then,
\begin{align*}
\underset{q\ge0}{\min} \hspace{1mm} \Tilde{\mathcal{L}}_{SRG}(q) \le \underset{q\ge0}{\min} \hspace{1mm} \Bar{\mathcal{L}}_{SRG}(q) := \underset{q\ge0}{\min} \hspace{1mm} &\frac{\lambda_{W_3}}{2}\frac{\frac{r(r-1)}{2}}{\frac{1}{\lambda_{H_1}\lambda_{W_1}\lambda_{W_2}}q^3+\frac{r(r-1)}{2}\lambda_{W_3}} \\ &+\frac{3}{2}r\left(\frac{\left(\sqrt{2}+\sqrt{(r-1)(r-2)}\right)^2(r-1)(r-2)}{8r}\right)^{\frac{1}{3}}q.
\end{align*}
The only difference between the expression considered here and the one considered in the $L>3$ case is that here we have $(r-1)(r-2)$ instead of $(r-1)^2$ within the expression. By  comparing against $\mathcal{L}_{DNC}$ again, we get that $\Bar{\mathcal{L}}_{SRG}<\mathcal{L}_{DNC}$ if and only if 
$r(r-1)^2>\left(\sqrt{2}+\sqrt{(r-1)(r-2)}\right)^2(r-2)$. This is equivalent to $(r-1)(3r-4)>2(r-2)+2(r-2)\sqrt{2(r-1)(r-2)}$, which holds for all $r\ge 4$. %By using $(r-1)>(r-2)$ we can conditionally on $r>10$ prove this inequality by a very simple computation. For $r\le10,$ we can check the inequality manually. %Note that the inequality is asymptotically rather tight because the linear term on the LHS has a multiplier $3,$ while on the RHS it has a multiplier $2\sqrt{2}.$ Therefore we couldn't afford doing any loser bound in the previous steps to get this step work.

It remains to show that \eqref{eq:the_wishful_inequality} holds, and it suffices to do so for the minimizer $q^*$ of $\mathcal{L}_{DNC}$, as
$\min_{q\ge 0} \Tilde{\mathcal{L}}_{SRG}(q)\le \Tilde{\mathcal{L}}_{SRG}(q^*)\le\Bar{\mathcal{L}}_{SRG}(q^*)<\mathcal{L}_{DNC}(q^*)=\min_{q\ge 0} \mathcal{L}_{DNC}(q)$. % holds for the minimizer $q$ of the $\mathcal{L}_{DNC}(q)$.
Note that this is equivalent to %aUnfortunately, we derived this inequality assuming a statement which does not hold for all $q.$ Therefore the final technical step of our proof for $L=3$ is the following. When does the inequality~\eqref{eq:the_wishful_inequality} really hold? With a simple computation we get the condition:
$$\frac{r(r-1)}{2}\lambda_{W_3}<\frac{(r-3)^2-2}{\lambda_{W_1}\lambda_{W_2}\lambda_{H_1}}(q^*)^3.$$
Note that the minimum of the function in~\eqref{eq:the_dnc_loss} (having the same reparametrization as $\Bar{\mathcal{L}}_{SRG}$) -- if it is not at $q^*=0,$ in which case the statement of the theorem is trivial -- must come after the unique inflection point of the function. A direct computation yields that this inflection point satisfies %is at such a $q$ for which 
$$\frac{r(r-1)}{2}\lambda_{W_3}=\frac{2}{\lambda_{W_1}\lambda_{W_2}\lambda_{H_1}}q^3.$$ Therefore, the minimum of~\eqref{eq:the_dnc_loss} is attained at $q^*$ for which $$\frac{r(r-1)}{2}\lambda_{W_{3}}<\frac{2}{\lambda_{W_{2}}\lambda_{W_1}\lambda_{H_1}}(q^*)^3.$$ For $r\ge5,$ this implies that such a $q^*$ satisfies~\eqref{eq:the_wishful_inequality}, which concludes the argument for %. Therefore, the minimum of the SRG loss must be smaller because the $\Tilde{\mathcal{L}}_{SRG}(q)\le\Bar{\mathcal{L}}_{SRG}(q)<\mathcal{L}_{DNC}(q)$ holds for the minimizer $q$ of the $\mathcal{L}_{DNC}(q).$ This finishes the proof of the inequality part of the theorem for 
$K={r\choose2}.$

For a general $K$, the extension is rather simple. Note that the \textit{first type of} SRG solution in Definition~\ref{def:srg_solution_complete} is constructed in a way so that the losses attained by the SRG and DNC parts sum up. Therefore, we can split the analysis for the SRG and DNC parts. The DNC part obviously attains equal loss to the DNC solution. For the SRG part, the analysis done above applies, and the argument is complete.

It remains to show the statement on the asymptotic relationship between $\mathcal{L}_{SRG}$ and $\mathcal{L}_{DNC}$ for $K\xrightarrow[]{}\infty$ when $L\ge4$. Formally, we should consider sequences of the problems and label everything with an extra index corresponding to the order within the sequence. However, with an abuse of notation, we drop this indexing and switch to the $\mathcal{O}, \Theta$ notations whenever convenient. 

As before, we start by considering $K$ of the form $r\choose2$ for some $r$. Let %Now we define a few terms. 
%let 
$\Lambda=\lambda_{H_1}\prod_{i=1}^L\lambda_{W_i}$ and 
$\Psi(K)=\left(\frac{2^{L-3}\left(\sqrt{2}+\sqrt{(r-1)(r-2)}\right)^2}{r(r-1)^{L-2}}\right)^{\frac{1}{L}}$, where $r$ corresponds to the value s.t.\ ${r\choose2}=K$. We note that % and  With an abuse of notation we will also write $\Psi(K)$ as a function of the number of classes and take the $r$ such that ${r\choose2}=K$ for an evaluation. It is important to note that
$\Psi(K)=\Theta(K^{\frac{3-L}{2L}}).$ Since we are interested in the ratio $\frac{\mathcal{L}_{SRG}}{\mathcal{L}_{DNC}},$ we do a few changes and reparametrizations to the expressions in~\eqref{eq:the_dnc_loss} and~\eqref{eq:srg_loss_param_1}: we multiply both by 2, divide all terms in the left summands by $\lambda_{W_L},$ rewrite $\frac{r(r-1)}{2}$ as $K,$ plug in the defined quantities, divide all the terms in the left summands by $K$ and finally replace $\Lambda^{-\frac{1}{L}}K^{-\frac{1}{L}}q$ with $q$ %and then re-write the new expressions with $q:=\Bar{q}$ 
to obtain the following expression
\begin{equation}\label{eq:dnc_loss_asymptotics}
\frac{1}{q^L+1}+LK^{\frac{L+1}{L}}\Lambda^{\frac{1}{L}}q
\end{equation} for the DNC loss, and the following expression
\begin{equation}\label{eq:srg_loss_asymptotics}
\frac{1}{q^L+1}+LK^{\frac{L+1}{L}}\Lambda^{\frac{1}{L}}\Psi(K)q
\end{equation} for the SRG loss. Using a similar trick as in the previous analysis for $L=3$, we have that the minimum of the function in~\eqref{eq:dnc_loss_asymptotics} is achieved when $q>1$. Hence, we can lower bound~\eqref{eq:dnc_loss_asymptotics} by 
$$\frac{1}{3}q^{-L}+LK^{\frac{L+1}{L}}\Lambda^{\frac{1}{L}}q.$$ For this convex expression, we can find the optimal solution by setting to zero the derivative, which gives that the optimal solution is $(1+L)3^{-\frac{1}{L+1}}K\Lambda^{\frac{1}{L+1}}.$ Similarly, we can upper bound~\eqref{eq:srg_loss_asymptotics} by
$$q^{-L}+LK^{\frac{L+1}{L}}\Lambda^{\frac{1}{L}}\Psi(K)q$$ and after finding the optimal solution we get that it equals $(1+L)K\Lambda^{\frac{1}{L+1}}\Psi(K)^{\frac{L}{L+1}}.$ This allows us to conclude that
$$\frac{\mathcal{L}_{SRG}}{\mathcal{L}_{DNC}}=\mathcal{O}(K^{\frac{3-L}{2(L+1)}}).$$ 

To get the same formula when the number of classes is not of the form ${r\choose2},$ we only need simple adjustments. For this part, we will employ the upper-index notation to denote the number of classes $K$ to which the solution corresponds. First, note that the optimal value of~\eqref{eq:dnc_loss_asymptotics} is continuous in the coefficient in front of the linear term $q.$ Therefore, if $\mathcal{L}_{DNC}^K<0.499$, then, choosing the smallest $r$ for which $K\le{r\choose 2}:=\Bar{K}$, we see that $\mathcal{L}^{\Bar{K}}_{DNC}<0.5$ for the same set of regularization parameters, as $\frac{\Bar{K}}{K}\xrightarrow[]{K\xrightarrow[]{}\infty}1.$ Since the argument above does not need $\mathcal{L}_{DNC}<0.499$ but only $\mathcal{L}_{DNC}<0.5$, we can now use that $\mathcal{L}^{\Bar{K}}_{SRG}$ with the same regularization parameters is still $\mathcal{O}(K^{\frac{3-L}{2(L+1)}}).$ Finally, choosing the second construction in Definition~\ref{def:srg_solution_complete}, we construct the SRG solution for $K$ classes from the SRG solution for $\Bar{K}$ classes with the same regularization parameters (thus also the same regularization as for the DNC solution with $K$ classes). To conclude, it just suffices to see that $\mathcal{L}^{\Bar{K}}_{SRG}\ge \mathcal{L}^K_{SRG}$ because we removed columns from $H_1,$ decreasing its norm and the fit loss is at most as big because the columns with the worst fit loss were removed and the fit loss is an average over the columns. This concludes the proof also for general $K.$
\end{proof}

We conclude the section by stating and proving a few auxiliary lemmas that were used in the proof of Theorem \ref{thm:dnc2notoptimal}.

\begin{restatable}{lemma}{optimalgenericwm}
\label{lem:optimal_generic_wm}
Consider the following optimization problem: 
\begin{align}
    \underset{w}{\min} &\norm{w}^2 \\
    \text{s.t.}\hspace{1mm} & z^T=w^T A_lT_r.
\end{align}
Then, the value of the optimal solution is $$\frac{(r-1)^2}{\alpha_l(r-2)^2}z^TT_r^T\left(I_r-\frac{3r-4}{4(r-1)^2}\mathbf{1}_r\mathbf{1}_r^T\right)T_rz.$$
\end{restatable}
\begin{proof}
Multiplying the constraint with $T_r^T(T_rT_r^T)^{-1}$ from the right we get $z^TT_r^T(T_rT_r^T)^{-1}=w^TA_l.$ Now, we can use that the minimum $l_2$ norm solution of such a system can be computed by multiplying with the right pseudoinverse of $A_l.$ Thus, we get
$w=A_l(A_l^TA_l)^{-1}(T_rT_r^T)^{-1}T_rz=1/\alpha_l A_l(T_rT_r^T)^{-1}T_rz.$ Then the squared norm of this is simply: $$w^Tw=\frac{1}{\alpha_l^2} z^TT_r^T(T_rT_r^T)^{-1}A_l^TA_l(T_rT_r^T)^{-1}T_rz=\frac{1}{\alpha_l} z^TT_r^T(T_rT_r^T)^{-2}T_rz.$$
Now, we know that $$T_rT_r^T=\frac{r-2}{r-1}I_r+\frac{1}{r-1}\mathbf{1}\mathbf{1}^T=\frac{r-2}{r-1}\left(I+\frac{1}{r-2}\mathbf{1}\mathbf{1}^T\right).$$
This can be seen by looking at the structure of $\mathcal{K}_n$ where two vertices have exactly one edge between them. Now we can compute the the inverse of this matrix using the Sherman-Morrison formula 
$$\left(I+\frac{1}{r-2}\mathbf{1}\mathbf{1}^T\right)^{-1}=I-\frac{1}{2(r-1)}\mathbf{1}\mathbf{1}^T$$ and the square is:
$$\left(I+\frac{1}{r-2}\mathbf{1}\mathbf{1}^T\right)^{-2}=I-\frac{3r-4}{4(r-1)^2}\mathbf{1}\mathbf{1}^T.$$
Putting this all together, the proof is complete. 
\end{proof}

\begin{restatable}{lemma}{optimalintermediateW}
\label{lem:optimal_intermediate_W}
Let $2\le l \le L-2.$ Consider $\Tilde{M}_{l+1}, M_l$ as in Definition \ref{def:srg_solution} of  the SRG solution. Then, the following optimization problem: 
\begin{align}
    \underset{W}{\min} &\norm{W}_F^2 \\
    \text{s.t.}\hspace{1mm} & \Tilde{M}_{l+1}=WM_l
\end{align}
achieves optimal value of $\frac{r\alpha_{l+1}}{\alpha_l}.$
\end{restatable}
\begin{proof}
This is a direct consequence of Lemma~\ref{lem:optimal_generic_wm}. Denote $\gamma z_0$ a row from $\Tilde{M}_{l+1}$ which corresponds to the first row of $T_r$ and such that $\norm{z_0}=1.$ Then,  $T_r\gamma z_0=\gamma (1, 1/(r-1), \dots, 1/(r-1))^T.$ Directly evaluating the expression in Lemma~\ref{lem:optimal_generic_wm} will yield $\gamma^2 \alpha_l^{-1}$ for a single row and thus for all rows we get the value from the lemma statement.
\end{proof}

\begin{restatable}{lemma}{optimalpenultimateW}
\label{lem:optimal_penultimate_W}
Consider $\Tilde{M}_L, M_{L-1}$ as in Definition \ref{def:srg_solution} of the SRG solution. Then, the following optimization problem: 
\begin{align}
    \underset{W}{\min} &\norm{W}_F^2 \\
    \text{s.t.}\hspace{1mm} & \Tilde{M}_L=WM_{L-1}
\end{align}
achieves optimal value of $$\frac{\alpha_L}{\alpha_{L-1}}\frac{r^2(r-1)^2}{4((r-2)(r-3)+2)}.$$
\end{restatable}
\begin{proof}
We first need to characterize the rows of $\Tilde{M}_L$. Note that $\Tilde{M}_L$ comes from the multiplication of $T_r$ and $A_L^{(1)}$, see Definition~\ref{def:srg_solution}. This operation can be seen as weighting vertices of $\mathcal{K}_r$ (rows of $T_r$) and then looking at what the sum of the weights of adjacent vertices of each edge (column of $T_r$) is. We can easily see that the resulting vector has exactly one ``$-1$'' entry (for the edge corresponding to the two vertices given negative weight) and exactly ${r-2\choose 2}$ entries with ``$+1$''. Therefore, it must be scaled with the inverse of $s:=\sqrt{\frac{(r-2)(r-3)}{2}+1}$ to be unit norm. Now, let $z_1$ be one of these vectors with the negative edge between first two vertices. Then, $$T_rz_1=\left(-\frac{1}{s\sqrt{r-1}}, -\frac{1}{s\sqrt{r-1}}, \frac{r-3}{s\sqrt{r-1}}, \frac{r-3}{s\sqrt{r-1}}, \dots\right)^T,$$
which can be easily derived if we imagine doing edge-wise dot-product between two vertex weightings, one with two ``$-1$s'' for $z_1$ and the other type with one ``$+1$'' representing the rows of $T_r.$ For the other vectors in $\Tilde{H}_L$, the resulting vectors would be similar, except they would have the negative entries for different pairs of vertices of $\mathcal{K}_r.$ %Now we can easily apply Lemma~\ref{lem:optimal_generic_wm} to compute the optimal value. 
Note that %We have that
\begin{align*}
    z_1^TT_r^TT_rz&=\frac{2}{s^2(r-1)}+\frac{(r-3)^2(r-2)}{s^2(r-1)}, \\
    z_1^TT_r^T\mathbf{1}_r\mathbf{1}_r^TT_rz&=\frac{\left((r-3)(r-2)-2\right)^2}{s^2(r-1)}.
\end{align*}
Therefore, if we are optimizing for $\gamma z_1$, then an application of Lemma~\ref{lem:optimal_generic_wm} gives %\marco{$w$ undefined}
$$w^Tw=\gamma^2 \alpha_{L-1}^{-1} \frac{(r-1)^2}{(r-2)^2}\left(\frac{2}{s^2(r-1)}+\frac{(r-3)^2(r-2)}{s^2(r-1)}-\frac{3r-4}{4(r-1)^2}\frac{((r-3)(r-2)-2)^2}{s^2(r-1)}\right),$$
where $w$ denotes a row of $W$.
Simplifying this expression, we get $$w^Tw=\gamma^2 \alpha_{L-1}^{-1}\frac{r(r-1)}{2((r-2)(r-3)+2)}.$$
To conclude, it suffices to 
%Now, to compute $\norm{W_{L-1}}_F^2$ of the optimal $W_{L_1}$ \marco{undefined/typo} we must 
sum up $\frac{r(r-1)}{2}$ such rows having total $l_2$ norm squared $\alpha_L$ (and, hence, $\gamma^2=\alpha_L$), which gives $$\norm{W}_F^2=\frac{\alpha_L}{\alpha_{L-1}}\frac{r^2(r-1)^2}{4((r-2)(r-3)+2)},$$
and concludes the proof.
\end{proof} 

\begin{restatable}{lemma}{eigenvaluesofinterfeatures}
\label{lem:eigenvals_inter_features}
For $2\le l \le L-1$, consider $M_l, \tilde M_l$ as in Definition \ref{def:srg_solution} of the SRG solution. Then, $M_l= \tilde M_l$ and the eigenvalues of $M_l^TM_l$ are:
\begin{align*}
    \mu_1 &= 2\alpha_l \hspace{3mm} \text{with multiplicity } \hspace{1mm} 1, \\
    \mu_2 &= \frac{r-2}{r-1}\alpha_l \hspace{3mm} \text{with multiplicity } \hspace{1mm} r-1, \\
    \mu_3 &= 0 \hspace{3mm} \text{with multiplicity } \hspace{1mm} \frac{r(r-3)}{2}. \\
\end{align*}
\end{restatable}
\begin{proof}
From the definition, it readily follows that $M_l= \tilde M_l$, so let us compute $M_l^TM_l.$ Looking at any row of $M_l,$ we see that it has non-negative equal entries of value $
\sqrt{\alpha_{ij}}/\sqrt{r-1},$ where $\sum_{j}\alpha_{ij}=\alpha_l$ on all the $r-1$ edges of $\mathcal{K}_r$ that contain the vertex corresponding to that row. Therefore, by definition of $\alpha_l$, the sum of squares of all entries corresponding to one row type within any column is $\alpha_l/(r-1)$. Each edge (and, thus, column) contains exactly two vertices, thus the diagonal elements of $M_l^TM_l,$ which are the $l_2$ norms squared of the columns of $M_l$, are simply equal to $\frac{2\alpha_l}{r-1}.$ There are two possible off-diagonal values. One is for the pairs of columns that correspond to edges that share a vertex and one is for those pairs that do not share a vertex. The pairs of columns whose edges do not share a vertex do not have any entries which would \textit{both} be jointly positive, because either a vertex does not belong to one edge or to the other. Therefore, the value of off-diagonal entries corresponding to such pairs is simply 0. On the other hand, there is exactly one vertex that has non-zero values for \textit{both} edges corresponding to columns whose edges do share a vertex -- it is the shared vertex. Therefore the value of off-diagonal entries of this type is $\frac{\alpha_L}{r-1}.$ Crucially, the structure of the off-diagonal entries is fully determined by the  graph $\mathcal{T}_r$, because two edges in $\mathcal{K}_r$ share a vertex if and only if they are connected in the graph $\mathcal{T}_r$. Therefore, $M_l^TM_l$ can be written as a weighted sum of $I_K$ and the adjacency matrix $G_r$ of $\mathcal{T}_r$, where the weight of $I_K$ simply corresponds to the size of the diagonal term and the weight of $G_r$ to the positive off-diagonal term. In conclusion, we get
\begin{align*}
    M_l^TM_l=\frac{2\alpha_l}{r-1}I_K+\frac{\alpha_l}{r-1}G_r.
\end{align*}
As $\mathcal{T}_r$ is a strongly regular graph with parameters $(r(r-1)/2, 2(r-2), r-2, 4)$, $G_r$ has a single eigenvalue equal to $2(r-2)$, $r-1$ eigenvalues equal to $r-4$ and $r(r-3)/2$ eigenvalues equal to $-2$, which concludes the proof. % with multiplicity .$
%from which the expression of the eigenvalues readily follows. %Now it is easy to compute the eigenvalues. The summation with $I_K$ only offests all the eigenvalues, while the spectrum of $G_r$ is well-known. Putting it together we conclude with the lemma statement.  
\end{proof}

\begin{restatable}{lemma}{eigenvaluesoffinalfeatures}
\label{lem:eigenvalues_of_final_feature_matrix}
Consider $M_L$ as in Definition \ref{def:srg_solution} of the SRG solution. Then, the eigenvalues of $M_L^TM_L$ are: 
\begin{align*}
    \mu_1 &= \frac{(r-2)(5r-19)\alpha_L}{(r-2)(r-3)+2} \hspace{3mm} \text{with multiplicity} \hspace{1mm} 1, \\
    \mu_2 &= \frac{2(r-3)^2\alpha_L}{(r-2)(r-3)+2} \hspace{3mm} \text{with multiplicity} \hspace{1mm} r-1, \\
    \mu_3 &= \frac{2\alpha_L}{(r-2)(r-3)+2} \hspace{3mm} \text{with multiplicity} \hspace{1mm} \frac{r(r-3)}{2}. \\
\end{align*}
\end{restatable}
\begin{proof}
Let us compute $M_L^TM_L.$ Looking at any row of $M_L,$ we see that it has non-negative equal entries of value $
\sqrt{\alpha_{ij}} /s,$ where $s=\sqrt{\frac{(r-2)(r-3)}{2}+1}$ and $\sum_{j}\alpha_{ij}=\alpha_L$ on all edges in a subgraph of $\mathcal{K}_r$ of size $r-2.$ Therefore, by definition of $\alpha_L$, the sum of squares of all entries corresponding to one row type within any column is $\alpha_L/s^2$. However, not all row types have non-zero value on any particular column. Namely, a row type will only have non-zero value on a column, if the row-type corresponds to such a subgraph of $\mathcal{K}_r$, which is disjoint with the edge corresponding to the column. This is because all edges outside the complete subgraph of $r-2$ vertices corresponding to the row type are assigned 0 in $M_L.$ Therefore, the number of row types that assign non-zero value in a particular column is equal to the number of $r-2$ vertex sets. This corresponds to the number of edges in $\mathcal{K}_r$, which is equal to ${r-2\choose2}=s^2-1.$ Thus, the diagonal elements of $M_L^TM_L,$ which are the $l_2$ norms squared of the columns of $M_L$, are simply equal to $\frac{(s^2-1)\alpha_L}{s^2}.$ There are two possible off-diagonal values. One is for the pairs of columns that correspond to edges that share a vertex, and one is for those pairs that don't share a vertex. Let us compute the number of row types assigning positive value to \textit{both} of these columns jointly. Using the same interpretation, the columns correspond to edges and only row types that correspond to $r-2$ vertex subsets disjoint with them assign positive value to the column of that edge. If we want this to be satisfied for both rows jointly, we need to take the intersection of those $r-2$ vertex subsets, which in this case will result in an $r-3$ vertex subset. Thus, exactly $r-3\choose2$ row types will jointly assign a positive value. Therefore, we have value $\frac{(r-3)(r-4)\alpha_L}{2s^2}$ on these off-diagonal entries. For the pairs of columns that correspond to edges with disjoint vertices, the same intersection will now yield a set of vertices of size only $r-4.$ Therefore, the value of this off-diagonal entry is $\frac{(r-4)(r-5)\alpha_L}{2s^2}.$ Crucially, the structure of the off-diagonal entries is fully determined by the graph $\mathcal{T}_r$, because two edges in $\mathcal{K}_r$ share a vertex if and only if they are connected in the graph $\mathcal{T}_r$. Therefore, $M_L^TM_L$ can be written as a weighted sum of $\mathbf{1}_K\mathbf{1}_K^T, I_K$ and the adjacency matrix $G_r$ of  $\mathcal{T}_r$. The weights can be determined as follows: we first subtract a multiple of $\mathbf{1}_K\mathbf{1}_K^T$ to make the smaller off-diagonal entry of $M_L^TM_L$ zero, then we subtract what is left of the diagonal and we take the rest to be a multiple of $G_r.$ In conclusion, we get %the following:
\begin{align*}
    M_L^TM_L=&\frac{(r-4)(r-5)\alpha_L}{2s^2}\mathbf{1}_K\mathbf{1}_K^T\\+&\left(\frac{(s^2-1)\alpha_L}{s^2}-\frac{(r-4)(r-5)\alpha_L}{2s^2}\right)I_K\\+&\left(\frac{(r-3)(r-4)\alpha_L}{2s^2}-\frac{(r-4)(r-5)\alpha_L}{2s^2}\right)G_r.
\end{align*}
As $\mathcal{T}_r$ is a strongly regular graph with parameters $(r(r-1)/2, 2(r-2), r-2, 4)$, $G_r$ has a single eigenvalue equal to $2(r-2)$, $r-1$ eigenvalues equal to $r-4$ and $r(r-3)/2$ eigenvalues equal to $-2$. The summation with $I_K$ only shifts all the eigenvalues. The term $\mathbf{1}_K\mathbf{1}_K^T$ has only one non-zero eigenvalue, and the eigenvector is identical to that of the eigenvector corresponding to the dominant eigenvalue of $G_r$. This concludes the proof. % therefore these two matrices are jointly diagonalizable and we can sum up their eigenvalues. Finally, the spectrum of $G_r$ is well-known. Putting it together we conclude with the lemma statement.  
\end{proof}

\begin{restatable}{lemma}{optimalridgeregression}
\label{lem:optimal_W_L}
Assuming DNC1, let $M_L$ be the mean matrix of the last layer. Let $M_L=U\Sigma V^T$ be the full SVD of $M_L$ and let $\sigma_i, \hspace{1mm} i\in[K]$, be the singular values of $M_L.$ Then, the following optimization problem:
\begin{align*}
    \underset{W_L}{\min} \frac{1}{2K}\norm{W_LM_L-I_K}_F^2+\frac{\lambda_{W_L}}{2}\norm{W_L}_F^2
\end{align*}
attains the minimum of $$\frac{\lambda_{W_L}}{2}\sum_{i=1}^K \frac{1}{\sigma_i^2+K\lambda_{W_L}}.$$
\end{restatable}
\begin{proof}
The proof consists in a direct computation. Let $W_L^*$ denote the minimizer. Computing the gradient and setting it to $0$ gives that % the optimal Note that, after writing $M_L$ instead of the rest of the model, the objective above is %nothing else but
%part of the loss in the $L$-DUFM problem with $n=1$ and without the regularization terms of the other weight matrices. Therefore, %the following computation is relevant for this problem:
\begin{align*}
 %   \frac{\partial \mathcal{L}}{\partial W_L} = \frac{1}{K}\left(W_LM_L-I_K\right)M_L^T+\lambda_{W_L}W_L &\overset{!}{=} M_L^T \iff \\
    W_L^*=M_L^T(M_LM_L^T+\lambda_{W_L}KI_{d_L})^{-1}%&=W_L^* \implies \\
    = V\Sigma^T(\Sigma\Sigma^T+\lambda_{W_L}KI_{d_L})^{-1}U^T, %&=W_L^*.
\end{align*}
which readily implies that %Now it is easy to compute the relevant summands: 
\begin{align*}
    \norm{W_L^*}_F^2&=\sum_{i=1}^K \frac{\sigma_i^2}{(\sigma_i^2+\lambda_{W_L}K)^2}, \\
    \norm{W_LM_L-I_K}_F^2 &= \norm{V\Sigma^T(\Sigma\Sigma^T+\lambda_{W_L}KI_{d_L})^{-1}\Sigma V^T-VV^T}_F^2 \\
    &=\sum_{i=1}^K \left(\frac{\sigma_i^2}{\sigma_i^2+\lambda_{W_L}K}-1\right)^2=\sum_{i=1}^K \frac{K^2\lambda_{W_L}^2}{(\sigma_i^2+\lambda_{W_L}K)^2},
\end{align*}
thus concluding the argument. %Putting it together we get the lemma statement. 
\end{proof}

\begin{restatable}{lemma}{variational}
\label{lem:variational}
The optimization problem
\begin{equation}\label{eq:pbvar}
\min_{A,B; C=AB} \quad
\frac{\lambda_A}{2} \norm{A}_F^2 + \frac{\lambda_B}{2} \norm{B}_F^2    
\end{equation}
attains the minimum of $\sqrt{\lambda_A\lambda_B} \norm{C}_*$, and the minimizers are of the form $A^* = \gamma_A U\Sigma^{1/2}R^T, B^* = \gamma_B R\Sigma^{1/2}V^T$. Here, the constants $\gamma_A, \gamma_B$ only depend on $\lambda_A, \lambda_B$; $U\Sigma V^T$ is the SVD of $C$; and $R$ is an orthogonal matrix.
\end{restatable}
\begin{proof}
See Lemma C.1 of \cite{tirer2022extended}. 
\end{proof}

\subsection{No within-class variability is still optimal}\label{app:theory_dnc1}

\nconefromdufm* 
\begin{proof}
\textbf{Step 1: Reduction to $n=1.$} In the first step, assume by contradiction that there exists an optimal solution of \eqref{eq:LDUFM} with regularization parameters $(\lambda_{H_1}, \lambda_{W_1}, \dots, \lambda_{W_L})$ denoted as $(H_1^*, W_1^*, \dots, W_L^*)$ which does not exhibit deep neural collapse at layer $L$. This means that there exist indices $c, i, j$ s.t.\ $h_{ci}^L \neq h_{cj}^L.$ Let us construct two solutions of the $n=1$ $L$-DUFM. They will share the weight matrices which will equal $(W_1^*, \dots, W_L^*)$ -- the weight matrices of the original solution. To construct the features, for every class except the $c$-th, pick any sample and share it between both solutions. For the class $c,$ take the samples $h_{ci}^L, h_{cj}^L$ and put one in one solution and the other one in the other solution. Denote $H_1^{(1)}, H_1^{(2)}$ the two $n=1$ sample matrices. It is not hard to see that both $(H_1^{(1)}, W_1^*, \dots, W_L^*)$ and $(H_1^{(2)}, W_1^*, \dots, W_L^*)$ are optimal solutions of \eqref{eq:LDUFM} with regularization parameters $(n\lambda_{H_1}, \lambda_{W_1}, \dots, \lambda_{W_L}).$ To prove it, assume by contradiction that, without loss of generality, $(H_1^{(2)}, W_1^*, \dots, W_L^*)$ is not an optimal solution of the corresponding problem. Then, there exists an alternative $(H_1^{(0)}, \Hat{W}_1^*, \dots, \Hat{W}_L^*)$ that achieves smaller loss for this problem. Let us duplicate all the samples of $\Hat{H}_1^{(0)}$ for $n$ times, thus constructing $\Hat{H}_1^*=H_1^{(0)}\otimes \mathbf{1}_n^T.$ The solution $(\Hat{H}_1^*, \Hat{W}_1^*, \dots, \Hat{W}_L^*)$ has the same loss under the $L$-DUFM problem with regularization parameters $(\lambda_{H_1}, \lambda_{W_1}, \dots, \lambda_{W_L})$ as the solution $(H_1^{(0)}, \Hat{W}_1^*, \dots, \Hat{W}_L^*)$ for the $L$-DUFM with $n=1$ and parameters $(n\lambda_{H_1}, \lambda_{W_1}, \dots, \lambda_{W_L}).$ This is easy to see from the separability of both $\norm{H_1}_F^2$ and the fit part of the loss in \eqref{eq:LDUFM} \wrt the columns of $H_1.$ For the same reasons, the loss functions for $(H_1^*, W_1^*, \dots, W_L^*)$ in the original problem equals the loss function of the solutions $(H_1^{(1)}, W_1^*, \dots, W_L^*)$ or $(H_1^{(2)}, W_1^*, \dots, W_L^*)$ in the reduced problem. In fact, if this was not the case, there would need to be an inequality between the losses exhibited by two different columns of $H_1$ belonging to the same class, from which we could arrive at a contradiction by taking the better column and multiplying it to all columns within that class, thereby obtaining a better solution. This means that the loss of $(\Hat{H}_1^*, \Hat{W}_1^*, \dots, \Hat{W}_L^*)$ in the original problem is smaller than the loss of $(H_1^*, W_1^*, \dots, W_L^*),$ which is a contradiction. 

\textbf{Step 2: Excluding an aligned case.} By assumption we know that not only $H_1^{(1)}$ and $H_1^{(2)}$ differ in the $c$-th column (from now on we assume without loss of generality that it is the first column) but also $H_L^{(1)}$ and $H_L^{(2)}$ do. Denote for simplicity the first (differing) columns of $H_L^{(1)}$ and $H_L^{(2)}$ as $x, y$ respectively. We now show that it is not possible that $y=\alpha x.$ First, $\alpha$ has to be non-negative since $x, y$ are entry-wise non-negative given that they come after the application of $\sigma.$ %We This case easily leads to a contradiction. 
Assume w.l.o.g. $\alpha>1$ (otherwise, we can just exchange the roles of $x$ and $y$).  Consider a reduced problem where we only optimize for the \textit{size} of the first column of either $H_L^{(1)}$ or $H_L^{(2)},$ focusing on that part of the problem \eqref{eq:LDUFM} which is relevant for this column, being: $$\underset{t\ge0}{\min} \hspace{1mm}\frac{1}{2N}\norm{tW_Lh_{11}^{L(i)}-e_1}_2^2+\frac{n\lambda_{H_1}}{2}\norm{th_{11}^{1(i)}}_2^2.$$ This problem is strongly convex, quadratic and simple enough to give the following: if $\alpha>1,$ then $\norm{h_{11}^{1(1)}}_2^2> \norm{h_{11}^{1(2)}}_2^2$ and simultaneously $\norm{W_Lx-e_1}_2^2>\norm{W_Ly-e_1}_2^2.$ %\marco{why?} 
This means that $H_1^{(2)}$ is a strictly better solution than $H_1^{(1)}$ -- a contradiction.

\textbf{Step 3: Contradiction by zero gradient condition.} %For convenience we denote $H_L^x := H_L^{(1)} \neq H_L^{(2)} =: H_L^y,$ with the same notation as in \textit{step 2}. Now we also know that $x, y$ are not aligned. 
By optimality of both solutions we get $$\at{\frac{\partial \mathcal{L}}{\partial W_L}}{(H_1, W_1, \dots, W_L)=(H_1^{(1)}, W_1^*, \dots, W_L^*)}=0=\at{\frac{\partial \mathcal{L}}{\partial W_L}}{(H_1, W_1, \dots, W_L)=(H_1^{(2)}, W_1^*, \dots, W_L^*)}.$$ An application of the chain rule gives $$\frac{\partial \mathcal{L}}{\partial W_L}=\frac{\partial \mathcal{L}_F}{\partial \Tilde{H}_{L+1}}\frac{\partial \Tilde{H}_{L+1}}{\partial W_L}+\lambda_{W_L}W_L=\frac{\partial \mathcal{L}_F}{\partial \Tilde{H}_{L+1}}H_L^T+\lambda_{W_L}W_L,$$ where $\Tilde{H}_{L+1}$ is the output of our model. Plugging this back to the previous equation and using that $W_L^*$ is the same in both expressions, we get $$\at{\frac{\partial \mathcal{L}_F}{\partial \Tilde{H}_{L+1}}}{(H_1, W_1, \dots, W_L)=(H_1^{(1)}, W_1^*, \dots, W_L^*)}(H_L^{(1)})^T=\at{\frac{\partial \mathcal{L}_F}{\partial \Tilde{H}_{L+1}}}{(H_1, W_1, \dots, W_L)=(H_1^{(2)}, W_1^*, \dots, W_L^*)}(H_L^{(2)})^T.$$

Let us denote $$A=\at{\frac{\partial \mathcal{L}_F}{\partial \Tilde{H}_{L+1}}}{(H_1, W_1, \dots, W_L)=(H_1^{(1)}, W_1^*, \dots, W_L^*)},\,\,\, B=\at{\frac{\partial \mathcal{L}_F}{\partial \Tilde{H}_{L+1}}}{(H_1, W_1, \dots, W_L)=(H_1^{(2)}, W_1^*, \dots, W_L^*)}.$$ Due to the separability of $\mathcal{L}_F$ with respect to the columns of $H_1$ (and, thus, also of $H_l, \Tilde{H}_l$ for all $l \le L+1$), we get that the matrices $A, B$ can only differ in their first columns and are identical otherwise. We denote these columns $a, b$ for $A, B,$ respectively. This implies that %Using that a matrix product is just a sum of outer products of to-each-other corresponding columns of the left matrix and rows of the right matrix, we see that from the above equation we get: 
$ax^T=by^T.$ 

Now we exclude a few cases. First, neither $a$ nor $b$ can be zero, because by the exact formula that exists for them, this would mean that exact fit was achieved for either of the columns. This is impossible with non-zero weight-decay, because decreasing the norm of the column of $H_1^{(1)}$ or $H_1^{(2)}$ that achieves the exact fit by a sufficiently small value would necessarily lead to an improvement on the objective value. Moreover, if $x=y=0$, then this is a contradiction with the assumption that $x\neq y.$ Finally, the case $x\neq y=0$ or $0=x\neq y$ is excluded already in the \textit{step 2}.  

Thus we get $x, y, a, b$ are all non-zero. Looking at any fixed row (column) of $ax^T$ and $by^T$ we see that necessarily $x, y$ ($a, b$) are aligned. However, this case is already solved in \textit{step 2} and leads to $x=y,$ which is the contradiction. This concludes the proof. 
\end{proof}

For proof enthusiasts, we include an alternative proof of the same theorem which uses an explicit formula for the conditionally optimal $W_L$ instead of the necessary condition on the zero gradient:

\begin{proof}
The first two steps are identical to the previous proof. 

\textbf{Step 3: Reducing the non-aligned case to an algebraic statement.} We are left with the assumption that $a, b$ (taking the notation from the \textit{step 2}) are not aligned. The strategy of this step is simple -- we will show that in at least one of the two solutions $(H_1^{(1)}, W_1^*, \dots, W_L^*)$ and $(H_1^{(2)}, W_1^*, \dots, W_L^*),$ the last weight matrix $W_L^*$ is not, in fact, conditionally optimal conditioned on the rest of the solution. To this end, let us realize that if we condition on the rest of the solution, then the optimization of $W_L$ only depends on $H_L^{(i)}, i=1,2.$ Moreover, there exists an explicit formula for this conditionally optimal $W_L$ since the resulting problem is just a ridge regression. The solution is:
$$(W_L^{(i)})^T=((H_L^{(i)})(H_L^{(i)})^T+\mu I_{d_L})^{-1}H_L^{(i)}, \hspace{2mm} i=1,2$$
for $\mu=K\lambda_L.$ If we show that necessarily $W_L^{(1)} \neq W_L^{(2)},$ then the $W_L^*$ can only equal one of them thus not being the conditionally optimal for the other solution. This can be seen as an abstract algebraic questoin: given two matrices $H_L^{(1)}; H_L^{(2)}$ that only differ in their first columns that are not aligned, are the $(W_L^{(i)})^T=((H_L^{(i)})(H_L^{(i)})^T+\mu I_{d_L})^{-1}H_L^{(i)}$ for $i=1,2$ necessarily different? 

\textbf{Step 4: Solving the algebraic statement.} We will prove that indeed they have to be different. Write the two feature matrices in their SVDs as follows: $H_L^{(1)}=U_1 \Sigma_1 V_1^T$ and $H_L^{(2)}=U_2 \Sigma_2 V_2^T.$ Assume for contradiction that the corresponding solutions are equal. Then, without loss of generality we can assume the feature matrices are equal in rank, denoted $r$, and we can afford writing the compact SVD. After simple algebraic manipulations, we can rewrite the solutions in terms of the SVDs as: $$(W_L^{(i)})^T=U_i f(\Sigma_i) V_i^T,$$ where $f(\cdot)$ is a function $\frac{x}{x^2+\mu}$ applied entry-wise. If the corresponding solutions equal, then we know the following:
$$U_1 f(\Sigma_1) V_1^T=U_2 f(\Sigma_2) V_2^T.$$ Since the SVD is unique up to permutations of the singular values and the corresponding left and right singular vectors and up to shared orthogonal transformations of the left and right singular spaces corresponding to a single singular value, we know that there exists such an $R \in \mathbb{R}^{r\times r}$ orthogonal that $U_1 R = U_2$ and $R^TV_1^T=V_2^T.$ Furthermore, $f(\Sigma_1)$ and $f(\Sigma_2)$ have the same singular values with the same multiplicities. We can now write:
$$H_L^{(2)}-H_L^{(1)}=U_2 \Sigma_2 V_2^T-U_1 \Sigma_1 V_1^T=U_1R \Sigma_2 R^TV_1^T-U_1 \Sigma_1 V_1^T=U_1(R\Sigma_2R^T-\Sigma_1)V_1^T.$$
We know that $H_L^{(2)}-H_L^{(1)}$ is a rank 1 matrix with only the first column non-zero. Therefore, $R\Sigma_2R^T-\Sigma_1$ must be rank 1. Since it is symmetric, we can write $R\Sigma_2R^T-\Sigma_1=cuu^T$ for some $c \in \mathbb{R}$ and $u \in \mathbb{R}^r.$ Without loss of generality we can assume $c>0.$ Furthermore, since $U_1$ is orthogonal and the matrix itself only has the first column non-zero, we know that $u$ is in null-space of $(V_1)_{2:,:}.$ We will now work with the following two equations:
\begin{align}
   Rf(\Sigma_2)R^T&=f(\Sigma_1) \label{eq:zero_diff} \\
   R\Sigma_2R^T &= \Sigma_1+cuu^T \label{eq:one_col_diff}
\end{align}
We see that $R$ is a matrix of eigenvectors in both equations. In the following, we will be investigating the possible structure of eigenvalues and eigenvectors of $\Sigma_1+cuu^T.$ Denote $\sigma_i^{(j)}, \hspace{1mm} i\in[r], j=1,2$ the non-decreasing ordering of the diagonal entries of $\Sigma_{j}.$ The $\sigma_i^{(2)}$ happen to be eigenvalues of $\Sigma_1+cuu^T$ either. According to standard interlacing results \cite{bunch1978rank}, we know that $\sigma_1^{(1)}\le \sigma_1^{(2)} \le \sigma_2^{(1)} \le \sigma_2^{(2)} \le \dots \le \sigma_r^{(1)} \le \sigma_r^{(2)}$ and the total increase in the eigenvalues is exactly equal $c.$

Let us consider any eigenvector of $\Sigma_1+cuu^T,$ \ie we have $(\Sigma_1+cuu^T)x=\gamma x.$ After small manipulations we get: $cu^Txu=(\gamma I-\Sigma_1)x.$ We have a few observations: 
\begin{itemize}
   \item The function $f$ is unimodal with 0 and $\sqrt{c}$ being the extremes. Moreover except $0$ and $\sqrt{c},$ there is an exact pairing of inputs smaller than $\sqrt{c}$ and bigger than $\sqrt{c}.$ The function $f$ is injective everywhere except within the pairs. The pairs share the function value. 
   \item The $x,$ being eigenvector of $f(\Sigma_1),$ must only have non-zero values at at most two distinct eigenvalues of $\Sigma_1$ (based on the previous item).
   \item If $u^Tx=0$ then $x$ is eigenvector of $\Sigma_1$ itself. If $u^Tx\neq 0$ then for each non-zero entry in $u,$ the corresponding entry in $x$ must also be non-zero. 
\end{itemize}
To find out what the eigenvectors of $\Sigma_1+cuu^T$ are, we further observe the following: If $u\equiv 0$ on all entries corresponding to some eigenvalue of $\Sigma_1,$ then $\Sigma_1$ and $\Sigma_1+cuu^T$ share the corresponding eigenspace including the multiplicity of the corresponding eigenvalue. This is easy to see because we can just plug all the eigenvectors of $\Sigma_1$ for that eigenvalue into the equation $(\Sigma_1+cuu^T)x=\gamma x$ and conclude that they are eigenvectors of $\Sigma_1+cuu^T$ with the same eigenvalue. Furthermore, the dimension of the eigenspace corresponding to this eigenvalue in $\Sigma_1+cuu^T$ cannot be higher, because the ``new'' eigenvector could not be orthogonal to any vector in the original eigenspace (as it is composed of all the vectors with zero values everywhere except the entries corresponding to the eigenvalue of $\Sigma_1$ in question). 

For the opposite case assume that $u \neq 0$ on the entries corresponding to some eigenvalue $\sigma$ of $\Sigma_1,$ and denote its multiplicity $m.$ In this case, $\sigma$ will have multiplicity at least $m-1$ in $\Sigma_1+cuu^T,$ since there is $m-1$ dimensional sub-space of eigenvectors in the eigenspace of $\sigma$ of $\Sigma_1$ that are orthogonal to $u.$ The multiplicity, however, can still be also $m$ since $cuu^T$ can create new eigenvectors. 

For all the eigenvectors that were not enumerated yet (\ie all within the eigenspace where $u\equiv 0$ and $m-1$ for the eigenspace where $u\neq 0$), we must have $u^Tx\neq 0$ because otherwise $x$ would need to be eigenvector of $\Sigma_1$ itself and would only have non-negative entries corresponding to a single eigenvalue of $\Sigma_1$ but then it could not be orthogonal to $u$ since we have already enumerated all eigenvectors orthogonal to $u$ that are non-zero only on one eigenvalue. 

To this end we can infer that every $x$ left (thus $u^Tx\neq 0$) has non-zero entries on a superset of those of $u.$ But since we already know that $x$ can only have non-zero entries on at most two eigenvalues of $\Sigma_1,$ this means that $u$ must as well. But then the number of eigenvectors for which $u^Tx=0$ is at least $r-2$ and we can have at most $2$ new eigenvectors. This also means that at most two eigenvalues of $\Sigma_1+cuu^T$ differ from those of $\Sigma_1.$ 

First assume that $u$ has non-zero entries on two distinct eigenvalues of $\Sigma_1$, call them $\lambda_1<\lambda_2.$ Since $x$ also has non-zero values on both, we must have $f(\lambda_1)=f(\lambda_2).$ What are the possible corresponding eigenvalues in $\Sigma_1+cuu^T$? Let $\gamma_1, \gamma_2$ be the two new eigenvalues. Since these are the entries in $\Sigma_2$ and $f(\Sigma_1)=f(\Sigma_2),$ neither of the $f(\gamma_1), f(\gamma_2)$ can be outside of the (multi-)set $\{f(\sigma_1), f(\sigma_2), \dots, f(\sigma_r)\}$ as then $f(\Sigma_2)$ would have some values that $f(\Sigma_1)$ has not. Moreover, from the interlacing we know $\gamma_1\ge \lambda_1; \hspace{1mm} \gamma_2\ge \lambda_1.$ Furthermore, since the multiplicity of $f(\lambda_1)=f(\lambda_2)$ eigenvalue of $f(\Sigma_1)$ must stay the same in $f(\Sigma_2),$ we also get $\gamma_1\le \lambda_2$ and $\gamma_2 \le \lambda_2.$ Even further, since only $\lambda_1, \lambda_2$ can change and both exhibit the same $f(\cdot)$ value, the $\gamma_1, \gamma_2$ must be equal to either of these, otherwise the number of $f(\lambda_1)$ singular values in $f(\Sigma_1)$ would be strictly bigger than in $f(\Sigma_2).$ Altogether, this creates only one possibility: $\gamma_1=\gamma_2=\lambda_2.$ However, this is a contradiction, because if $\lambda_2=\gamma_2$ then we have $cu^Txu=(\gamma_2 I-\Sigma_1)x=(\lambda_2 I-\Sigma_1)x.$ This cannot happen since the LHS has a non-zero entry corresponding to the $\lambda_2$ eigenvalue in $\Sigma_1$ (assumption), while the matrix $\lambda_2 I-\Sigma_1$ has zero-diagonal at exactly those positions, zeroing out the corresponding non-zero entries of $x$ and forcing inequality.

From this we judge that $u$ only has non-negative entries corresponding to one single singular value (say $\lambda$) in $\Sigma_1.$ There can only be one eigenvalue that changes and moreover $x=u,$ because $x$ cannot have any non-zero entries not-corresponding to $\lambda,$ as it would fail to be orthogonal to the other eigenvectors. Moreover for the entries corresponding to $\lambda$ it needs to be aligned to $u$ because all the other eigenvectors corresponding to $\lambda$ form an orthogonal complement of $u$ on that eigenspace. 

Now, since $u\in \mathcal{N}((V_1)_{2:,:}),$ we know that $u^TV_1^T=d e_1^T$ for some $d\in \mathbb{R}.$ Then multiplying from right with $V_1$ we get $u^T=d e_1^TV_1.$ The RHS is just the first row of $V_1.$ Therefore $u$ is aligned with the first row of $V_1.$\footnote{This fact alone is deducible far sooner in the proof but the following argumentation is crucial to make the proof work.} Moreover, since $u$ is only non-zero on the entries corresponding to the $\lambda$ eigenvalue, this means that so does the first row of $V_1.$ We can now see that the first column of $H_1^{(1)}$ is just a combination of the columns in $U_1$ corresponding to $\lambda$ with the corresponding entries of $(V_1)_{1:}.$ At the same time, the first column of the $H_1^{(2)}-H_1^{(1)},$ which is nothing else than $U_1 cuu^T V_1^T$ happens to be the very same combination of the columns of $U_1.$ Therefore, the first column of $H_1^{(1)}$ and that of $H_1^{(2)}-H_1^{(1)}$ are aligned, which is a contradiction. This concludes the proof of the theorem. 
\end{proof}

%Now we also restate the Definition~\ref{def:relaxed_dufm}:

%\relaxeddufm*

%Now we restate and prove Theorem~\ref{thm:approx_dnc1_relaxed}:

\approxdnconerelaxedrelu*
\begin{proof}
In order not to mix approximation technical details with the gist of the proof, we split the argument into two parts: in the first part, we provide a heuristic for $\epsilon=0$ by assuming that ReLU is differentiable at 0; in the second part, we discuss what changes in the proof if we use the relaxation and then execute a technical computation that bounds the error based on the strength of the approximation (the size of $\epsilon$). It is useful to split the loss function $\mathcal{L}$ into the two terms $\mathcal{L}_F$ and $\mathcal{L}_R$. The former represents the fit part of the loss (which penalizes for deviation of the predictions from $Y$), and the latter represents the regularization part of the loss. 

\textbf{Part 1: Heuristic for $\epsilon=0$.} We start identically to the proof of Theorem~\ref{thm:nc1_dufm}. As in \textit{Step 1} of that argument, if we have a globally optimal solution that does not exhibit DNC1 in the first layer, we can construct two different solutions of the $n=1$ $L$-DUFM problem where the two solutions only differ in the first column of the $H_1$ matrix. Let us denote two such constructions $H_1^x \neq H_1^y,$ where we denote $x \neq y,$ to be the mentioned first columns, respectively. We emphasize that $H_1^x, H_1^y$ both form optimal solutions of the $n=1$ $L$-DUFM with \textit{the same} tuple of weight matrices $(W_1^*, \dots, W_L^*).$ In particular, this means that $$\at{\frac{\partial \mathcal{L}}{\partial W_1}}{(H_1, W_1, \dots, W_L)=(H_1^x, W_1^*, \dots, W_L^*)}=0=\at{\frac{\partial \mathcal{L}}{\partial W_1}}{(H_1, W_1, \dots, W_L)=(H_1^y, W_1^*, \dots, W_L^*)}.$$ An application of the chain rule gives $$\frac{\partial \mathcal{L}}{\partial W_1}=\frac{\partial \mathcal{L}_F}{\partial \Tilde{H}_2}\frac{\partial \Tilde{H}_2}{\partial W_1}+\lambda_{W_1}W_1=\frac{\partial \mathcal{L}_F}{\partial \Tilde{H}_2}H_1^T+\lambda_{W_1}W_1.$$ Plugging this back to the previous equation and using that the $W_1^*$ is the same in both expressions, we get $$\at{\frac{\partial \mathcal{L}_F}{\partial \Tilde{H}_2}}{(H_1, W_1, \dots, W_L)=(H_1^x, W_1^*, \dots, W_L^*)}(H_1^x)^T=\at{\frac{\partial \mathcal{L}_F}{\partial \Tilde{H}_2}}{(H_1, W_1, \dots, W_L)=(H_1^y, W_1^*, \dots, W_L^*)}(H_1^y)^T.$$

Let us denote $$A=\at{\frac{\partial \mathcal{L}_F}{\partial \Tilde{H}_2}}{(H_1, W_1, \dots, W_L)=(H_1^x, W_1^*, \dots, W_L^*)},\,\,\, B=\at{\frac{\partial \mathcal{L}_F}{\partial \Tilde{H}_2}}{(H_1, W_1, \dots, W_L)=(H_1^y, W_1^*, \dots, W_L^*)}.$$ Due to the separability of $\mathcal{L}_F$ with respect to the columns of $H_1$ (and, thus, also $H_l$ for all $l \le L$), we get that the matrices $A, B$ can only differ in their first columns and are identical otherwise. We denote these columns $a, b$ for $A, B,$ respectively. This implies that $ax^T=by^T.$ 

Now, we treat a few cases. First, assume $a=b=0.$ Since it holds that $(H_1^x, W_1^*, \dots, W_L^*)$ is the optimal solution, then necessarily $$0=\at{\frac{\partial \mathcal{L}}{\partial h_{11}^1}}{(H_1, W_1, \dots, W_L)=(H_1^x, W_1^*, \dots, W_L^*)}=(W_1^*)^T a + \lambda_{H_1} x \iff 0 = x.$$ Similarly, we get $y=0,$ but that is a contradiction with $x\neq y.$ Therefore, at least one of $a, b$ is non-zero. Next assume $x=0.$ Then $ax^T=0.$ If $b=0$ then $y=0$ and we have a contradiction. On the other hand, if $b\neq 0,$ then the row of $by^T$ that corresponds to a non-zero entry of $b$ must be zero and thus $y=0$. Similarly, if $y=0$ we can get $x=0.$ 

Let us therefore assume $x, y$ are both non-zero, which also implies $a, b$ are both non-zero. Looking at any fixed row (column) of $ax^T$ and $by^T$ we see that necessarily $x, y$ ($a, b$) are aligned. Let us write $x=\alpha y$ and $\alpha a = b$ for some $\alpha \neq 0.$ We will first show that if $\alpha>0$ then necessarily $\alpha=1$ and $x=y.$ For this, let us fix any ray $r \in \mathbb{R}^{d_1}.$ The ray $r$ represents a set of possible first columns in $H_1.$ Let us fix any $(W_1, \dots, W_L)$. Since $\mathcal{L}$ is separable in the columns of $H_1,$ we can consider an optimization over $\beta\ge 0$ to minimize $\mathcal{L}$ on $(\beta r, W_1, \dots, W_L).$ However, $\mathcal{L}_F$ is convex by assumption and the mapping $h_{11}^1 \xrightarrow[]{} h_{11}^L$ is ray-linear, therefore $\mathcal{L}_F$ is convex in $\beta.$ Moreover $\mathcal{L}_R$ is strongly convex in $h_{11}^1$ and therefore $\mathcal{L}$ is strongly convex in $\beta.$ This means it has a unique optimal solution $\beta^*.$ 

We have just showed that, if $\alpha>0,$ then since $x, y$ are aligned and thus lie on the same ray and are both optimal together with the same tuple of weight matrices, they must necessarily be identical and so $\alpha=1, x=y.$ This is a contradiction and thus we are left with the case $\alpha<0.$ Denote $s=W_1^* x$, $t=W_1^* y.$ From linearity, we have $s=\alpha t.$ Note that if $s=t=0,$ then necessarily $x=y=0$ because they don't have any effect on $\mathcal{L}_F$ and they minimize $\mathcal{L}$ at $0$. Thus, $s, t$ are non-zero and negative entries of $s$ are positive entries of $t$ and vice-versa. This means that $\sigma(s)$ and $\sigma(t)$ have different sets of positive entries. However, from the optimality of the both solutions we know: 
$$\at{\frac{\partial \mathcal{L}}{\partial W_2}}{(H_1, W_1, \dots, W_L)=(H_1^x, W_1^*, \dots, W_L^*)}=0=\at{\frac{\partial \mathcal{L}}{\partial W_2}}{(H_1, W_1, \dots, W_L)=(H_1^y, W_1^*, \dots, W_L^*)}.$$ Again, using the chain rule we get: $$\frac{\partial \mathcal{L}}{\partial W_2}=\frac{\partial \mathcal{L}_F}{\partial \Tilde{H}_3}\frac{\partial \Tilde{H}_3}{\partial W_2}+\lambda_{W_2}W_2=\frac{\partial \mathcal{L}_F}{\partial \Tilde{H}_3}H_2^T+\lambda_{W_2}W_2.$$ Plugging this back to the previous equation and using that $W_2^*$ is the same in both expressions, we get $$\at{\frac{\partial \mathcal{L}_F}{\partial \Tilde{H}_3}}{(H_1, W_1, \dots, W_L)=(H_1^x, W_1^*, \dots, W_L^*)}(H_2^x)^T=\at{\frac{\partial \mathcal{L}_F}{\partial \Tilde{H}_3}}{(H_1, W_1, \dots, W_L)=(H_1^y, W_1^*, \dots, W_L^*)}(H_2^y)^T.$$

Let us denote $$C=\at{\frac{\partial \mathcal{L}_F}{\partial \Tilde{H}_3}}{(H_1, W_1, \dots, W_L)=(H_1^x, W_1^*, \dots, W_L^*)}, \,\,\, D=\at{\frac{\partial \mathcal{L}_F}{\partial \Tilde{H}_3}}{(H_1, W_1, \dots, W_L)=(H_1^y, W_1^*, \dots, W_L^*)}.$$ Due to the separability of $\mathcal{L}_F$ with respect to the columns of $H_1$ (and, thus, also $H_l$ for all $l \le L$), we get that the matrices $C, D$ can only differ in their first columns and are identical otherwise. We denote these columns $c, d$ for $C, D$ respectively. This implies that $c\sigma(s)^T=d\sigma(t)^T.$ 

As above, if either $c$ or $d$ is zero, using the chain rule we would get that $x$ or $y$ (respectively) are zero too, which cannot happen. Therefore, both $c$ and $d$ are non-zero and $\sigma(s), \sigma(t)$ must be aligned. Since they are non-zero, non-negative and with different supports, we have reached a contradiction. This proves that $\alpha<0$ is also impossible and the only possible case is $\alpha=1,$ but that forces $x=y,$ which is also a contradiction. 

\textbf{Part 2: Relaxation to ReLU$_\epsilon$.} The difficulty in making the previous heuristic rigorous is that ReLU is not differentiable at %previous part \textit{would} prove the full DNC1 strict optimality in the original $L$-DUFM problem was there not for one issue -- the ReLU is not differentiable at 
0 and, thus, we do not have the desired analytical statement that global solutions must necessarily admit zero derivative (because the loss might not be differentiable at them at all). If we tried to use the same proof as in \textit{part 1} with the differentiable relaxed version of ReLU -- ReLU$_\epsilon,$ an issue would occur when showing that, if $x, y$ are aligned and $\alpha\ge0$, then $\alpha$ must be 1. The reason is that the mapping $h_{11}^1 \xrightarrow[]{} h_{11}^L$ is no longer ray-linear and the corresponding optimization problem on any fixed ray is no longer quadratic and strongly convex. However, we can prove that the optimization problem admits a solution that is \textit{close} to the solution of the corresponding ReLU optimization problem. For this, let us fix any direction $r$ such that $\norm{r}=1$ and an optimization parameter $t\ge0$, and define the following two losses: 
\begin{align}
L_0(t, r)&=\frac{1}{2K}\norm{W_L^*\sigma(W_{L-1}^*\sigma(\dots W_2^*\sigma(W_1^*tr)\dots))-e_1}_F^2+\frac{n\lambda_{H_1}}{2}t^2, \label{eq:auxiliary_loss_relu} \\
L_\epsilon(t, r)&=\frac{1}{2K}\norm{W_L^*\sigma_\epsilon(W_{L-1}^*\sigma_\epsilon(\dots W_2^*\sigma_\epsilon(W_1^*tr)\dots))-e_1}_F^2+\frac{n\lambda_{H_1}}{2}t^2. \label{eq:auxiliary_loss_relaxed_relu}
\end{align}
We now  bound $\underset{t\ge0}{\max}\hspace{1mm}|L_0(t, r)-L_\epsilon(t, r)|.$ For this, we fix any $t\ge0$, and bound the accumulated $l_2$ error of using $\sigma_\epsilon$ instead of $\sigma$ throughout the layers. For this, a useful statement that we will need is the following: 
$$\frac{\lambda_{W_L}}{2}\norm{W_L^*}_F^2=\frac{\lambda_{W_{L-1}}}{2}\norm{W_{L-1}^*}_F^2=\dots=\frac{\lambda_{W_1}}{2}\norm{W_1^*}_F^2=\frac{n\lambda_{H_1}}{2}\norm{H_1^*}_F^2\le\frac{1}{2(L+1)}.$$
This is true because, by a simple computation, we get that the terms must be balanced and the inequality comes from the fact that the solution $(0, 0, \dots, 0)$ achieves full loss $1/2$ in the $L$-DUFM as well as $L$-DUFM$_\epsilon$ problems and, thus, $\mathcal{L}_R$ is trivially upper-bounded by this. This implies that, for each $W_l$ (and $H_1$), $\norm{W_l}\le\norm{W_l}_F\le\frac{1}{\sqrt{(L+1)\lambda_{W_l}}}.$

For ease of exposition, we set $t=1$ (the same argument would work for all $t$). We see that $W_1r$ does not introduce any $l_2$ error. Then, the $l_2$ error that is introduced in the application of the ReLU is trivially upper-bounded by $\epsilon\sqrt{d_2}.$ After applying $W_2,$ we can use the bound on the operator norm $\norm{W_2}$ obtained above to upper bound the propagated $l_2$ error by $$\frac{1}{\sqrt{(L+1)\lambda_{W_2}}}\epsilon\sqrt{d_2}.$$ After that, we obtain an additive $l_2$ error of $\epsilon\sqrt{d_3}$ by applying $\sigma_\epsilon$, which gives a total $l_2$ error of
$$\frac{1}{\sqrt{(L+1)\lambda_{W_2}}}\epsilon\sqrt{d_2}+\epsilon\sqrt{d_3}.$$
Then, again, we multiply this whole expression with $\frac{1}{\sqrt{(L+1)\lambda_{W_3}}}$ to account for the multiplication by $W_3.$ Inductively, the upper bound on the $l_2$ error in the output space is $$\epsilon\sum_{l=2}^L\frac{\sqrt{d_l}}{(L+1)^{\frac{L-l+1}{2}}}\prod_{j=l}^L\frac{1}{\sqrt{\lambda_{W_j}}}.$$ This can be further upper bounded as $$\frac{\epsilon \sqrt{D}\sqrt{\lambda_{H_1}\lambda_{W_1}}}{(L+1)^{\frac{L-1}{2}}\sqrt{\Bar{\lambda}}}.$$ Using the triangle inequality, the upper bound on $|L_0(t, r)-L_\epsilon(t, r)|$ is this expression squared. On the other hand, the second derivative of $L_0(t,r)$ with respect to $t$ is lower bounded by $n\lambda_{H_1}.$ 

Given two functions $f_1$ and $f_2$, with $f_1$ strongly convex with second derivative at least $c$ and $f_2$ everywhere at most $d$ distant from $f_1$, then the distance between their global minimizers is at most $2\sqrt{d/c}$. Applying this to our case, we get that the distance between the minimizers $t_0$ and $t_\epsilon$ of $L_0(t, r)$ and $L_\epsilon(t, r)$ is at most $$\frac{2\epsilon \sqrt{D(L+1)}}{(L+1)^{\frac{L+1}{2}}\sqrt{\Bar{\lambda}n}}.$$ 
%\marco{unclear}
Since the ray $r$ is unit-norm, this is also the upper bound on the distance between two feature vectors of any globally optimal solution in the first layer. To obtain the upper bound on the distance between two vectors in any layer, we proceed as follows. %only need to do the error propagation argument once again, but this time with a little catch. 

Assume we have two input vectors of the first class (now we know they need to be aligned): $x^1, y^1=\alpha^1 x^1,$ where $\alpha^1>1$. As before, $\Tilde{x}^l, x^l, \Tilde{y}^l, y^l$ are the $l$-th layer representations of these vectors before and after $\sigma_\epsilon.$ If we would compute $\frac{\partial \mathcal{L}}{\partial W_l}$ with respect to any layer $l\ge2$ and used the same arguments as in part 1 (that still carry to the analysis for $\sigma_\epsilon$), we would find out that in each layer, we need to have $y^l = \alpha^l x^l.$ Moreover, we can assume that $\Tilde{x}^l, x^l, \Tilde{y}^l, y^l$ are non-zero at all layers because otherwise the same argument in \textit{part 1} would trivialize the rest of the proof. By a simple inductive argument, since $\alpha_1>1,$ we know that $\alpha_2>1$ as well because $\sigma_\epsilon$ is strictly increasing on $[0, \infty).$ Similarly, $\alpha_l>1$ for all $l.$ Note that, if there exists at least one index $i$ on which $x^l_i$ is bigger or equal than $\epsilon,$ then necessarily $\alpha^l=\alpha^{l-1},$ because $\sigma_\epsilon$ is the identity on inputs of at least $\epsilon$, and the $\alpha_l$ can be uniquely determined from $y_i^l/x_i^l=y_i^{l-1}/x_i^{l-1}.$ Having $\alpha^l=\alpha^{l-1}$ makes us have more control over the distance between $x^l$ and $y^l.$ If this fails to hold at some layer, we need a separate analysis.

For this, let $l_0$ denote the first layer (assuming it exists), where $\Tilde{x}^{l_0}$ does not have any entries that are bigger or equal than $\epsilon.$ This means that $\norm{x^{l_0}}\le\epsilon\sqrt{d_{l_0}}.$ We now compute the maximal possible norm of $y^{l_0}$. Since $\alpha_1=\alpha_2=\dots=\alpha_{l_0-1},$ for all $2\le l < l_0$ and for each index $i,$ either $\Tilde{x}^l_i, \Tilde{y}^l_i \le \epsilon$ or $\epsilon \le \Tilde{x}^l_i, \Tilde{y}^l_i.$ Otherwise, since $\sigma_\epsilon(x)<x \hspace{1mm}\forall x\in (0, \epsilon),$ $y_i^l/x_i^l > \alpha_l,$ a contradiction. Thus, we get $\norm{\Tilde{x}^l-\Tilde{y}^l}\ge \norm{x^l-y^l}.$ Therefore, $\norm{x^l-y^l}$ can only grow by applying $W_l.$ Thus, after applying $W_1, W_2, \dots, W_{l_0-1}$ and using the operator norm bound with all of the matrices, we get $$\norm{x^l-y^l}\le\frac{2\epsilon \sqrt{D(L+1)}}{(L+1)^{\frac{L+1}{2}}\sqrt{\Bar{\lambda}n}}\frac{1}{(L+1)^{\frac{l_0-1}{2}}\sqrt{\lambda_{W_1}\lambda_{W_2}\dots\lambda_{W_{l_0-1}}}}.$$ Combining with $\norm{x^{l_0}}\le\epsilon\sqrt{d_{l_0}}$ and using the triangle inequality, we get $$\norm{y^l}\le \frac{2\epsilon \sqrt{D(L+1)}}{(L+1)^{\frac{L+1}{2}}\sqrt{\Bar{\lambda}n}}\frac{1}{(L+1)^{\frac{l_0-1}{2}}\sqrt{\lambda_{W_1}\lambda_{W_2}\dots\lambda_{W_{l_0-1}}}}+\epsilon\sqrt{d_{l_0}}.$$ Applying the operator norm bound with $W_{l_0}, W_{l_0+1}, \dots, W_{L-1},$ we get:
\begin{align*}
    \norm{y^L}&\le \frac{2\epsilon \sqrt{D(L+1)}}{(L+1)^{\frac{L+1}{2}}\sqrt{\Bar{\lambda}n}}\frac{1}{(L+1)^{\frac{L-1}{2}}\sqrt{\lambda_{W_1}\lambda_{W_2}\dots\lambda_{W_{L-1}}}}\\
    &\hspace{4em}+\epsilon\sqrt{d_{l_0}}\frac{1}{(L+1)^{\frac{L-l_0}{2}}\sqrt{\lambda_{W_{l_0}}\lambda_{W_{l_0+1}}\dots\lambda_{W_{L-1}}}} \\ 
    &\le \frac{3\epsilon \sqrt{D(L+1)}}{(L+1)^{L+1}\Bar{\lambda}\sqrt{n}}.
\end{align*}
Since $\norm{x^L} \le \norm{y^L}$ we can use the triangle inequality again to obtain the final: $$\norm{x^L-y^L}\le\frac{6\epsilon \sqrt{D(L+1)}}{(L+1)^{L+1}\Bar{\lambda}\sqrt{n}}.$$ Note that the same bound is an upper bound for any layer, since in each layer we are increasing the upper bound by at least a factor $(L+1)^{-1/2}\lambda_{W_l}^{-1/2}$.

In case such a layer $l_0$ does not exist, the upper bound above is still valid, since we can simply use the error propagation through all the layers, as we did for $l<l_0$ in the previous case. We would arrive at the same bound as above but with $4$ instead of $6$ as a multiplicative factor (because we would not introduce the additive error $\epsilon\sqrt{d_{l_0}}$). 
\end{proof}

\section{Further experimental evidence}\label{app:experiments}
\subsection{$L$-DUFM experiments}\label{app:experiments_dufm}
\paragraph{Experiments on small number of classes and layers.}
For $L=3,$ we weren't able to find any solutions that would outperform DNC for $K\le 6$. For $K=7,$ the difference in the losses is extremely small. For $K=6,$ we did find \textit{some} low rank solutions, but these were either slightly worse than DNC or the losses were the same up to four decimal places. Already for $L\ge 4,$ we find low-rank solutions for all $K\ge3.$ We present some of them for $K\in\{3, 4, 5\}$ in Figure~\ref{fig:app_small_K_sols}. Note, that all these solutions were found \textit{automatically} by gradient descent through the optimization process. %, none was created manually. 

\begin{figure}
    \centering
    \includegraphics[width=0.35\textwidth]{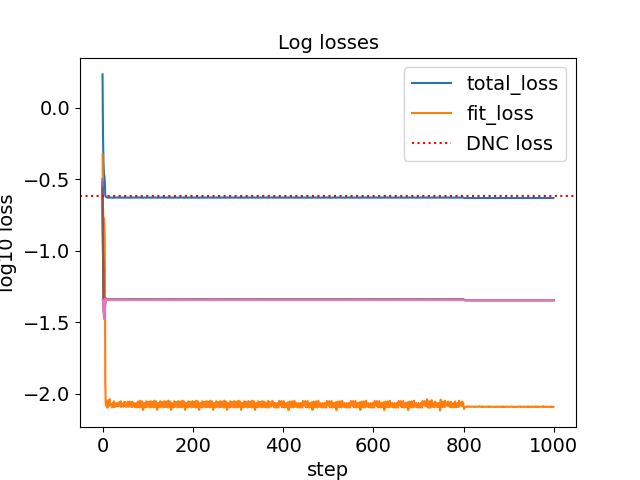}
    \includegraphics[width=0.13\textwidth]{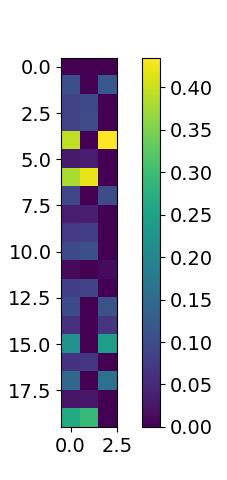}
    \includegraphics[width=0.13\textwidth]{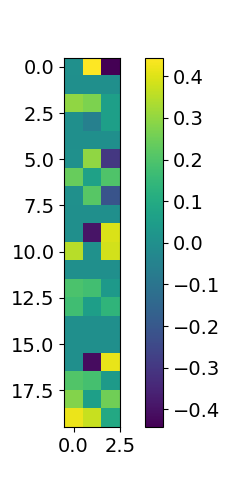}
    \includegraphics[width=0.35\textwidth]{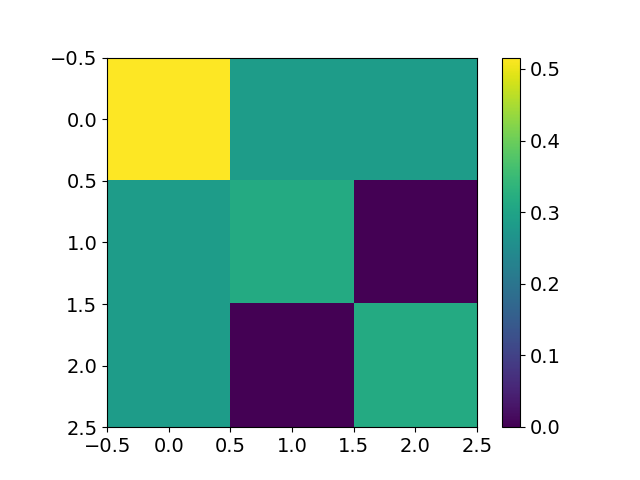}
    \includegraphics[width=0.35\textwidth]{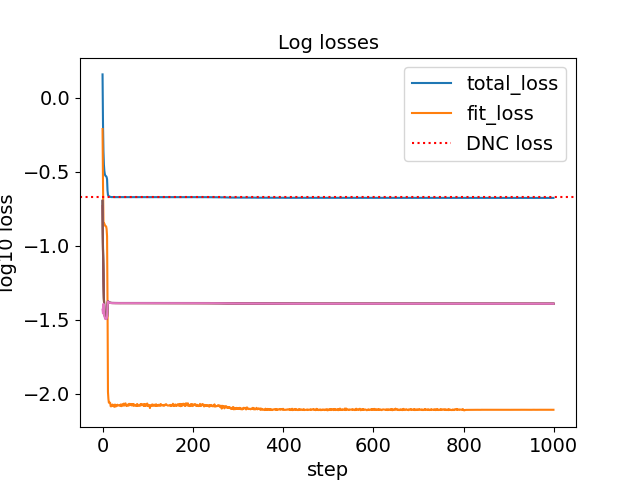}
    \includegraphics[width=0.14\textwidth]{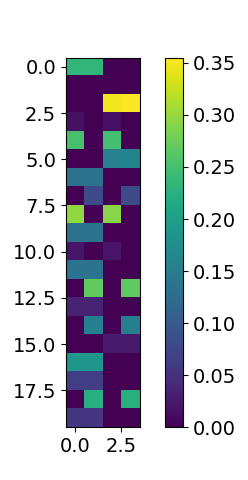}
    \includegraphics[width=0.14\textwidth]{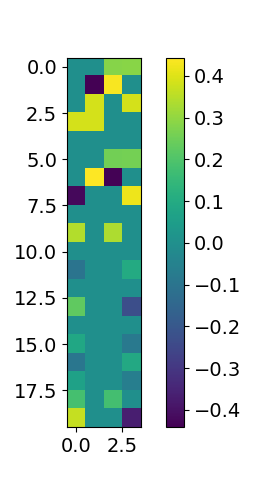}
    \includegraphics[width=0.35\textwidth]{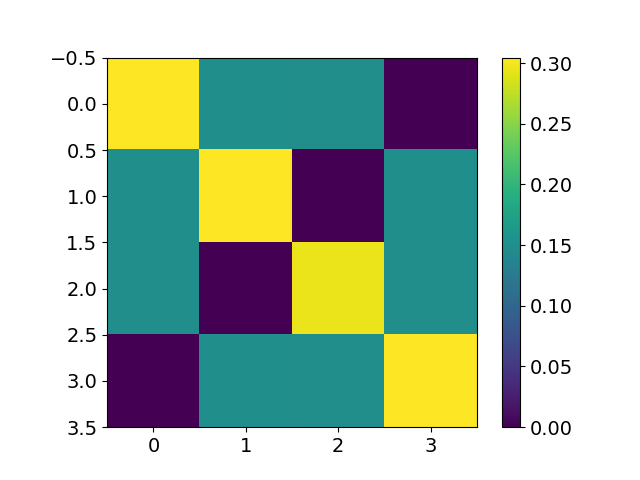}
    \includegraphics[width=0.34\textwidth]{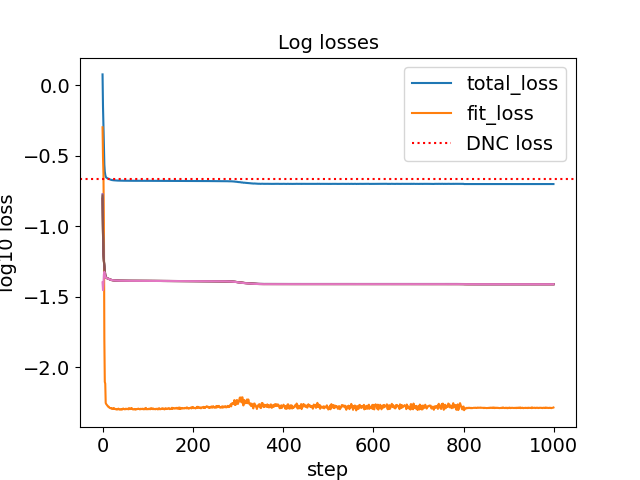}
    \includegraphics[width=0.14\textwidth]{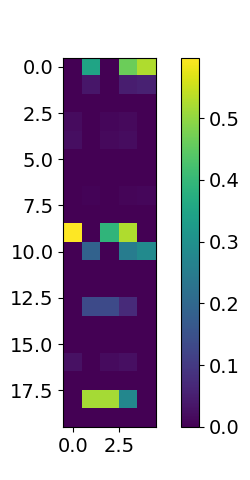}
    \includegraphics[width=0.15\textwidth]{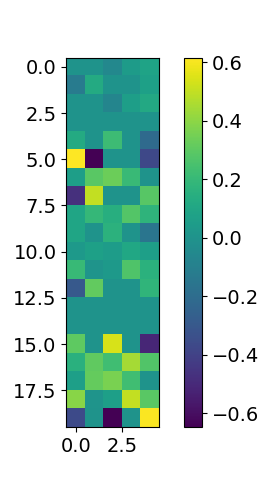}
    \includegraphics[width=0.34\textwidth]{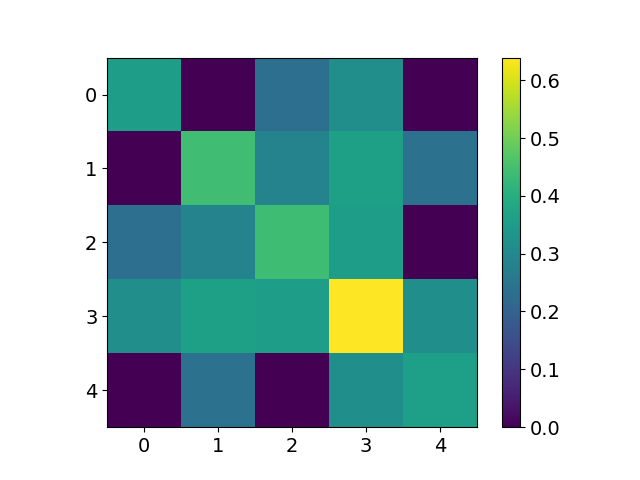}
    \caption{$4$-DUFM training for $K=3$ (\textbf{top}), $K=4$ (\textbf{middle}), and $K=5$ (\textbf{bottom}). \textbf{Left:} Loss progression, also decomposed into the fit and regularization terms. \textbf{Middle left:} Visualization of the matrix $M_3$. \textbf{Middle right:} Visualization of the matrix $\Tilde{M}_4$. \textbf{Right:} Visualization of the matrix $M_3^T M_3$.}
    \label{fig:app_small_K_sols}
\end{figure}

We specifically highlight the solution for $K=4$, as it represents another graph structure different from SRG. Namely, this solution is based on the square graph. We see that each column (edge) has positive scalar product with exactly two other columns (the corresponding edges share a vertex), and it is orthogonal to one other column (the non-touching sides of a square). Each row (vertex) has only two positive values, and these correspond to the two sides (edges) to which they belong. 

\paragraph{Experiments on large number of classes or layers.}
For large $L$ or $K,$ the behavior of the low-rank solutions is slightly different from the behavior when both $L$ and $K$ are moderate. When $l$ is close to $L$, it no longer holds that $\Tilde{M}_l = M_l$, and the rank of $M_l$ is larger as the singular values do not decay sharply to $0$. %, but rather have a smoother decay. Moreover, , but rather these layers have some negative values. 
However, we note that the singular values are rather small and, thus, it is not clear whether they would disappear with a longer training (we trained for 200000 steps of full gradient descent with learning rate 0.5). In Figure~\ref{fig:app_large_L}, we take $L=7, K=15$, and weight decay $0.0025$ in all layers. We highlight that $\Tilde{M}_6 \neq M_6$ and the singular values of $M_6$ are small but non-zero after the 3 dominant values. We also highlight that the rank of previous layers is 3, which is remarkably small (and also smaller than the rank of the SRG solution). 

\begin{figure}
    \centering
    \includegraphics[width=0.43\textwidth]{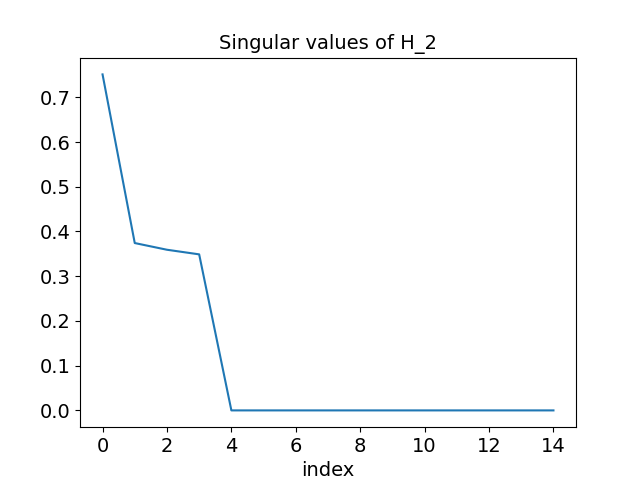}
    \includegraphics[width=0.18\textwidth]{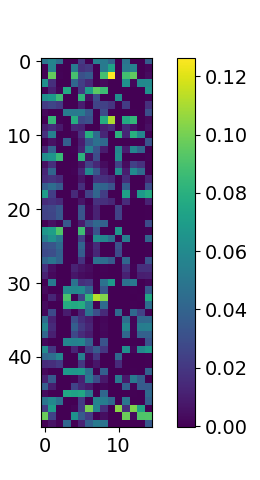}
    \includegraphics[width=0.17\textwidth]{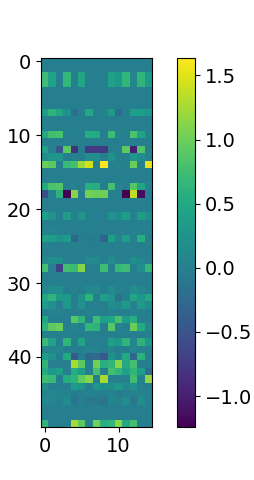}
    \includegraphics[width=0.17\textwidth]{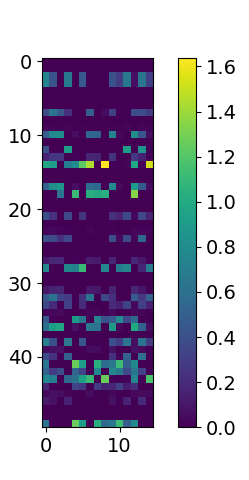}
    \includegraphics[width=0.45\textwidth]{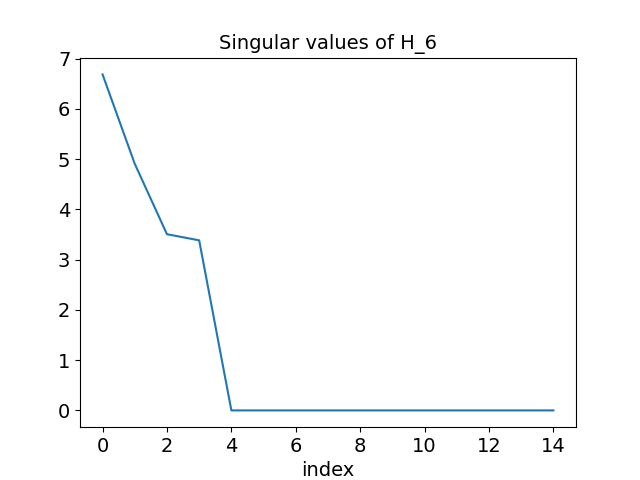}
    \includegraphics[width=0.45\textwidth]{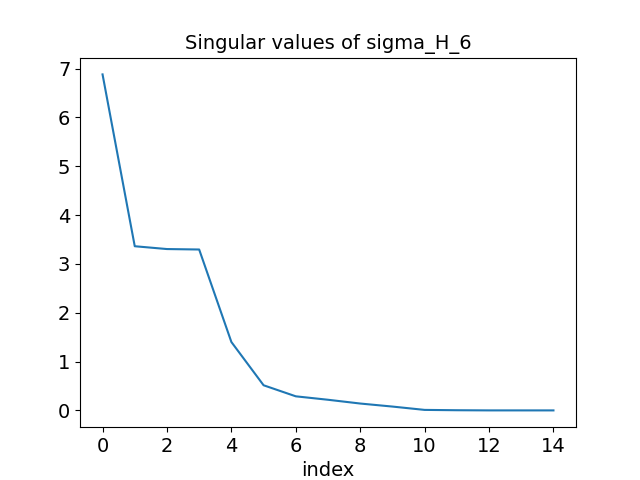}
    \caption{Class-mean matrices and singular values at convergence for a DUFM model with $K=15$ and $L=7$. \textbf{Top row:} Singular values of $\Tilde{M}_2$, and visualization of the matrices $\Tilde{M}_2, \Tilde{M}_6, M_6$ and \textbf{Bottom row:} Singular values of $\Tilde{M}_6$ and $M_6$.}
    \label{fig:app_large_L}
\end{figure}

\paragraph{SRG solutions.}
In Figure~\ref{fig:srg_illustration} presented in the main body, we recover the SRG solution for $K=10$. In Figures~\ref{fig:app_srg_six} and~\ref{fig:app_srg_fifteen}, we show that solutions very similar to SRG are recovered for $K=6$ and $K=15$, respectively. The only difference with SRG is in the construction of $\Tilde{M}_{L}$. We note that the losses of these solutions are slightly lower than the loss of our construction, which proves that the SRG solution itself is not necessarily globally optimal. 

%We recover SRG for $K\in \{6, 10, 15\}$, and present them in Figures~\ref{fig:app_srg_six} and~\ref{fig:app_srg_fifteen}. The solution for $K=10$ is included in the main body in Figure~\ref{fig:srg_illustration}. We highlight that neither the $K=6,$ nor the $K=15$ solutions found by gradient descent obey our construction of $\Tilde{M}_{L}$ and they do it their own way. 

\begin{figure}
    \centering
    \includegraphics[width=0.45\textwidth]{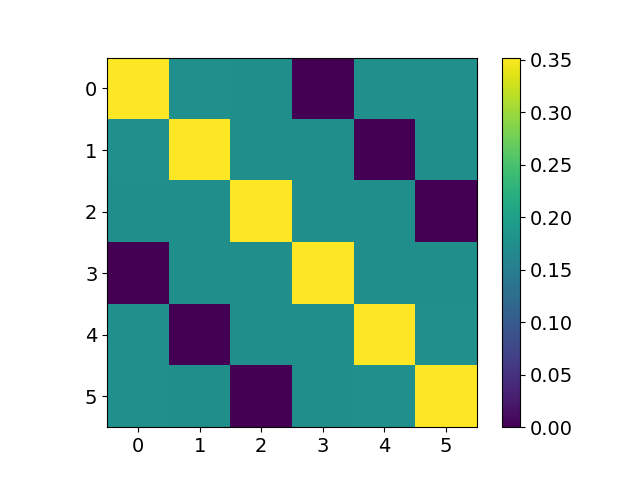}
    \includegraphics[width=0.15\textwidth]{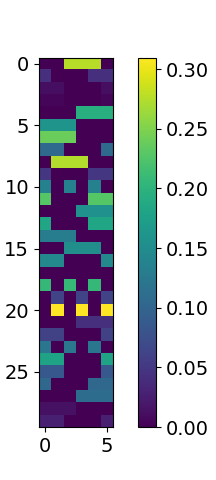}
    \includegraphics[width=0.15\textwidth]{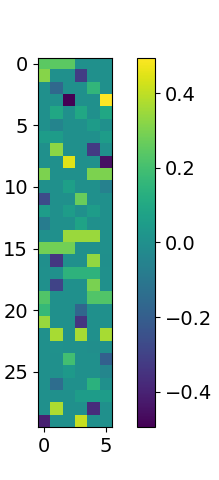}
    \includegraphics[width=0.45\textwidth]{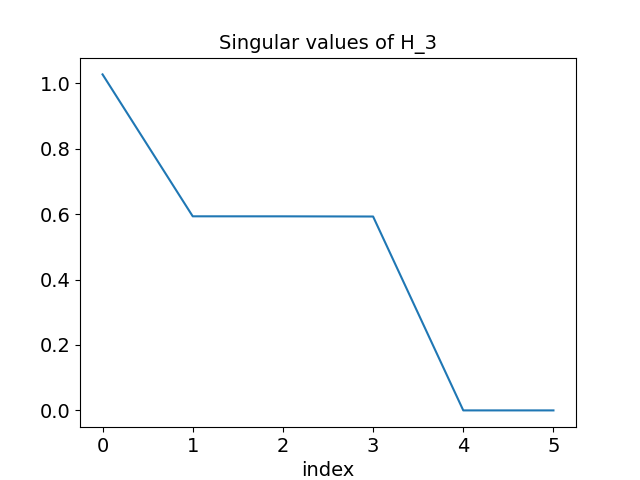}
    \includegraphics[width=0.45\textwidth]{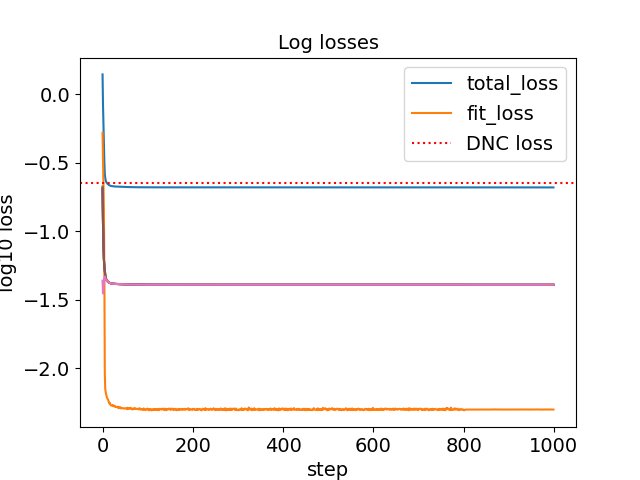}
    \caption{$4$-DUFM training for $K=6$. \textbf{Top row:} Visualization of the matrices $\Tilde{M}_3^T \Tilde{M}_3, \Tilde{M}_3$, and $ \Tilde{M}_4.$ \textbf{Bottom row:} Singular values of $H_3,$ and loss progression including its decomposition into fit and regularization terms.}
    \label{fig:app_srg_six}
\end{figure}

\begin{figure}
    \centering
    \includegraphics[width=0.50\textwidth]{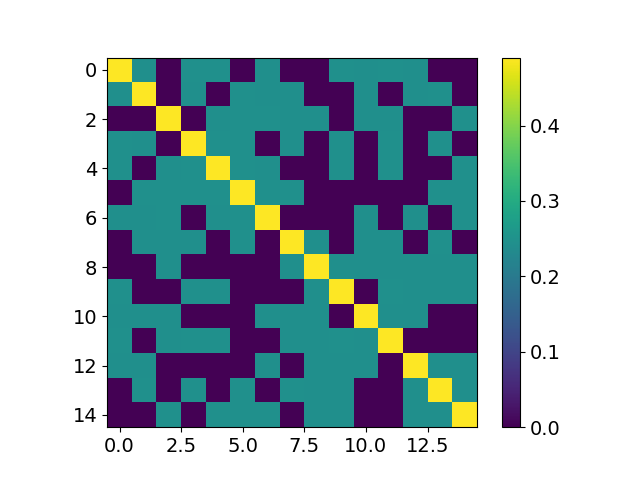}
    \includegraphics[width=0.20\textwidth]{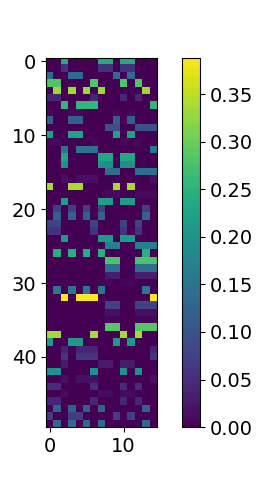}
    \includegraphics[width=0.20\textwidth]{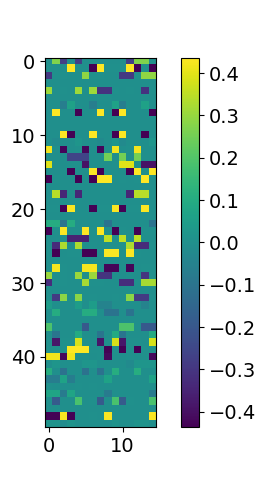}
    \includegraphics[width=0.50\textwidth]{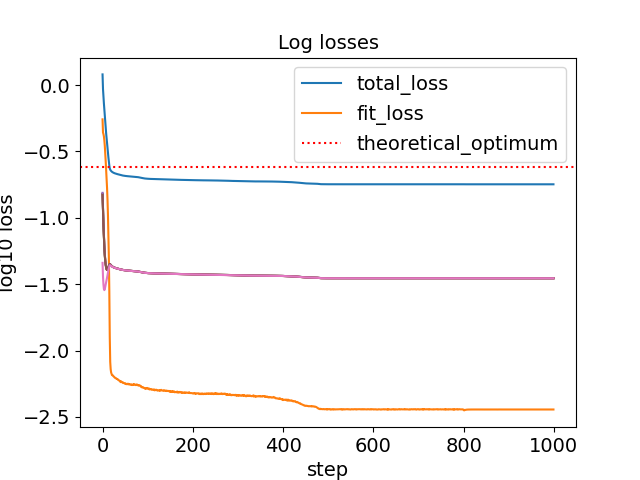}
    \caption{$4$-DUFM training for $K=15$. \textbf{Top row:} Visualization of the matrices $\Tilde{M}_3^T \Tilde{M}_3, \Tilde{M}_3$ and $\Tilde{M}_4$. \textbf{Bottom row:}  Loss progression including its decomposition into fit and regularization terms.}
    \label{fig:app_srg_fifteen}
\end{figure}

\subsection{End-to-end experiments with DUFM-like regularization}\label{app:exper_dufm_reg}
We complement the experiments of Figure~\ref{fig:dufm_ablations_main} with two extra ablation studies for ResNet20 trained on CIFAR-10 with a 4-layer MLP head. We focus on the dependence of the average rank on the weight decay and the learning rate, and present the results in Figure~\ref{fig:app_ablations}. The weight decay has a clear effect on the rank of the solutions found by gradient descent, similarly to the results in Figure~\ref{fig:dufm_ablations_main} for the $L$-DUFM model. The effect of the learning rate is slightly less clear, but we still see a general downward trend. %sA more refined and high-scale experimental analysis is required to study the effect of the learning rate. 

\begin{figure}
    \centering
    \includegraphics[width=0.4\textwidth]{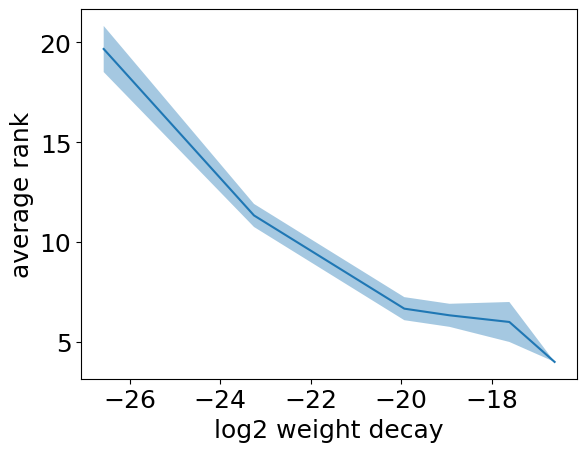}
    \hspace{2em}\includegraphics[width=0.4\textwidth]{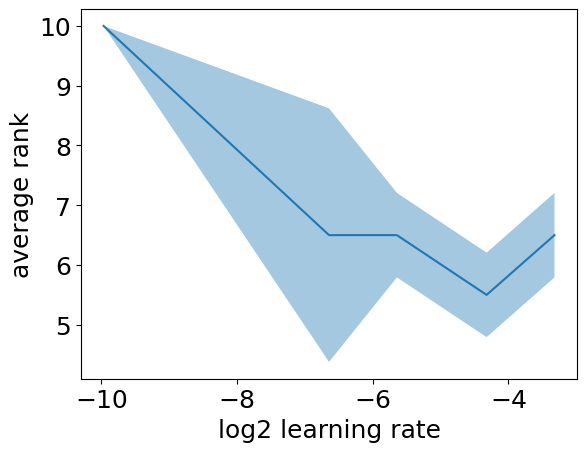}
    \caption{Average ranks as a function of $\log_2$ weight decay (\textbf{left}) and $\log_2$ learning rate (\textbf{right}). We trained ResNet20 with 4-layer MLP head on CIFAR-10.
    The experiments are averaged over three and two independent runs, respectively.}
    \label{fig:app_ablations}
\end{figure}

\subsection{End-to-end experiments with standard regularization}\label{app:experiments_standard}
Finally, in Figure~\ref{fig:app_abl_standard}, we include the analysis of the average rank as a function of weight decay in the standard regularization setting for training on the MNIST dataset. The results confirm the trend from the previous experimental settings, showing that the weight decay strength is a crucial predictor of the final rank even in standard regularization setting, which has a different loss landscape compared to the $L$-DUFM. %We did not include the results for the CIFAR10 dataset because only for weight decay $0.08$ were the feature matrices of deficient rank, meaning CIFAR10 is much more challenging for the ResNet20 to fit and/or resemble $L$-DUFM closely enough.  

\begin{figure}
    \centering
    \includegraphics[width=0.4\textwidth]{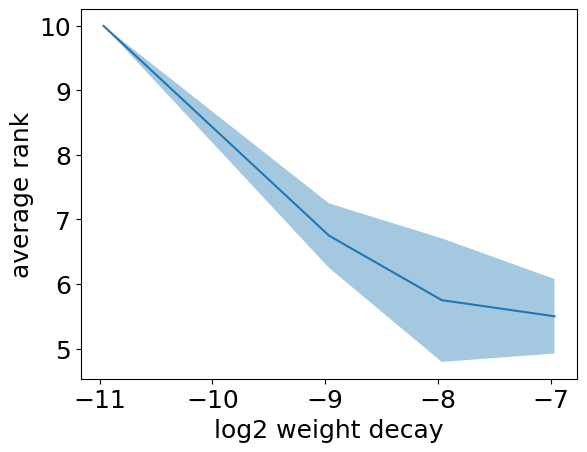}
    \caption{Average rank as a function of the $\log_2$ weight decay. We trained ResNet20 with 5-layer MLP head on MNIST. The experiments are averaged over 4 independent runs.}
    \label{fig:app_abl_standard}
\end{figure}

\subsection{Rebuttals}

\begin{figure}
    \centering
    \includegraphics[width=0.45\linewidth]{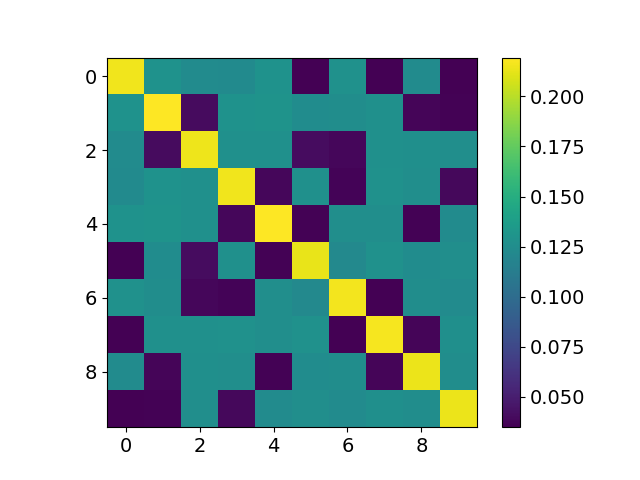}
    \includegraphics[width=0.45\linewidth]{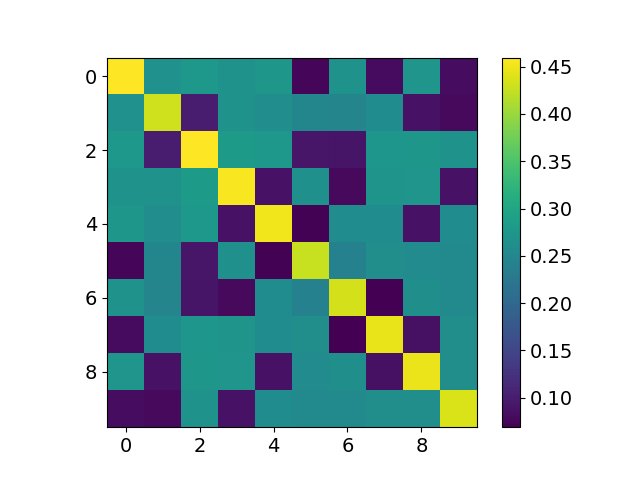}
    \includegraphics[width=0.45\linewidth]{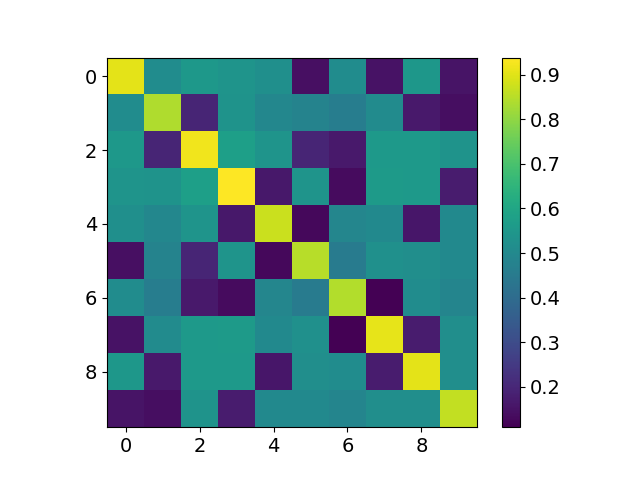}
    \includegraphics[width=0.45\linewidth]{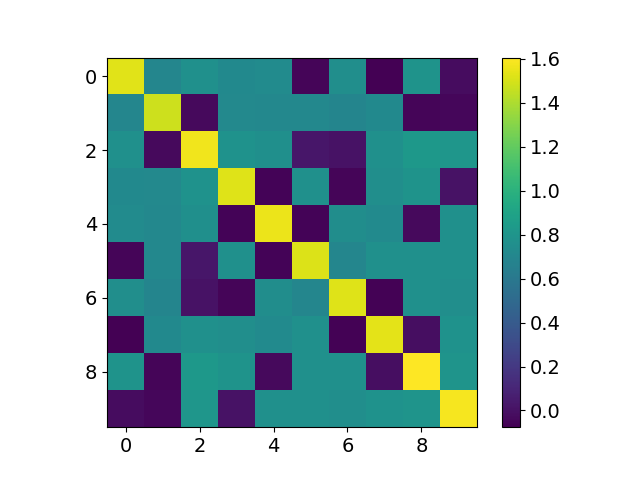}
    \includegraphics[width=0.45\linewidth]{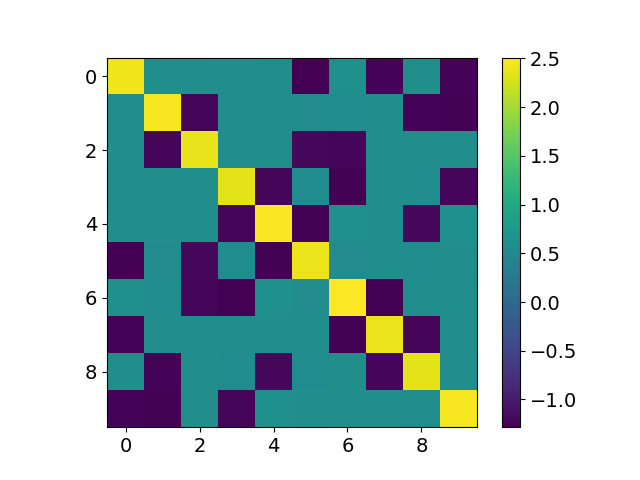}
    \caption{Gram matrices of class-means $\Tilde{M}_1^T\Tilde{M}_1, \Tilde{M}_2^T\Tilde{M}_2, \dots, \Tilde{M}_5^T\Tilde{M}_5$ (consecutively) in the 5-layer MLP head of ResNet20 trained on MNIST with weight decay $0.04$ (standard regularization) and learning rate $0.01$. The displayed solution is the SRG solution (a single run).}
    \label{fig:grams_rebuttals}
\end{figure}

%%%%%%%%%%%%%%%%%%%%%%%%%%%%%%%%%%%%%%%%%%%%%%%%%%%%%%%%%%%%

\end{document}